\documentclass{article}

\usepackage[preprint]{neurips_2026}

\usepackage[utf8]{inputenc} 
\usepackage[T1]{fontenc}    
\usepackage{hyperref}       
\usepackage{url}            
\usepackage{booktabs}       
\usepackage{amsfonts}       
\usepackage{nicefrac}       
\usepackage{microtype}      
\usepackage{xcolor}         

\usepackage{graphicx}
\usepackage{subcaption}

\usepackage{amsmath,amsfonts}
\usepackage{amssymb}
\usepackage{mathtools}
\usepackage{amsthm}

\usepackage{bm}
\usepackage{bbm}
\usepackage{enumerate}
\usepackage{comment}
\usepackage{nicefrac}  
\usepackage{enumitem}
\usepackage{here}

\theoremstyle{plain}
\newtheorem{theorem}{Theorem}[section]

\newtheorem{lemma}[theorem]{Lemma}

\theoremstyle{definition}
\newtheorem{definition}[theorem]{Definition}

\theoremstyle{remark}

\makeatletter
\newcommand*\rel@kern[1]{\kern#1\dimexpr\macc@kerna}
\newcommand*\widebar[1]{%
  \begingroup
  \def\mathaccent##1##2{%
    \rel@kern{0.8}%
    \overline{\rel@kern{-0.8}\macc@nucleus\rel@kern{0.2}}%
    \rel@kern{-0.2}%
  }%
  \macc@depth\@ne
  \let\math@bgroup\@empty \let\math@egroup\macc@set@skewchar
  \mathsurround\z@ \frozen@everymath{\mathgroup\macc@group\relax}%
  \macc@set@skewchar\relax
  \let\mathaccentV\macc@nested@a
  \macc@nested@a\relax111{#1}%
  \endgroup
}
\makeatother

\newcommand{\Two}{I\hspace{-1.2pt}I}
\newcommand{\Three}{I\hspace{-1.2pt}I\hspace{-1.2pt}I}
\newcommand{\Four}{I\hspace{-1.2pt}V}
\newcommand{\Six}{V\hspace{-1.2pt}I}

\allowdisplaybreaks[3]

\title{Supervised Learning as Lossy Compression: Characterizing Generalization and Sample Complexity via Finite Blocklength Analysis}

\author{%
  Kosuke Sugiyama \\
  Waseda University\\
  3-4-1 Okubo, Shinjuku, Tokyo 169-8555, Japan \\
  \texttt{kohsuke0322@asagi.waseda.jp} \\
  \And
  Masato Uchida \\
  Waseda University \\
  3-4-1 Okubo, Shinjuku, Tokyo 169-8555, Japan \\
  \texttt{m.uchida@waseda.jp} \\
}

\begin{document}

\maketitle

\begin{abstract}

This paper presents a novel information-theoretic perspective on generalization in machine learning by framing the learning problem within the context of lossy compression and applying finite blocklength analysis.
In our approach, the sampling of training data formally corresponds to an encoding process, and the model construction to a decoding process.
By leveraging finite blocklength analysis, we derive lower bounds on sample complexity and generalization error for a fixed randomized learning algorithm and its associated optimal sampling strategy.
Our bounds explicitly characterize the degree of overfitting of the learning algorithm and the mismatch between its inductive bias and the task as distinct terms.
This separation provides a significant advantage over existing frameworks.
Additionally, we decompose the overfitting term to show its theoretical connection to existing metrics found in information-theoretic bounds and stability theory, unifying these perspectives under our proposed framework.

\end{abstract}

\section{Introduction}
\label{sec:intro}

Learning theory aims to characterize the relationship between learning algorithm properties and generalization performance.
Various frameworks quantify these properties to assess generalization error and sample complexity.
Specifically, these include uniform convergence based on hypothesis set complexity \cite{mohri2018foundations}, PAC-Bayes focusing on the divergence between prior and posterior distributions \cite{mcallester1999pac}, and stability theory grounded in sensitivity to training data \cite{bousquet2002stability,elisseeff2005stability}.
Given the strong connection between machine learning and information theory, recent works have explored information-theoretic bounds (IT-bounds) using mutual information to measure data-hypothesis dependence \cite{xu2017information}, compression bounds emphasizing model compressibility \cite{shalev2014understanding,arora2019stronger,sefidgaran2022rate}, and frameworks drawing analogies to lossy compression \cite{berger1975rate,nokleby2016rate,jeon2022an}.

In this paper, we focus on the analogy between machine learning and lossy compression as utilized by \cite{nokleby2016rate,jeon2022an}.
This perspective is grounded in a natural correspondence where the complete information of the true distribution is lossily compressed into a finite training dataset, and then decoded into a predictive distribution via learning.
Based on this premise, prior works have analyzed the sample complexity and generalization error of optimal Bayesian learning.
Specifically, they adapt the asymptotic analysis of the minimum rate for compressing infinite symbols to the learning problem.
By linking this rate to sample complexity within a Bayesian framework, these analysis is performed.

Conversely, lossy compression theory has evolved to include finite blocklength analysis \cite{kostina2012fixed}, which characterizes the minimum rate in a {\it non-asymptotic} setting.
This method explicitly quantifies the uncertainty due to the finite blocklength, thereby providing a more detailed analysis than the asymptotic approach utilized in \cite{nokleby2016rate,jeon2022an}.

In this paper, building on this analogy, we investigate the potential of generalization analysis enabled by the novel application of finite blocklength analysis of lossy compression.
To achieve this, we must formulate the entire process of data sampling and learning as lossy compression to enable finite blocklength analysis.
However, prior works provide only a conceptual analogy without detailed specifications, leaving it uncertain whether a concrete correspondence can be established.

As our first contribution, we present a correspondence enabling the application of finite blocklength analysis (Figure \ref{fig:correspondence_overview}).
We formalize the analogy by viewing the training dataset as a codeword of the true distribution, its sampling as encoding, and learning as decoding.
Accordingly, the sampling strategy and learning algorithm correspond to the encoder and decoder, respectively.
However, this correspondence requires aligning the two-step process of symbol observation and encoding with single-step sampling from the true distribution.
We resolve this by redefining, without loss of generality, sampling as two steps: stochastically determining observable regions in the input space, and sampling from them.
By mapping symbols to these regions and the source distribution to the region distribution, we establish a valid correspondence.
This yields a natural mapping where block coding corresponds to sampling from multiple regions, rate to training data size, and distortion to generalization error on observable regions.

Our second contribution is to demonstrate that this correspondence allows us to leverage finite blocklength analysis, specifically the derivation techniques for lower bounds of minimum rates and distortions, to analyze generalization.
Through this approach, we derive lower bounds on sample complexity and generalization error for any fixed learning algorithm and its associated optimal sampling strategy.
A key distinction of our work is its flexibility: whereas \cite{nokleby2016rate,jeon2022an} are limited to optimal Bayesian learning, our lower bounds can be evaluated for any learning algorithm of interest.
Furthermore, this framework considers arbitrary sampling strategies beyond i.i.d. sampling, where training data are drawn independently from the true distribution.
Consequently, this lower bound establishes a fundamental limit for any strategy, including active learning \cite{settles2009active}.

Our third contribution is to demonstrate the utility of these lower bounds through two distinct advantages.
First, our bounds allow for an explicit separation between the term representing the degree of overfitting of the learning algorithm and the term quantifying the mismatch between its inductive bias and the task.
This provides a novel perspective on the relationship between generalization and both the degree of overfitting and the mismatch of inductive bias, represented as separate terms, distinct from existing frameworks.
Second, we establish a direct connection between the overfitting term in our lower bounds and the overfitting metrics appearing in the error upper bounds of and IT-bounds \cite{xu2017information,hellstrom2020generalization} stability theory \cite{bousquet2002stability,elisseeff2005stability}.
This theoretical connection serves as strong evidence supporting the validity of our framework.
Moreover, these two benefits provide a new perspective on the theoretical limits of active learning, which have traditionally been analyzed based on the properties of hypothesis sets and data distributions \cite{,castro2008minimax,raginsky2011lower,hanneke2015minimax,yuan2024regimes,hanneke2025agnostic}.

\section{Preliminaries}
\label{sec:preliminaries}

\textbf{Machine Learning Process.}
Let $\bm{X}$ and $Y$ be random variables representing the input instance and label, taking values in $\mathcal{X}$ and $\mathcal{Y}$, respectively, with realizations denoted by $\bm{x}$ and $y$. 
The pair $(\bm{X}, Y)$ follows a true distribution $P^*_{\bm{X}Y}(\bm{x},y)$ with support $\mathrm{supp}(P^{*}_{\bm{X}Y})$.
The target of learning is the $P^*_{Y|\bm{X}}(y|\bm{x})$.
To this end, a training dataset $\{(\bm{x}_i,y_i)\}^n_{i=1}$ ($n \in \mathbb{N}_+$) is collected from $P^*_{\bm{X}Y}$ via a sampling strategy such as i.i.d. sampling and active learning \cite{settles2009active}.
Next, a hypothesis $h:\mathcal{X} \rightarrow \mathcal{Y}$ is selected from a hypothesis class $\mathcal{H}$ by a learning algorithm using the training dataset.
Let $H$ be the random variable representing this hypothesis. 
If $h$ is a probabilistic model, it is denoted as $\widehat{P}^{h}_{Y|\bm{X}}$.
The generalization error is measured using a divergence $\mathcal{D}(\cdot || \cdot)$ or a loss function $l:\mathcal{Y}\times\mathcal{Y}\rightarrow \mathbb{R}_+$, and is expressed as $\mathcal{D}(P^*_{Y|\bm{X}} || \widehat{P}^{h}_{Y|\bm{X}})$ or $\mathbb{E}_{P^*_{\bm{X}Y}}[l(h(\bm{X}),Y)]$.

\textbf{Fixed-Length Lossy Compression.}
Lossy compression is a framework for compressing (encoding) a symbol $A$, following a source distribution $P_A$, into a codeword and decoding it, permitting decoding errors \cite{berger1975rate}.
Specifically, an encoder $f$ maps $A$ to one of $M \in \mathbb{N}_+$ codewords, and a decoder $g$ reconstructs it.
In fixed-length compression, each codelength is $\log M$ bits, representing the compression level.
Typically, block coding is used to compress a sequence $A^{k}=(A_1,\dots,A_k)$ (where $k \in \mathbb{N}_+$) jointly, rather than single symbols.
Block coding reduces the {\it rate} $R=\log M /k$ as $k$ grows; thus, classical theory analyzes the minimum rate as $k \rightarrow \infty$.
The decoding error, termed {\it distortion}, is quantified by a measure $\mathsf{d}(\cdot;\cdot)$ as $\mathsf{d}(A^k; g(f(A^k))).$
Rate and distortion exhibit a trade-off: higher rates retain more information, reducing distortion (and vice versa).
As $k \rightarrow \infty$, this optimal trade-off is characterized by the rate-distortion function.
It is defined as the minimum rate $R(d)$ achieving a distortion level $d \in \mathbb{R}_+$.
Classical theory derives $R(d)$ to reveal the theoretical compression limits.

\textbf{Finite Blocklength Analysis.}
Whereas traditional information theory deals with asymptotic limits ($k \rightarrow \infty$), finite blocklength analysis characterizes non-asymptotic limits where $k < \infty$ \cite{polyanskiy2010channel}.
We outline the framework for fixed-length lossy compression provided by \cite{kostina2012fixed}.
This approach defines the minimum rate $R(k,d,\epsilon)$ required to maintain the distortion within a level $d$ with probability at least $1-\epsilon$, i.e., satisfying $\mathbb{P}[\mathsf{d}(A^k; g(f(A^k))) > d] \le \epsilon$. 
Specifically, $R(k,d,\epsilon)$ is characterized as:
\begin{align}
    \textstyle R(k, d, \epsilon) = R(d) + \sqrt{\frac{V(d)}{k}} Q^{-1}(\epsilon) + O\big(\frac{\log k}{k}\big). 
\label{eq:FBA_example}
\end{align}
Here, $Q(\cdot)$ is the standard normal complementary cumulative distribution function.
The term $V(d)$, known as the rate-dispersion function, represents the variance of the {\it tilted information} (corresponding to the per-symbol minimum codelength). 
Essentially, $V(d)$ quantifies the uncertainty due to finite blocklength $k$, capturing the stochastic fluctuation in the minimum codelength required for compressing $A^k \sim P_{A^k}$.
The derivation of this result involves proving the achievability part (existence of a code achieving $R(k,d,\epsilon) \le \dots$) and the converse part (a fundamental lower bound $R(k,d,\epsilon)\ge \dots$ satisfied by any valid code).
Additionally, this approach can analyze the minimum achievable distortion $d$ for a fixed rate $R$.
We leverage this analysis to construct a new theoretical framework for machine learning.

\section{Machine Learning Process as  Fixed-Length Lossy Compression}
\label{sec:lc_correspondence}

\begin{figure*}[t]
    \centering
    \includegraphics[width=0.95\linewidth]{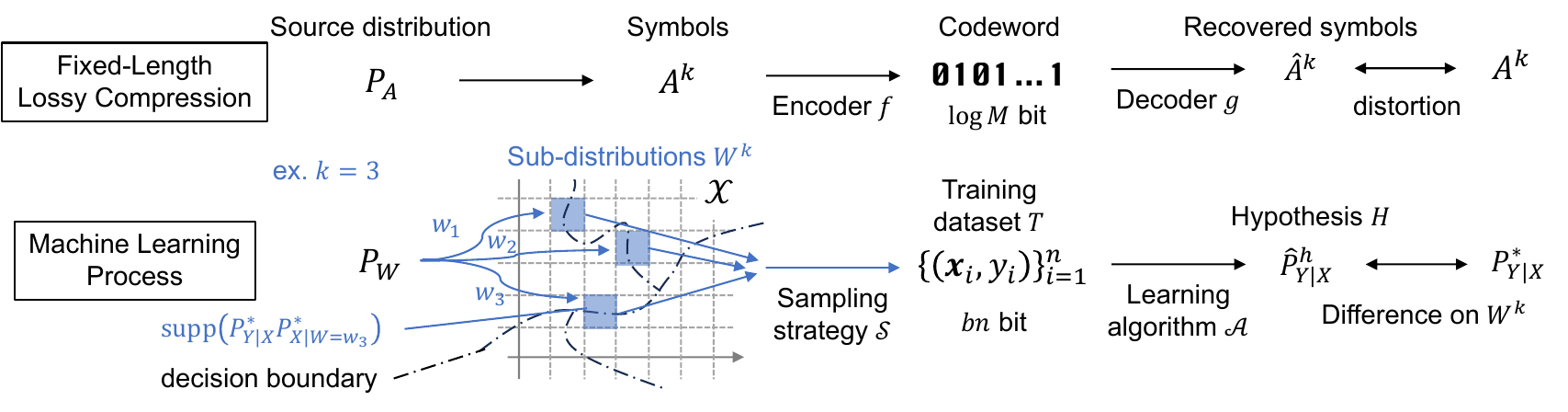}
    \caption{
             Schematic diagram of the correspondence between the machine learning process and fixed-length lossy compression.
            }
    \label{fig:correspondence_overview}
\end{figure*}

In this section, we establish a correspondence between the machine learning process and fixed-length lossy compression, as outlined in Figure \ref{fig:correspondence_overview}.
The core insight lies in interpreting machine learning as a process where the target information $P^{*}_{Y|\bm{X}}$ is compressed into a training dataset through sampling, from which a hypothesis $\widehat{P}^h_{Y|\bm{X}}$ approximating $P^{*}_{Y|\bm{X}}$ is reconstructed via learning.
This process parallels the mechanism of lossy compression, which compresses a symbol into a codeword and decodes it to recover an approximation of the original symbol.
Fundamentally, both share the same abstract structure: compressing an object and reconstructing a close approximation from its compressed representation.
Building on this perspective, we define the following correspondence to develop a theoretical framework grounded in finite blocklength analysis.
We start by considering the case of $k=1$.

\textbf{Codeword and Encoding:}
We first map the codeword to the training dataset $\{(\bm{x}_i,y_i)\}^n_{i=1}$, viewing both as compressed representations.
Because encoding generates codewords and sampling generates datasets, the encoding process (encoder) corresponds to the sampling process (sampling strategy).
While the compression level of encoding is expressed by the codelength $\log M$ bits, the description length of the dataset $\{(\bm{x}_i,y_i)\}^n_{i=1}$ is $bn$ bits, assuming each sample $(\bm{x},y)$ is represented by a constant $b \in \mathbb{N}_+$ bits.
We adopt this assumption hereafter, and the compression level of sampling is expressed as $bn$ bits.

\textbf{Source Distribution and Symbol:}
Encoding compresses a symbol into a codeword.
By mapping the codeword to the training dataset, we view an encoder as a sampling strategy, implying that the symbol corresponds to the target distribution of sampling.
A challenge arises, however, because a symbol is a random variable (or realization), whereas $P^{*}_{\bm{X}Y}$ is a distribution; they do not directly correspond. 
Addressing this mismatch requires a formulation where the probability distribution itself is treated as a random variable.
To this end, we focus on the data generation process: $\bm{x} \sim P^{*}_{\bm{X}} \rightarrow y \sim P^{*}_{Y|\bm{X}=\bm{x}}$.
Let us assume $P^{*}_{\bm{X}}$ is a mixture distribution given by $P^{*}_{\bm{X}}(\bm{x})=\sum_{w\in\mathcal{W}}P^{*}_{\bm{X}|W}(\bm{x}|w)P_{W}(w)$, where $W$ is a random variable on $\mathcal{W}$ following $P_W$.
Here, $W$ serves as the index for the mixture component, with realization $w$.
Assuming $W$ is sampled i.i.d. from $P_W$, the data generation process from $P^{*}_{\bm{X}Y}$ can be described as $w \sim P_W \rightarrow \bm{x} \sim P^{*}_{\bm{X}|W=w} \rightarrow y \sim P^{*}_{Y|\bm{X}=\bm{x}}$.
In particular, we postulate that the support of $P^{*}_{\bm{X}|W}$ varies for each $w$.
This process consists of two stages: (i) the stochastic determination of the sampling target $P^{*}_{\bm{X}Y|W}=P^{*}_{Y|\bm{X}}P^{*}_{\bm{X}|W}$ (or the observable region $\mathrm{supp}(P^{*}_{\bm{X}Y|W})$), and (ii) the observation of a sample according to $P^{*}_{\bm{X}Y|W}$.

Generalizing step (ii) to an arbitrary sampling strategy yields the following workflow for training dataset generation:
an index $w$ (representing $P^{*}_{\bm{X}Y|W=w}$) is drawn from $P_{W}$, followed by the acquisition of data from $P^{*}_{\bm{X}Y|W=w}$ using the sampling strategy.
This parallels the process of lossy compression, where a symbol $A$ is drawn from a source distribution $P_A$ and subsequently encoded into a codeword.
Thus, we map the source distribution to $P_{W}$ and the symbol to $W$ (or $P^{*}_{\bm{X}Y|W}$).
We term the component $P^{*}_{\bm{X}Y|W=w}$ a {\it sub-distribution} of $P^{*}_{\bm{X}Y}$.
For the sake of brevity, we refer to $w$ and $P^{*}_{\bm{X}Y|W=w}$ interchangeably throughout this paper.

This mixture representation of $P^{*}_{\bm{X}}$ can always be constructed by introducing an auxiliary variable $W$ without loss of generality.
For example, if $\bm{X}$ contains discrete latent factors, $P^{*}_{\bm{X}}$ is naturally expressed as a mixture of $P^{*}_{\bm{X}|W}$.
In the general case, we can partition $\mathrm{supp}(P^{*}_{\bm{X}})$ into multiple regions and index each region by $w$, defining $P^{*}_{\bm{X}|W=w}$ as the distribution normalized within each region and $P_W(w)$ as the probability mass of $P^{*}_{\bm{X}}$ on that region.
Figure \ref{fig:correspondence_overview} illustrates a construction example based on grid partitioning.
The choice of $W$ is not unique and may be tailored to the analytical objectives.

\textbf{Block Coding:}
Next, we consider block coding where $k > 1$.
This setting corresponds to sampling $k$ indices $\{w_i\}^k_{i=1}$ from $P_{W^k}=P_W\times\cdots\times P_W$ and obtaining $n$ training samples by applying a sampling strategy to the associated sub-distributions $\{P^{*}_{\bm{X}Y|W=w_i}\}^k_{i=1}$.
For $k=1$, a single sub-distribution is determined, and $n$ samples are obtained via the strategy.
In contrast, for $k>1$, $k$ sub-distributions are determined, and a total of $n$ samples are acquired via the strategy.
Thus, $k$ signifies the number of sub-distributions to which the sampling strategy can be applied.
We term this quantity {\it sampling opportunities}.

\textbf{Rate:}
Since the compression rate is $R=\frac{\log M}{k}$, we define the sampling rate as $R=b\frac{n}{k}$ via our correspondence. 
The term $\frac{n}{k}$ denotes the number of samples per sampling opportunity, serving as a measure of sampling efficiency.
With $b$ constant, for a fixed $k$, the rate $R$ corresponds one-to-one with the training data size $n$, allowing us to treat them as interchangeable quantities.

For example, consider obtaining $n$ samples from $P^{*}_{\bm{X}Y}$ via i.i.d. sampling.
In this case, the process $w \sim P_W \rightarrow (\bm{x},y) \sim P_{\bm{X}Y|W}$ repeats $n$ times, meaning $k=n$ and the sampling efficiency is $1$.
As another example, consider an adaptive experimental design where $50$ samples are collected from $100$ observable regions $\{\mathrm{supp}(P^{*}_{\bm{X}Y|W=w_i})\}^{100}_{i=1}$; in this case, the sampling efficiency is $0.5$.

\textbf{Decoding:}
Decoding in lossy compression reconstructs an approximate symbol from a codeword.
This parallels the learning, which derives a hypothesis approximating the sub-distribution from a training dataset.
Thus, the decoder maps to the learning algorithm.

\textbf{Distortion:}
Distortion quantifies the deviation of the decoded output from the original symbol.
Under our framework, this maps to the discrepancy between $P^*_{Y|\bm{X}}$ and the learned model $P^{h}_{Y|\bm{X}}$ on the sub-distribution $W$, specifically, the generalization error of $h$ on $W$.

In the following section, we show that this correspondence enables the application of the converse part from finite blocklength analysis of lossy compression \cite{kostina2012fixed}, allowing us to derive new lower bounds for sample complexity and generalization error.

\section{Theoretical Analysis}
\label{sec:lc_analysis}

\paragraph{Notation}
Let $T$ be the random variable for the training dataset $\{(\bm{x}_i,y_i)\}^n_{i=1}$ on the domain $\mathcal{T}$, with realization $t$ ($|t|=n$). 
We define the sampling strategy $\mathcal{S}:\mathcal{W}^k \rightarrow \mathcal{T}$ to map the sub-distribution sequence $W^k = (W_1,\dots,W_k)$ to $T$, and the randomized learning algorithm $\mathcal{A}: \mathcal{T} \rightarrow \mathcal{H}$ to map $T$ to a hypothesis $H$.
The entire process is thus given by $H=\mathcal{A}(\mathcal{S}(W^k))$. 
Additionally, $\mathcal{S}$ and $\mathcal{A}$ are represented by Markov kernels $P^{\mathcal{S}}_{T|W^k}$ and $P^{\mathcal{A}}_{H|T}$, respectively, where $P^{\mathcal{S}}_{T|W}(t|w)=0$ for any $t$ unobtainable from $w$.
The distortion measure can be defined as $\mathsf{d}(W;H) := \mathbb{E}_{P^*_{\bm{X}|W}}[\mathcal{D}(P^{*}_{Y|\bm{X}}(Y|\bm{X}) || \hat{P}^{H}_{Y|\bm{X}}(Y|\bm{X}))]$ for probabilistic models, and $\mathsf{d}(W;H) := \mathbb{E}_{P^{*}_{Y|\bm{X}}P^*_{\bm{X}|W}}[l(Y,H(\bm{X}))]$ for deterministic models.
For $k>1$, let $\mathsf{d}(W^k;H) := \frac{1}{k} \sum^k_{i=1} \mathsf{d}(W_i;H)$.
This measure evaluates prediction error on $\mathrm{supp}(\frac{1}{k} \sum^k_{i=1}P^*_{\bm{X}|W_i})$ instead of the entire support $\mathrm{supp}(P^{*}_{\bm{X}})$.
This implies that, based on the $W^k$ used to obtain $T$, regions where training samples are unavailable are treated as out-of-distribution and excluded from the error evaluation.
Since each $W_i$ is sampled independently from $P_W$, as $k \rightarrow \infty$, $d(W^k;H)$ converges to $\mathbb{E}_{P^*_{\bm{X}Y}}[l(Y,H(\bm{X}))]$ or $\mathbb{E}_{P^*_{\bm{X}}}[\mathcal{D}(P^{*}_{Y|\bm{X}}(Y|\bm{X}) || \hat{P}^{H}_{Y|\bm{X}}(Y|\bm{X}) )]$ \cite{van2000asymptotic}.
Thus, $\mathsf{d}(W^k;H)$ asymptotically coincides with the expected risk in learning theory \cite{mohri2018foundations}.
We assume that $\mathsf{d}(\cdot;\cdot)$ is bounded (Assumption (I)).

\subsection{Definition of Key Concepts}
\label{subsec:lc_def_key_concepts}

In active learning, sampling strategy effectiveness depends on the learning algorithm \cite{hanneke2014theory,sloman2022characterizing}.
Thus, a unified treatment of sampling and learning must account for this interdependence.
Our framework analyzes the optimal sampling strategy for a fixed learning algorithm $\mathcal{A}$.
As a preliminary, we introduce key definitions based on \cite{kostina2012fixed,kostina2016nonasymptotic}; details are provided in Appendix~\ref{apdx:sec:lc_detail_def_interpret}.

\paragraph{``Good'' Sampling Strategies}
The converse part in \cite{kostina2012fixed}, which we apply here, restricts its focus to "good" codes that achieve low distortion with high probability.
Accordingly, building on Definitions 1 and 2 in \cite{kostina2012fixed}, we introduce Definition \ref{def:lc_k_n_d_e_A_sampling} to characterize a class of ``good'' sampling strategies capable of achieving small distortion with high probability via $\mathcal{A}$ for a finite sampling opportunities $k$, which serves as the primary subject of our analysis.
Furthermore, we define the minimum training data size and minimum rate within this class.

\begin{definition}
  For sampling opportunities $k$, a $(k, n, d, \epsilon, \mathcal{A})$ sampling in $\{\mathcal{W}^k, \mathcal{H}, P_{W^k}, \mathsf{d}:\mathcal{W}^k\times\mathcal{H} \rightarrow [0,+\infty]\}$ is defined as a sampling strategy $\mathcal{S}$ that satisfies the training data size $|\mathcal{S}(W^k)|=n$ and the condition $\mathbb{P}[\mathsf{d}(W^k; \mathcal{A}(\mathcal{S}(W^k))) > d] \le \epsilon$, given a distortion level $d$, an excess distortion probability $\epsilon$, and a learning algorithm $\mathcal{A}$.
  In the case of $k=1$, it is denoted as $(n, d, \epsilon, \mathcal{A})$ sampling.
  We define the minimum training data size as $n^*(k, d, \epsilon, \mathcal{A}) := \min\{ n: \exists (k, n, d, \epsilon, \mathcal{A}) \text{ sampling} \}$ and the minimum rate as $R(k,d,\epsilon, \mathcal{A}) := \frac{b}{k}n^*(k,d,\epsilon, \mathcal{A})$.
\label{def:lc_k_n_d_e_A_sampling}
\end{definition}

Depending on the parameters $n$, $d$, and $\epsilon$, this class encompasses strategies such as i.i.d. sampling and active learning, as these methods are capable of producing high-accuracy hypotheses with sufficient data.
By definition, since satisfying $\mathbb{P}[\mathsf{d}(W^k; \mathcal{A}(\mathcal{S}(W^k))) > d] \le \epsilon$ with $\mathcal{A}$ requires at least $n^*(k, d, \epsilon, \mathcal{A})$ samples, this quantity represents the sample complexity of $\mathcal{A}$.

In Section \ref{subsec:lc_sample_comp_bound}, we apply the approach of \cite{kostina2012fixed} to derive a lower bound for this sample complexity $n^*(k, d, \epsilon, \mathcal{A})$ ($\propto R(k,d,\epsilon, \mathcal{A})$).
This analysis characterizes $R(k,d,\epsilon, \mathcal{A})$ using the asymptotic minimum rate as $k\rightarrow\infty$ and the variance of the minimum required training data size per pair $(w,h)$.
These are quantified by the {\it rate-distortion function} and the variance of {\it tilted information}, respectively.
Below, we define these quantities, building on \cite{kostina2012fixed} while adapting them to our framework.

\paragraph{Rate-Distortion Function}
We define the rate-distortion function $R(d,\mathcal{A})$, representing the minimum rate as $k\rightarrow\infty$, based on \cite{kostina2012fixed}. 
To begin, we define $\langle k,n,d,\mathcal{A} \rangle$ sampling as the class obtained by replacing the condition in Definition \ref{def:lc_k_n_d_e_A_sampling} with $\mathbb{E}_{P^{\mathcal{A}}_{H|T}P^{\mathcal{S}}_{T|W}P_W}[\mathsf{d}(W^k;H)] \le d$.
We denote the minimum rate within this class as $R(k,d,\mathcal{A})$.
$R(d,\mathcal{A})$ is then defined as $R(d,\mathcal{A}) := \limsup_{k\rightarrow\infty}R(k,d, \mathcal{A})$.
Let $\bar{R}(k,d,\mathcal{A})$ denote the minimum rates for the scheme $\{H_i = \mathcal{A}(\mathcal{S}(W_i))\}^k_{i=1}$, where training is performed for each $W$ independently.
We assume that $R(d, \mathcal{A}) \le \limsup_{k\rightarrow \infty}\bar{R}(k,d, \mathcal{A})$ (Assumption (\Two)). 
This implies that it is more efficient for $\mathcal{A}$ to learn $W^k$ collectively than to learn each $W_i$ individually.
Under this assumption, we establish the rate-distortion theorem for our proposed framework as follows, where $\mathbb{I}(\cdot;\cdot)$ denotes the mutual information.
See Appendix~\ref{apdx:proof:thm:rate_distortion_theorem} for the proof:
\begin{theorem}
    Under Assumptions (I) and (\Two), $R(d,\mathcal{A})$ is non-increasing and convex in $d$, and for any finite $d >0$, we have
    \footnote{
        We can further show that $\limsup_{k\rightarrow \infty} R(k,d,\epsilon, \mathcal{A})=\mathbb{R}_W(d, \mathcal{A})$ for any $\epsilon \in (0,1)$, but we defer the formal derivation to Theorem~\ref{thm:rate_distortion_theorem_epsilon} to maintain consistency with the subsequent definitions and assumptions in Appendices \ref{apdx:subsec:rate_distortion_theorem_RkdeA} and \ref{apdx:subsec:proof_thm:rate_distortion_theorem_epsilon}.}
    :
    \begin{align}
        \textstyle R(d,\mathcal{A})=\mathbb{R}_W(d, \mathcal{A}) := \inf_{\mathcal{S}: \mathbb{E}[\mathsf{d}(W;\mathcal{A}(\mathcal{S}(W)))] \le d}  \mathbb{I}(W;T).
        \label{eq:lc_R_W}
    \end{align}
\label{thm:rate_distortion_theorem}
\end{theorem}
To exclude extremes of $d$, we assume that $d_{\mathrm{min}} := \inf\{d: \mathbb{R}_W(d, \mathcal{A}) < \infty \}$ exists and is finite (Assumption (\Three)).
Additionally, applying the analysis of \cite{kostina2012fixed} requires assuming an optimal sampling strategy $\mathcal{S}^*$ that attains the infimum of Eq. \eqref{eq:lc_R_W} exactly as $\mathbb{E}[\mathsf{d}(W;\mathcal{A}(\mathcal{S^*}(W)))] = d$ (Assumption (\Four)).
Meeting this condition entails that $\mathcal{S}^*$ can select a minimal sample set from $\mathrm{supp}(P^*_{\bm{X}|W})$ to satisfy $\mathbb{E}_{P^{\mathcal{A}}_{H|T}P^{\mathcal{S}}_{T|W}P_W}[\mathsf{d}(W;H)] \le d$ 
This implies that $\mathcal{S}^*$ allows for the selective acquisition of any sample within $\mathrm{supp}(P^{*}_{\bm{X}Y|W})$.
Such a capability corresponds to the membership query synthesis in active learning \cite{angluin1988queries}.

\paragraph{Tilted Information}
The analysis in \cite{kostina2012fixed} quantifies the deviation of the minimum rate at finite $k$ ($<\infty$) from the rate-distortion function (the asymptotic minimum rate as $k\rightarrow\infty$) using the variance of a quantity known as tilted information.
This quantity corresponds to the minimum codelength per symbol, and its expectation coincides with the rate-distortion function.
In this section, following \cite{kostina2012fixed,kostina2016nonasymptotic}, we define tilted information within our framework.

As a preparation, we introduce the {\it information density} for $(w,t) \in \mathcal{W}\times\mathcal{T}$ distributed according to $P_{WT}(w,t)$ as $\iota_{W;T}(w;t) := \log \frac{\mathrm{d}P_{WT}}{\mathrm{d}(P_{W} \times P_{T})}(w,t)$.
By this definition, $\mathbb{E}_{P_{WT}}[\iota_{W;T}(W;T)] = \mathbb{I}(W;T)$ holds.
Specifically, when using $P^{\mathcal{S}^*}_{T|W}P_W$ with $T^{\mathcal{S}^*}
= \mathcal{S}^*(W)$, we have 
$\mathbb{E}_{P^{\mathcal{S}^*}_{T|W}P_{W}}[\iota_{W;T^{\mathcal{S}^{*}}}(W;T)] = \mathbb{I}(W;T^{\mathcal{S}^*}) = \mathbb{R}_W(d, \mathcal{A})$.
Since its expectation equals $\mathbb{R}_W(d, \mathcal{A})$, $\iota_{W;T^{\mathcal{S}^*}}(w;t)$ corresponds to the minimum value of the rate $bn$ ($k=1$) per pair $(w,h)$.
That is, $\iota_{W;T^{\mathcal{S}^*}}(w;t)$ can be interpreted as the information content of $t$ obtained from $w$ via $\mathcal{S}^*$, which corresponds to $|t|b$ [bit].
Using $\iota_{W;H^{*}_{\mathcal{A}}}$, we define the tilted information within our framework:
\begin{definition}
  For any $\mathcal{A}$ and $d > d_{\mathrm{min}}$, the $\mathcal{A}$-specified $d$-tilted information at $(w,t)$ is defined as follows 
  , with $\lambda^{*}_{\mathcal{A}}(d) = -\frac{\mathrm{d} \mathbb{R}_W(d,\mathcal{A})}{\mathrm{d}d}> 0$:
  \begin{equation}
    \jmath_{W}(w,t,h,d,\mathcal{A}) := \iota_{W;T^{\mathcal{S}^{*}}}(w;t) + \lambda^{*}_{\mathcal{A}}(d)(\mathsf{d}(w;h) - d).
  \label{eq:ML_lc_dfn_d-tilted}
  \end{equation}
\label{def:tilted_information}
\end{definition}
Since $\mathbb{E}_{P^{\mathcal{A}}_{H|T}P^{\mathcal{S}^*}_{T|W}P_{W}}[\mathsf{d}(W;H)]=d$, we have $\mathbb{E}_{P^{\mathcal{A}}_{H|T}P^{\mathcal{S}^*}_{T|W}P_{W}}[\jmath_{W}(W,T,H,d,\mathcal{A})] = \mathbb{R}_W(d, \mathcal{A})$.

\textbf{Interpretation of $\jmath_{W}$:} 
For the following reasons, $\jmath_{W}(w,t,h,d,\mathcal{A})$ can be interpreted as the training data size $\times b$ [bits] required to obtain $h$ from $w$ ``under the constraint $\mathbb{E}[\mathsf{d}(W;\mathcal{A}(\mathcal{S}(W)))]\le d$.''
First, as we discussed previously, the first term $\iota_{W;T^{\mathcal{S}^{*}}}(w;t)$ corresponds to the size of the dataset $t$ sampled via $\mathcal{S}^*$, which is $|t| b$ [bit].
The second term, $\lambda^{*}_{\mathcal{A}}(d)(\mathsf{d}(w;h) - d)$, serves as a reward or penalty based on the distortion $\mathsf{d}(w;h)$ of $h$.
This term is negative if the distortion $\mathsf{d}(w;h)$ is lower than the target $d$ and positive otherwise.
Since $R(d,\mathcal{A}) ~(=\mathbb{R}_W(d,\mathcal{A}))$ is non-increasing and convex, the coefficient $\lambda^{*}_{\mathcal{A}}(d)$ increases as the required $d$ decreases.
Denoting the unit of distortion as [dist], $\lambda^{*}_{\mathcal{A}}(d)$ has the unit [bit/dist] and can be interpreted as a conversion rate that translates the reward or penalty measured in distortion [dist] into information quantity [bit].
Thus, $\jmath_{W}$ represents the pure required training data size (the first term) adjusted, or ``tilted'', according to the specified distortion level $d$ via the second term.
Specifically, in the case where $\mathsf{d}(w;h) < d$, the second term reflects the fact that the amount of training data required to merely achieve the level $d$ is less than that required to obtain the specific $h$, resulting in $\jmath_{W}(w,t,h,d,\mathcal{A}) < \iota_{W;T^{\mathcal{S}^*}}(w;t)$ (and vice versa).

\subsection{Lower Bound of Sample Complexity}
\label{subsec:lc_sample_comp_bound}

In this section, we derive a lower bound on sample complexity by applying the converse part of the finite blocklength analysis from \cite{kostina2012fixed}.
As a preparation, we show Theorem \ref{thm:eq:ML_lossy_sc_eps_conv_bound} that describes the limit on the excess distortion probability $\epsilon$ for any $(n, d, \epsilon, \mathcal{A})$ sampling. 
Although this theorem corresponds to Theorem 7 in \cite{kostina2012fixed}, their proof method assumes optimal encoders and decoders and is thus inapplicable to our setting, where the $\mathcal{A}$ (corresponding to the decoder) is fixed.
We resolve this by adapting the approach of \cite{kostina2016nonasymptotic}, which proved a comparable result in a different context.
See Appendix~\ref{apdx:proof:thm:eq:ML_lossy_sc_eps_conv_bound} for the proof.
\begin{theorem}
  For any $d > d_{\mathrm{min}}$ and any $(n, d, \epsilon, \mathcal{A})$ sampling, the following holds:
  \begin{equation}
    \textstyle \epsilon \ge \sup_{\gamma \ge 0} \{ \mathbb{P}[\jmath_W(W,T,H,d,\mathcal{A}) \ge bn + \gamma] - e^{-\gamma} \}.
  \label{eq:ML_lossy_sc_eps_conv_bound}
  \end{equation}
\label{thm:eq:ML_lossy_sc_eps_conv_bound}
\end{theorem}
This theorem demonstrates that, given $n$, $d$ and $\mathcal{A}$, the limit of $\epsilon$ for which $\mathbb{P}[\mathsf{d}(W; \mathcal{A}(\mathcal{S}(W))) > d] \le \epsilon$ can be evaluated by the probability that $\jmath_{W}$ exceeds $n$ ($\times b$). 
This result validates the plausibility of interpreting $\jmath_{W}$ as the minimum required training data size.
Furthermore, since the converse part of the analysis in \cite{kostina2012fixed} is demonstrated using their Theorem 7, which corresponds to this theorem, our result enables the application of their converse part to our framework.

As a final preparation, we make the following two assumptions based on \cite{kostina2012fixed}:
(V) The required distortion level $d$ satisfies  $d \in (d_{\mathrm{min}},d_{\mathrm{max}})$, where $d_{\mathrm{max}} := \sup \{d : \mathbb{R}_W(d,\mathcal{A})>0 \}$, and the excess distortion probability $\epsilon$ satisfies $0 < \epsilon < 1$.
(\Six) $\jmath_W(w,t, h,d,\mathcal{A})$ possesses a finite absolute third moment.
Under these assumptions, we derive a lower bound for the minimum rate $R(k,d,\epsilon, \mathcal{A})$ as following Theorem \ref{thm:ML_lossy_sc_rate_conv_bound}.
This theorem is established using Theorem \ref{thm:eq:ML_lossy_sc_eps_conv_bound} and the Berry-Esseen central limit theorem \cite{erokhin1958CLT}.
The proof and full statement are deferred to Appendices~\ref{apdx:proof:thm:ML_lossy_sc_rate_conv_bound} and \ref{apdx:sec:detail_thm:ML_lossy_sc_rate_conv_bound}.

\begin{theorem}
  Under Assumptions (I)--(\Six), for any $k$, $\mathcal{A}$, $d \in (d_{\mathrm{min}},d_{\mathrm{max}})$ and $\epsilon \in (0,1)$:
  \begin{equation}
    \textstyle R(k,d,\epsilon, \mathcal{A}) \ge R(d,\mathcal{A}) + \sqrt{\frac{V(d,\mathcal{A})}{k}}Q^{-1}(\epsilon) + \Delta_k(d,\epsilon)+ O\big(\frac{\log k}{k}\big).
  \label{eq:ML_lossy_sc_rate_conv_bound}
  \end{equation}
  Here, the term $V(d,\mathcal{A})$ denotes the rate-dispersion function, defined as $V(d, \mathcal{A}) := \mathrm{Var}_{=P^{\mathcal{A}}_{H|T}P^{\mathcal{S}^*}_{T|W} P_W}(\jmath_W(W,T,H,d,\mathcal{A}))$. 
  The term $\Delta_k(d,\epsilon)$ is defined in Appendix~\ref{apdx:sec:detail_thm:ML_lossy_sc_rate_conv_bound}.
\label{thm:ML_lossy_sc_rate_conv_bound}
\end{theorem}

Via $n^*(k,d,\epsilon, \mathcal{A})=\frac{k}{b}R(k,d,\epsilon, \mathcal{A})$, Theorem \ref{thm:ML_lossy_sc_rate_conv_bound} also lower-bounds sample complexity.
This theorem characterizes the minimum rate at finite $k$ via the asymptotic rate $R(d,\mathcal{A})$ and the dispersion $V(d,\mathcal{A})$, reflecting the stochastic region determination of $\mathrm{supp}(\{P^{*}_{\bm{X}|W=W_i}\}^k_{i=1})$.
This analysis models the fundamental sampling uncertainty, persisting even with optimal $\mathcal{S}^*$, as $k$ stochastic region determinations.
By capturing all randomness of $W$, $\mathcal{S}^*$, and $\mathcal{A}$ in $V(d,\mathcal{A})$, we evaluate the theoretical minimum rate limit.
Larger $k$ reduces randomness and weakens $V(d,\mathcal{A})$'s effect, whereas stricter $\epsilon$ increases dependence on $W^k$, strengthening it.
Moreover, $\Delta_k(d,\epsilon)$ quantifies the efficiency with which $\mathcal{A}$ learns the entire sequence $W^k$, taking a larger negative value as this efficiency increases (see Appendix~\ref{apdx:sec:detail_thm:ML_lossy_sc_rate_conv_bound}).
Since $R(k,d,\epsilon, \mathcal{A})$ is the minimum rate satisfying $\mathbb{P}[\mathsf{d}(W^k; \mathcal{A}(\mathcal{S}(W^k))) > d] \le \epsilon$, this analysis is analogous to the concept of sample complexity in PAC learning \cite{mohri2018foundations}.

The lower bound derived via finite blocklength analysis exhibits two distinct features:
(i) $V(d,\mathcal{A})$ can be decomposed into two terms that separately evaluate the degree of overfitting of $\mathcal{A}$ and the mismatch between the inductive bias of $\mathcal{A}$ and the task.
This decomposition allows for an analysis of how these factors relate to $n^*(k,d,\epsilon, \mathcal{A})$.
(ii) Many components of $R(d,\mathcal{A})$ and $V(d,\mathcal{A})$ are directly related to representative quantities found in IT-bounds and stability theory.
These features imply that our framework enables a multifaceted characterization of the theoretical limits of learning.
Furthermore, the fact that our results relate to existing theories derived from completely different perspectives supports the theoretical validity of our framework.
We outline these features below; details and proofs are provided in Appendices~\ref{apdx:sec:bridge_ours_others} and \ref{apdx:subsec:derive_V_decomposition}, respectively.

Feature (i) follows from the variance decomposition of $V(d,\mathcal{A})$ over $P_W$: 
\begin{align}
\begin{split}
  &V(d,\mathcal{A})=\
    \underbracket[0.5pt]{\mathbb{E}_{P_W}\big[ \mathrm{Var}_{P^{\mathcal{A}}_{H|T}P^{\mathcal{S}^*}_{T|W}} (\jmath_{W}(W,T,H,d,\mathcal{A})) \big]  }_{=: V_{\mathrm{in}}(d,\mathcal{A})}
    + \underbracket[0.5pt]{\mathrm{Var}_{P_W}\big( \mathbb{E}_{P^{\mathcal{A}}_{H|T}P^{\mathcal{S}^*}_{T|W}}[\jmath_{W}(W,T,H,d,\mathcal{A})] \big)   }_{=: V_{\mathrm{bet}}(d,\mathcal{A})}. \notag
\end{split}
\end{align}

\textbf{Interpretation of $V_{\mathrm{in}}$ and $V_{\mathrm{bet}}$:}
As argued below, $V_{\mathrm{in}}$ and $V_{\mathrm{bet}}$ can be interpreted as measures of $\mathcal{A}$'s overfitting and inductive bias mismatch, respectively.
First, $V_{\mathrm{in}}$ is the expectation of $\mathrm{Var}_{P^{\mathcal{A},\mathcal{S}^*}_{H|W}} (\jmath_{W})$ over $W$, capturing fluctuations in the required data size $\jmath_W$ due to $\mathcal{S}^*$ and $\mathcal{A}$.
We assume hereafter that the optimal $\mathcal{S}^*$ samples freely, yielding $T$ with low variance.
A large $V_{\mathrm{in}}$ implies $\mathcal{A}$'s high sensitivity to small variations in $T \sim P^{\mathcal{S}^*}_{T|W}$ or significant inherent randomness of $\mathcal{A}$.
As detailed later, expanding $V_{\mathrm{in}}$ yields terms corresponding to overfitting metrics in IT-bounds and stability theory.
Thus, $V_{\mathrm{in}}$ effectively measures overfitting.

In contrast, $V_{\mathrm{bet}}$ measures the variation of the average required data size $\mathbb{E}_{P^{\mathcal{A}}_{H|T}P^{\mathcal{S}^*}_{T|W}}[\jmath_{W}]$ across $W$.
A large $V_{\mathrm{bet}}$ implies a significant disparity between $W$ easy for $\mathcal{A}$ to learn and those difficult.
For instance, if $\mathcal{A}$ optimizes a linear model, $W$ exhibiting a linear relationship between $\bm{X}$ and $Y$ are easy to learn, while non-linear ones are challenging; this discrepancy leads to a large $V_{\mathrm{bet}}$.
Conversely, if $\mathcal{A}$ requires a consistent data size across all $W$, $V_{\mathrm{bet}}$ remains small.
Thus, when $\mathcal{A}$ suits some $W$, $V_{\mathrm{bet}}$ quantifies the mismatch between the task $P^*_{\bm{X}Y}$ and $\mathcal{A}$'s inductive bias.

Next, we outline $V_{\mathrm{in}}$, $V_{\mathrm{bet}}$, and $R(d,\mathcal{A})$, and their connections to existing theories (Feature (ii)).

\textbf{Details of $V_{\mathrm{in}}$:}
Using the definition of $\jmath_{W}$ and variance decomposition over $P^{\mathcal{S}^*}_{T|W}$ and $P^{\mathcal{A}}_{H|T}$, $V_{\mathrm{in}}$ admits a unique decomposition that facilitates a detailed analysis of $\mathcal{A}$:

\begin{align}
    &V_{\mathrm{in}}(d,\mathcal{A}) = \mathbb{E}_{P_W}[ V^{\iota}_{\mathrm{in},\mathcal{S}} 
        +  (\lambda^{*}_{\mathcal{A}}(d))^2(V^{\mathsf{d}}_{\mathrm{in},\mathcal{S}} + V^{\mathsf{d}}_{\mathrm{in},\mathcal{A}}) + 2\lambda^{*}_{\mathcal{A}}(d) V^{\mathrm{cov}}_{\mathrm{in}}], \notag \\ 
    &\text{where~} V^{\iota}_{\mathrm{in},\mathcal{S}} := \mathrm{Var}_{P^{\mathcal{S}^*}_{T|W}} (\iota_{W;T^{\mathcal{S}^*}}(W;T) ), 
    V^{\mathsf{d}}_{\mathrm{in},\mathcal{S}} := \mathrm{Var}_{P^{\mathcal{S}^*}_{T|W}} ( \mathbb{E}_{P^{\mathcal{A}}_{H|T}} [ \mathsf{d}(W;H) ]  ),  \notag \\
    &~~~~~~~~~~ V^{\mathsf{d}}_{\mathrm{in},\mathcal{A}} := \mathbb{E}_{P^{\mathcal{S}^*}_{T|W}} [ \mathrm{Var}_{P^{\mathcal{A}}_{H|T}} ( \mathsf{d}(W;H) )  ], V^{\mathrm{cov}}_{\mathrm{in}} :=\mathrm{Cov}_{P^{\mathcal{A}}_{H|T}P^{\mathcal{S}^*}_{T|W}} (\iota_{W;T^{\mathcal{S}^*}}(W;T), \mathsf{d}(W;H)).\notag
\end{align}

$V^{\iota}_{\mathrm{in},\mathcal{S}}$ measures the degree of fluctuation in the size of $t$ when sampled via the optimal strategy $\mathcal{S}^*$.
Expanding this variance yields the term $\mathrm{Var}_{P^{\mathcal{A}}_{H|T}P^{\mathcal{S}^*}_{T|W}} (\iota_{T;H}(t;h))$. 
This quantity matches the overfitting metric in IT-bounds \cite{hellstrom2020generalization}, differing only in the sampling strategy.
Whereas \cite{hellstrom2020generalization} assumes i.i.d. sampling, our term evaluates this variance using $\mathcal{S}^*$.
Details are provided in Appendix~\ref{apdx:subsec:V_in_S_A_density}.

$V^{\mathsf{d}}_{\mathrm{in},\mathcal{S}}$ and $V^{\mathsf{d}}_{\mathrm{in},\mathcal{A}}$ quantify fluctuations in $\mathsf{d}(W;H)$  due to small variations in $T$ and the randomness of $\mathcal{A}$, respectively.
These correspond to the quantities measured by uniform stability with respect to training data variations \cite{bousquet2002stability} and the internal randomness of $\mathcal{A}$ \cite{elisseeff2005stability}, respectively.
While this theory assumes i.i.d. sampling, our terms differ only by using $\mathcal{S}^*$.
Uniform stability measures overfitting and is widely used in theoretical analyses, such as for SGD \cite{hardt2016train}. 
Details are provided in Appendix~\ref{apdx:subsec:V_in_S_A_distortion}.

$V^{\mathrm{cov}}_{\mathrm{in}}$ represents the correlation between $\iota_{W;T^{\mathcal{S}^*}}(W;T)$ and $\mathsf{d}(W;H)$.
It becomes negative if large $\iota_{W;T^{\mathcal{S}^*}}(W;T)$ correlates with small $\mathsf{d}(W;H)$.
Since hypotheses requiring more data typically achieve lower distortion, $V^{\mathrm{cov}}_{\mathrm{in}}$ is generally expected to be negative.

\textbf{Details of $V_{\mathrm{bet}}$:}
Based on the definition of $\jmath_{W}$, $V_{\mathrm{bet}}$ can be expanded as follows:

\begin{align}
    &V_{\mathrm{bet}}(d,\mathcal{A}) = V^{\iota}_{\mathrm{bet}} + (\lambda^{*}_{\mathcal{A}}(d))^2V^{\mathsf{d}}_{\mathrm{bet}} + 2\lambda^{*}_{\mathcal{A}}(d) V^{\mathrm{cov}}_{\mathrm{bet}}, ~~~~\text{where~} \iota=\iota_{W;T^{\mathcal{S}^*}}(W;T), \mathsf{d}=\mathsf{d}(W;H), \notag \\ 
    & V^{\iota}_{\mathrm{bet}} := \mathrm{Var}_{P_W}(  \mathbb{E}_{P^{\mathcal{S}^*}_{T|W}}[\iota  ]), ~V^{\mathsf{d}}_{\mathrm{bet}} := \mathrm{Var}_{P_W}(  \mathbb{E}_{P^{\mathcal{A},\mathcal{S}^*}_{H|W}}[\mathsf{d}]),
    V^{\mathrm{cov}}_{\mathrm{bet}} := \mathrm{Cov}_{P_W}( 
            \mathbb{E}_{P^{\mathcal{S}^*}_{T|W}}[\iota],  \mathbb{E}_{P^{\mathcal{A},\mathcal{S}^*}_{H|W}}[\mathsf{d}]).\notag
\end{align}

The terms $V^{\iota}_{\mathrm{bet}}$ and $V^{\mathsf{d}}_{\mathrm{bet}}$ quantify the mismatch of $\mathcal{A}$'s inductive bias in terms of $\iota_{W;T^{\mathcal{S}^*}}(W;T)$ and $\mathsf{d}(W;H)$, respectively.
Like $V^{\mathrm{cov}}_{\mathrm{in}}$, $V^{\mathrm{cov}}_{\mathrm{bet}}$ is expected to be negative.
Evaluation via $V_{\mathrm{bet}}$ parallels the concept that the required data size decreases as $\mathcal{A}$ contains more of the inductive bias required by the task \cite{canatar2021spectral,boopathy2023model,boopathy2024towards}, and the result that sample-wise difficulty varies significantly by bias type \cite{kwok2024dataset}.
While the theoretical correspondence between $V_{\mathrm{bet}}$ and these studies remains unclear, $V_{\mathrm{in}}$, naturally derived within this framework, directly relates to multiple existing frameworks quantifying overfitting.
Thus, $V_{\mathrm{bet}}$ is expected to serve as a novel metric for quantifying inductive bias mismatch.

\textbf{Overview of $R(d,\mathcal{A})$:}
This term can be expressed as $R(d,\mathcal{A})=\mathbb{I}(T;H) + (\mathbb{I}(W;T|H) - \mathbb{I}(T;H|W))$.
The term $\mathbb{I}(T;H)$ corresponds to the overfitting metric in IT-bounds \cite{xu2017information}.
Although \cite{xu2017information} evaluates $\mathbb{I}(T;H)$ assuming i.i.d. sampling, our framework evaluates it using $\mathcal{S}^*$.
Consequently, the first term $R(d,\mathcal{A})$ in Eq. \eqref{eq:ML_lossy_sc_rate_conv_bound} is also directly related to IT-bounds.
Details are provided in Appendices~\ref{apdx:subsec:detail_interpret_lc_rate_dist_func} and \ref{apdx:subsec_bridge_rate_dist_others}.

\subsection{Lower Bound of Generalization Error}
\label{subsec:lc_error_bound}

Finally, we derive the limit on the generalization error realizable by $\mathcal{S}^*$ and $\mathcal{A}$ at a fixed rate $R$ (or $n$).
This limit is denoted $D(k,R,\epsilon, \mathcal{A})$, the inverse of $R(k,d,\epsilon, \mathcal{A})$.
Specifically, $D(k,R,\epsilon, \mathcal{A})$ represents the minimum $d$ satisfying $\mathbb{P}[\mathsf{d}(W^k; \mathcal{A}(\mathcal{S}(W^k))) > d] \le \epsilon$ holds at rate $R$. 
Using Theorem \ref{thm:ML_lossy_sc_rate_conv_bound}, we present the lower bound for $D(k,R,\epsilon, \mathcal{A})$ as follows.
The proof is in Appendix~\ref{apdx:proof:thm:ML_lossy_sc_distortion_conv_bound}.

\begin{theorem}
In addition to Assumptions (I)--(\Six), assume that $R(d,\mathcal{A})$ is twice differentiable with $\frac{\mathrm{d} \mathbb{R}_W(d,\mathcal{A})}{\mathrm{d}d}\neq 0$, and that $V(d,\mathcal{A})$ is differentiable on the interval $(\underline{d},\bar{d}] \subseteq (d_{\mathrm{min}}, d_{\mathrm{max}}]$. 
Then, for any rate $R$ satisfying $R=R(d,\mathcal{A})$ for some $d \in (\underline{d},\bar{d})$, the following holds:
\begin{align}
\begin{split}
    \textstyle D(k,R,\epsilon, \mathcal{A}) 
    &\textstyle \ge D(R,\mathcal{A}) + \sqrt{\frac{\mathcal{V}(R,\mathcal{A})}{k}}Q^{-1}(\epsilon) -D'(R,\mathcal{A})\Delta_k(D(R,\mathcal{A}),\epsilon) +  O \big( \frac{\log k}{k} \big).
\end{split}
  \label{eq:ML_lossy_sc_distortion_conv_bound}
\end{align}
Here, $D(R,\mathcal{A}) := \lim_{k \rightarrow \infty} D(k,R,\mathcal{A})$ and $\mathcal{V}(R,\mathcal{A}) := (D'(R,\mathcal{A}))^2 V(D(R,\mathcal{A}), \mathcal{A})$, where $D(k,R,\mathcal{A})$ denotes the inverse function of $R(k,d,\mathcal{A})$.
\label{thm:ML_lossy_sc_distortion_conv_bound}
\end{theorem}

$D(R,\mathcal{A})$, known as the distortion-rate function, represents the asymptotic limit as $k\rightarrow\infty$ and is the inverse of $R(d,\mathcal{A})$.
The term $\mathcal{V}(R,\mathcal{A})$, called the distortion-dispersion function, reflects the uncertainty due to finite $k$.
Analogous to Theorem \ref{thm:ML_lossy_sc_rate_conv_bound}, Theorem \ref{thm:ML_lossy_sc_distortion_conv_bound} establishes a lower bound on the generalization error realized when selecting the optimal set of $n$ samples for a specified $\mathcal{A}$ and $n$.
Since $\mathcal{V}(R,\mathcal{A})$ is determined by $V(d, \mathcal{A})$ discussed in the preceding section, the decompositions and interpretations applied to $V(d, \mathcal{A})$ remain directly applicable here.
Thus, this framework enables us to express the degree of overfitting of $\mathcal{A}$ and its inductive bias mismatch influence the generalization error lower bound, just as in Section \ref{subsec:lc_sample_comp_bound}.

\textbf{Further Results: }
While Theorem~\ref{thm:ML_lossy_sc_rate_conv_bound} and Theorem~\ref{thm:ML_lossy_sc_distortion_conv_bound} provides theoretical limits for arbitrary sampling strategies, its derivation can be extended to characterize theoretical limits within specific classes of sampling strategies, such as pool-based active learning.
See Appendix~\ref{apdx:sec_limiting_sampling_strategy} for details.

\section{Related Work}
\label{sec:related_work}

Our lower bounds relate to IT-bounds \cite{xu2017information,hellstrom2020generalization} and stability theory \cite{bousquet2002stability,elisseeff2005stability}.
These frameworks establish upper bounds of generalization error that incorporate the properties of $\mathcal{A}$.
Our lower bound of sample complexity parallels upper bounds of generalization error by characterizing the relationship between the required error and the necessary data size.
Moreover, our lower bound of generalization error identifies the limit of the error at a specific training data size, which is a result that upper bounds cannot provide.

Our framework offers a novel perspective distinct from existing analyses of theoretical limits in active learning.
While existing theories evaluate sample complexity and generalization error using hypothesis complexity, distributional noise, and Bayes risk to enable detailed analyses \cite{kaariainen2006active,hanneke2007abound,castro2008minimax,raginsky2011lower,hanneke2014theory,hanneke2015minimax,yuan2024regimes,hanneke2025agnostic}, our framework characterizes these bounds via properties of the learning algorithm itself, like overfitting and inductive bias mismatch.

Lossy compression finds applications in frameworks other than \cite{nokleby2016rate,jeon2022an}.
Specifically, \cite{sefidgaran2022rate} evaluates generalization by treating symbols as predictive models and codewords as their compressed forms, whereas \cite{pereg2023information} evaluates sample complexity by linking symbols to training data and codewords to latent representations of encoder-decoder models.
While both of these approaches utilize the asymptotic theory of lossy compression, our use of finite blocklength analysis (non-asymptotic theory) allows us to incorporate various properties of $\mathcal{A}$ into the analysis.

Works such as \cite{zhou2020second,kuramata2022analysis,behboodi2025fundamental} apply finite blocklength analysis to learning theory.
These studies formulate learning as a hypothesis testing problem, applying the analysis of channel coding introduced by \cite{polyanskiy2010channel}.
While these approaches focus on methods viewable as hypothesis testing, our framework enables the analysis of arbitrary randomized learning algorithms.

\section{Conclusion}
\label{sec:conclusion}

We propose a novel theoretical framework by mapping machine learning processes to fixed-length lossy compression and leveraging finite blocklength analysis.
We established a detailed correspondence, centering on mapping data sampling to encoding and learning to decoding, to a level that permits the application of the finite blocklength analysis.
This framework encompasses both the sampling strategy $\mathcal{S}$ and the learning algorithm $\mathcal{A}$, and evaluates lower bounds on sample complexity and generalization error for a given $\mathcal{A}$ under the optimal $\mathcal{S}^*$.
These lower bounds are characterized by terms that separately quantify $\mathcal{A}$'s overfitting and the mismatch between its inductive bias and the task.
Furthermore, it is shown that some of these quantities are related to metrics found in IT-bounds and stability theory.

\section*{Acknowledgment}
This work was supported in part by the Japan Society for the Promotion of Science through Grants-in-Aid for Scientific Research (C) (23K11111).

\bibliographystyle{unsrt}

\bibliography{reference}

\newpage
\appendix
\section{Detailed definitions of each concept and their interpretations}
\label{apdx:sec:lc_detail_def_interpret}

In this section, we elaborate on the definitions introduced in Section \ref{subsec:lc_def_key_concepts} and discuss their interpretations.

\subsection{Definition of ``Good'' Sampling Strategy}
\label{apdx:subsec:detail_def_lc_good_sampling}

\begin{definition}[Definition \ref{def:lc_k_n_d_e_A_sampling}]
  For sampling opportunities $k$, a $(k, n, d, \epsilon, \mathcal{A})$ sampling in $\{\mathcal{W}^k, \mathcal{H}, P_{W^k}, \mathsf{d}:\mathcal{W}^k\times\mathcal{H} \rightarrow [0,+\infty]\}$ is defined as a sampling strategy $\mathcal{S}$ that satisfies the training data size $|\mathcal{S}(W^k)|=n$ and the condition $\mathbb{P}[\mathsf{d}(W^k; \mathcal{A}(\mathcal{S}(W^k))) > d] \le \epsilon$, given a distortion level $d$, an excess distortion probability $\epsilon$, and a learning algorithm $\mathcal{A}$.
  In the case of $k=1$, it is denoted as $(n, d, \epsilon, \mathcal{A})$ sampling.
  We define the minimum training data size as $n^*(k, d, \epsilon, \mathcal{A}) := \min\{ n: \exists (k, n, d, \epsilon, \mathcal{A}) \text{ sampling} \}$ and the minimum rate as $R(k,d,\epsilon, \mathcal{A}) := \frac{b}{k}n^*(k,d,\epsilon, \mathcal{A})$.
\end{definition}

The class of ``good'' sampling strategies introduced in Definition \ref{def:lc_k_n_d_e_A_sampling} was characterized by the excess distortion probability condition $\mathbb{P}[\mathsf{d}(W^k; \mathcal{A}(\mathcal{S}(W^k))) > d] \le \epsilon$.
Next, we define a class of ``good'' sampling strategies based on the expected distortion condition $\mathbb{E}_{P_{W^k}}[\mathsf{d}(W^k, \mathcal{A}(\mathcal{S}(W^k)) )] \le d$, as follows.

\begin{definition}
  For sampling opportunities $k$, a $\langle k, n, d, \mathcal{A} \rangle$ sampling in $\{\mathcal{W}^k, \mathcal{H}, P_{W^k}, \mathsf{d}:\mathcal{W}^k\times\mathcal{H} \rightarrow [0,+\infty]\}$ is defined as a sampling strategy $\mathcal{S}$ that satisfies the training data size $|\mathcal{S}(W^k)|=n$ and the condition $\mathbb{E}_{P_{W^k}}[\mathsf{d}(W^k, \mathcal{A}(\mathcal{S}(W^k)) )] \le d$, given a distortion level $d$, an excess distortion probability $\epsilon$, and a learning algorithm $\mathcal{A}$.
  In the case of $k=1$, it is denoted as $\langle n, d, \mathcal{A} \rangle$ sampling.
  We define the minimum training data size as $n^*(k, d, \mathcal{A}) := \min\{ n: \exists \langle k, n, d, \mathcal{A} \rangle \text{ sampling} \}$ and the minimum rate as $R(k,d, \mathcal{A}) := \frac{b}{k}n^*(k,d,\mathcal{A})$.
  \label{def:lc_k_n_d_A_sampling}
\end{definition}

\subsection{Definition of Distortion-Rate Function}
\label{apdx:subsec:detail_def_lc_dist_rate_func}

The functions $D(k,R,\epsilon, \mathcal{A})$, $D(k,R,\mathcal{A})$, and $D(R,\mathcal{A})$ are defined as the minimum distortion achievable at a given rate $R$, corresponding to $R(k,d,\epsilon, \mathcal{A})$, $R(k,d,\mathcal{A})$, and $R(d,\mathcal{A})$, respectively.
In particular, $D(R,\mathcal{A})$ is referred to as the \textit{distortion-rate function} and is the inverse of $R(d,\mathcal{A})$ \cite{cover1999elements}:
\begin{align}
    D(k,R,\epsilon, \mathcal{A}) &:= \min \{d: \exists(k,Rk/b, d,\epsilon,\mathcal{A}) \text{~sampling} \} \label{eq:D_k_R_eps_A}, \\
    D(k,R,\mathcal{A}) &:= \min \{d: \exists \langle k,Rk/b, d, \mathcal{A} \rangle \text{~sampling} \} \label{eq:D_k_R_A}, \\
    D(R,\mathcal{A}) &:= \limsup_{k \rightarrow \infty} D(k,R,\mathcal{A}). \label{eq:dist_rate_func}
\end{align}
For fixed $k$ and $\epsilon$, $D(k,\cdot,\epsilon, \mathcal{A})$ and $R(k,\cdot,\epsilon,\mathcal{A})$ are inverses of each other \cite{kostina2012fixed}.

\subsection{Interpretation of Rate-Distortion Function}
\label{apdx:subsec:detail_interpret_lc_rate_dist_func}

In this section, we discuss the interpretation of the rate-distortion function $R(d,\mathcal{A})$ within our framework.
From Theorem~\ref{thm:rate_distortion_theorem}, $R(d,\mathcal{A})$ is equivalent to $\mathbb{R}_{W}(d,\mathcal{A})$ as defined in Eq.~\eqref{eq:lc_R_W}:
\begin{align}
\begin{split}
  &\mathbb{R}_W(d, \mathcal{A}) 
  := \inf_{\mathcal{S}: \mathbb{E}[\mathsf{d}(W;\mathcal{A}(\mathcal{S}(W)))] \le d}  \mathbb{I}(W;T) = \inf_{P^{\mathcal{S}}_{T|W}: \mathbb{E}_{P^{\mathcal{A}}_{H|T}P^{\mathcal{S}}_{T|W} P_W } [\mathsf{d}(W;H)] \le d}  \mathbb{I}(W;T).
\end{split}
\end{align}

To gain further insight into $R(d,\mathcal{A})$, we investigate the mutual information term $\mathbb{I}(W;T)$ appearing in the definition of $\mathbb{R}_W(d, \mathcal{A})$.
Given the Markov chain $W \rightarrow T \rightarrow H$, this term $\mathbb{I}(W;T)$ can be decomposed as follows:
\begin{theorem}
  For any distributions $P_W$, $P^{\mathcal{S}}_{T|W}$, and $P^{\mathcal{A}}_{H|T}$, the following holds:
  \begin{align}
    \mathbb{I}(W;T) = \mathbb{I}(T;H) + \mathbb{I}(W;T|H) - \mathbb{I}(T;H|W)  \label{eq:I-WT_expansion}. 
  \end{align}
\label{thm:rate_dist_MI_transform}
\end{theorem}

\begin{proof}[Proof of Theorem \ref{thm:rate_dist_MI_transform}]
By the chain rule for mutual information, we have:
\begin{align}
    \mathbb{I}(T;W,H) &= \mathbb{I}(T;H) + \mathbb{I}(T;W|H), \label{eq:chain1} \\
    \mathbb{I}(T;W,H) &= \mathbb{I}(T;W) + \mathbb{I}(T;H|W). \label{eq:chain2}
\end{align}
From these equations, we obtain:
\begin{align}
    \mathbb{I}(T;H) + \mathbb{I}(W;T|H) = \mathbb{I}(W;T) + \mathbb{I}(T;H|W). \notag
\end{align}
Rearranging the following holds:
\begin{align}
    \mathbb{I}(W;T) = \mathbb{I}(T;H) + \mathbb{I}(W;T|H) - \mathbb{I}(T;H|W). \notag
\end{align}

\end{proof}

In the following, we provide an interpretation for each term in Eq.~\eqref{eq:I-WT_expansion}.
To begin with, $\mathbb{I}(W;T)$ on the left-hand side quantifies the extent to which the training dataset $T$ preserves information about the sub-distribution $W$.
In other words, this term measures the average capability of $\mathcal{S}$ to acquire information regarding $W$.

The first term on the right-hand side, $\mathbb{I}(T;H)$, represents the strength of the correlation between the dataset $T$ sampled by $\mathcal{S}$ and the hypothesis $H$ obtained from $T$ via $\mathcal{A}$.
That is, this quantity can be interpreted as the degree of dependence of the learned hypothesis $H$ on the training dataset $T$.
In other words, it quantifies how likely the learning algorithm $\mathcal{A}$ is to output a hypothesis that overfits the training data.
This term can be shown to correspond to the quantity that upper-bounds the generalization error in the IT-bound presented by \cite{xu2017information}.
This correspondence is crucial for establishing the relationship between existing IT-bounds and our proposed framework, and it is discussed in detail in Appendix~\ref{apdx:subsec_bridge_rate_dist_others}.

The second term on the right-hand side, $\mathbb{I}(W;T|H)$, represents the amount of additional information about $W$ available from $T$, given the information about $W$ already captured by $H$.
That is, this quantity represents the amount of information about $W$ contained in $T$ that is not captured by the hypothesis $H$.
Intuitively, simpler learning algorithms $\mathcal{A}$ are prone to greater information loss, leading to a larger value of $\mathbb{I}(W;T|H)$.
For example, if $\mathcal{A}$ is a simple model such as linear regression, it may fail to capture non-linear dependencies in the data, resulting in significant information loss.
Conversely, if $\mathcal{A}$ employs a complex model like a deep neural network, it can extract more information from $T$, thereby keeping $\mathbb{I}(W;T|H)$ small.
Consequently, $\mathbb{I}(W;T|H)$ can be interpreted as a metric reflecting the complexity of $\mathcal{A}$.

The third term on the right-hand side, $\mathbb{I}(T;H|W)$, represents the magnitude of information shared between $T$ and $H$ that is distinct from the information contained in $W$.
That is, this quantity can be interpreted as a measure of the extent to which $H$ captures noise specific to $T$ that is irrelevant to $W$.
It is expected that as the capacity of the learning algorithm $\mathcal{A}$ decreases, the tendency to overfit such noise diminishes, leading to a smaller value of $\mathbb{I}(T;H|W)$.
For instance, when $\mathcal{A}$ employs a robust linear regression model, $\mathbb{I}(T;H|W)$ tends to be small because such models struggle to capture locally non-linear noise that is irrelevant to $W$.
Conversely, for high-capacity algorithms such as deep learning models, $\mathbb{I}(T;H|W)$ can be large depending on the methodology, as such models have the potential to memorize the training data \cite{zhang2017understanding}.

\subsection{Supplementary Details on the Interpretation of Tilted Information}
\label{apdx:subsec:detail_interpret_lc_tilted}

In Section \ref{subsec:lc_def_key_concepts}, we defined the $\mathcal{A}$-specified $d$-tilted information and provided its interpretation.
Specifically, we discussed the individual terms composing $\jmath_{W}$ and interpreted $\jmath_{W}(w,t,h,d,\mathcal{A})$ as the minimum training data size (scaled by $b$ bits) required to derive $h$ from $w$ via $\mathcal{S}^*$ and $\mathcal{A}$, subject to the required distortion level $d$:
\begin{equation}
    \jmath_{W}(w,t, h,d,\mathcal{A}) := \iota_{W;T^{\mathcal{S}^{*}}}(w;t) + \lambda^{*}_{\mathcal{A}}(d)(\mathsf{d}(w;h) - d).
  \tag{Eq.~\eqref{eq:ML_lc_dfn_d-tilted}}
  \end{equation}

In this section, we expand on the previous discussion by analyzing how $\jmath_{W}(w,t,h,d,\mathcal{A})$ varies depending on the learning algorithm $\mathcal{A}$.
The algorithm $\mathcal{A}$ primarily influences $\lambda^{*}_{\mathcal{A}}(d)$.

\textbf{Relationship between $\mathcal{A}$ and $\lambda^{*}_{\mathcal{A}}(d)$.}
Next, we elucidate the relationship between $\mathcal{A}$ and $\lambda^{*}_{\mathcal{A}}(d)$.
Recall that $\lambda^{*}_{\mathcal{A}}(d)$ is defined as the negative slope of the rate-distortion function, $\lambda^{*}_{\mathcal{A}}(d)=-\frac{\mathrm{d} \mathbb{R}_W(d,\mathcal{A})}{\mathrm{d}d}$, where $\mathbb{R}_W(d,\mathcal{A})~(=R(d,\mathcal{A}))$ denotes the minimum rate required to achieve an expected distortion of at most $d$.
This implies that $\lambda^{*}_{\mathcal{A}}(d)$ quantifies the marginal increase in the minimum rate necessitated by a slight reduction in the required distortion level $d$.
Generalization error analysis based on uniform convergence establishes that a higher complexity of the hypothesis class of $\mathcal{A}$ necessitates a larger training data size to reduce generalization error \cite{mohri2018foundations}.
Similarly, other theoretical frameworks, such as IT-bounds \cite{xu2017information} and stability theory \cite{bousquet2002stability,elisseeff2005stability}, arrive at comparable conclusions by examining the degree of overfitting of $\mathcal{A}$.
Building on these insights, we posit that as the complexity and overfitting tendency of $\mathcal{A}$ increase, the incremental rate required to tighten the distortion constraint $d$ becomes steeper.
In other words, a more complex and overfitting-prone $\mathcal{A}$ leads to a larger value of $\lambda^{*}_{\mathcal{A}}(d)$.
Thus, we conclude that $\lambda^{*}_{\mathcal{A}}(d)$ likewise characterizes the complexity of the hypothesis class and the degree of overfitting associated with $\mathcal{A}$.

\section{Details and subsequent results for Theorem~\ref{thm:ML_lossy_sc_rate_conv_bound}}
\label{apdx:sec:detail_thm:ML_lossy_sc_rate_conv_bound}

\subsection{Preparation}
\label{apdx:subsec:detail_thm:ML_lossy_sc_rate_conv_bound_preparation}

In this section, we present the full statement of Theorem~\ref{thm:ML_lossy_sc_rate_conv_bound} and provide the discussion that was omitted in Section~\ref{subsec:lc_sample_comp_bound}.

In finite blocklength analysis, it is strictly assumed that the encoder $\mathcal{S}$ and decoder $\mathcal{A}$ can be applied to each symbol $W$.
However, the primary object of analysis in this framework is $H = \mathcal{A}(\mathcal{S}(W^k))$, which does not satisfy this assumption.
To facilitate the analysis, we therefore consider two learning schemes: the primary \textit{global scheme}, defined by $H = \mathcal{A}(\mathcal{S}(W^k))$, and a \textit{local scheme}, defined by $\{H_i = \mathcal{A}(T_i), T_i=\mathcal{S}(W_i)\}_{i=1}^k$.
In the proof, we derive the theoretical limits for the global scheme by applying finite-length analysis techniques to the local scheme while characterizing the relationship between the two schemes.

First, we introduce the necessary definitions.
As specified in Definition \ref{def:lc_k_n_d_e_A_sampling}, the minimum number of training samples and the minimum rate for a $(k,n,d,\epsilon, \mathcal{A})$-sampling under the global scheme are denoted by $n^*(k, d, \epsilon, \mathcal{A})$ and $R(k,d,\epsilon, \mathcal{A})$, respectively.
In contrast, for the local scheme, where the distortion constraint is defined as $\mathbb{P}[\frac{1}{k}\sum^k_{i=1}\mathsf{d}(W_i; \mathcal{A}(\mathcal{S}(W_i))) > d] \le \epsilon$, the corresponding minimum number of samples and minimum rate are denoted as $\bar{n}^*(k, d, \epsilon, \mathcal{A})$ and $\widebar{R}(k,d,\epsilon, \mathcal{A})$, respectively.

Let $T_{\mathrm{glob}}$ and $H_{\mathrm{glob}}$ denote the random variables representing the training dataset and the hypothesis obtained under the global scheme, with their respective realizations denoted by $t_{\mathrm{glob}}$ and $h_{\mathrm{glob}}$.
For the local scheme, we define the sequences of random variables for the training datasets and hypotheses as $T^k=(T_1,\dots,T_k)$ and $H^k=(H_1,\dots,H_k)$, with realizations $t^k=(t_1,\dots,t_k)$ and $h^k=(h_1,\dots,h_k)$, respectively.
Furthermore, in the global scheme, we define $\mathcal{S}^*_{\mathrm{glob}}$ as the optimal sampling strategy that achieves the infimum of $\inf_{\mathcal{S}: \mathbb{E}[\mathsf{d}(W^k;\mathcal{A}(\mathcal{S}(W^k))] \le d}\mathbb{I}(W^k;T)$.
Under these definitions, we define the $\mathcal{A}$-specified $d$-tilted information for each scheme.
First, for $k>1$, the tilted information $\jmath_{\mathrm{glob}}$ in the global scheme is defined as follows:
\begin{align}
    \jmath_{\mathrm{glob}}(w^k,t,h,d,\mathcal{A}) 
    := \iota_{W;T^{\mathrm{S}^*_{\mathrm{glob}}}}(w^k;t) + k\lambda^{*\mathrm{glob}}_{\mathcal{A}}(d)(\mathsf{d}(w^k;h) -d),
\end{align}
where $\lambda^{*\mathrm{glob}}_{\mathcal{A}}(d)$ is the negative gradient of $\frac{1}{k}\mathbb{I}(W^k;T^{\mathcal{S}^*_{\mathrm{glob}}})$ with respect to $d$ and $T^{\mathcal{S}^*_{\mathrm{glob}}} \sim P^{\mathcal{S}^*_{\mathrm{glob}}}_{T|W}$.
Based on Eq.~\eqref{eq:ML_lc_dfn_d-tilted}, the tilted information $\jmath_{\mathrm{local}}$ for the local scheme with $k>1$ is defined as:
\begin{align}
    \jmath_{\mathrm{local}}(w^k,t^k,h^k,d,\mathcal{A}) 
    := \jmath_{W}(w^k,t^k,h^k,d,\mathcal{A}) 
    := \iota_{W;T^{\mathrm{S}^*}}(w^k;t^k) + k\lambda^{*}_{\mathcal{A}}(d)(\mathsf{d}(w^k;h^k) -d).
\end{align}
In the local scheme, the relations $P^{\mathcal{S}^*}_{T^k|W^k} = \prod_{i=1}^k P^{\mathcal{S}^*}_{T_i|W_i}$ and $P^{\mathcal{A}}_{H^k|T^k} = \prod_{i=1}^k P^{\mathcal{A}}_{H_i|T_i}$ hold.
Consequently, we have $\iota_{W;T^{\mathcal{S}^*}}(w^k;t^k) = \sum_{i=1}^k \iota_{W;T^{\mathcal{S}^*}}(w_i;t_i)$.
Furthermore, by definition, $\mathsf{d}(w^k;h^k) = \frac{1}{k}\sum_{i=1}^k \mathsf{d}(w_i;h_i)$.
Therefore, the following holds:
\begin{align}
    \jmath_{\mathrm{local}}(w^k,t^k,h^k,d,\mathcal{A})
    &= \sum^k_{i=1}\iota_{W;T^{\mathcal{S}^*}}(w_i;t_i) + \sum^k_{i=1}\lambda^{*}_{\mathcal{A}}(d)(\mathsf{d}(w_i;h_i) -d) \notag \\
    &= \sum^k_{i=1} \jmath_{W}(w_i, t_i, h_i, d, \mathcal{A}).
\end{align}
We define the joint distribution $P_{\mathrm{joint}}$ when the global and local schemes are executed simultaneously for $w^k$ as:
\begin{align}
\begin{split}
&P_{\mathrm{joint}}(w^k, t^k, h^k, t_{\mathrm{glob}}, h_{\mathrm{glob}}) \\
    &:=  P_{W^k}(w^k) P_{T|W^k}^{\mathcal{S}_{\mathrm{glob}}^{**}(k,\epsilon)}(t_{\mathrm{glob}}|w^k) P_{H|T}^{\mathcal{A}}(h_{\mathrm{glob}}|t_{\mathrm{glob}}) \prod_{i=1}^k \left[ P_{T|W}^{\mathcal{S}^*}(t_i|w_i) P_{H|T}^{\mathcal{A}}(h_i|t_i) \right].
\end{split}
\end{align}
Here, $\mathcal{S}_{\mathrm{glob}}^{**}(k,\epsilon)$ denotes the sampling strategy that achieves the minimum rate $R(k,d,\epsilon,\mathcal{A})$ in global scheme.
Under this distribution, let $\Delta\jmath := \jmath_{\mathrm{glob}}(w^k,t,h,d,\mathcal{A}) - \jmath_{\mathrm{local}}(w^k,t^k,h^k,d,\mathcal{A})$ be the difference between the tilted information of the two schemes.
We define its mean and variance as $\mu_{\Delta}=\mathbb{E}_{P_{\mathrm{joint}}}[\Delta\jmath]$ and $V_{\Delta} = \mathrm{Var}_{P_{\mathrm{joint}}}(\Delta \jmath)$, respectively.

\subsection{Detail Statement and Interpretation}
\label{apdx:subsec:detail_thm:ML_lossy_sc_rate_conv_bound_statement_interpretation}

Next, the complete statement of Theorem~\ref{thm:ML_lossy_sc_rate_conv_bound} is given below.
The proof is provided in Appendix~\ref{apdx:proof:thm:ML_lossy_sc_rate_conv_bound}.
\begin{theorem}[Detail version of Theorem \ref{thm:ML_lossy_sc_rate_conv_bound}]
Under Assumptions (I)--(\Six), for any $k$, $\mathcal{A}$, $d \in (d_{\mathrm{min}},d_{\mathrm{max}})$ and $\epsilon \in (0,1)$:
\begin{align}
\begin{split}
    &R(k,d,\epsilon, \mathcal{A}) \\
    &\ge R(d,\mathcal{A}) + \sqrt{\frac{V(d,\mathcal{A})}{k}} Q^{-1}(\epsilon)
    + \underbrace{\frac{\mu_{\Delta}}{k} - \frac{3}{2 k^{\frac{5}{6}}}\bigg(2 V_{\Delta} \frac{\sqrt{V(d,\mathcal{A})}}{\phi(Q^{-1}(\epsilon))} \bigg)^{\frac{1}{3}} }_{=: \Delta_k(d,\epsilon)}
    + O\bigg(\frac{\log k}{k} \bigg).       
\end{split}
\label{eq:ML_lossy_sc_rate_conv_bound_detail}
\end{align}
\end{theorem}

The meanings of $R(d,\mathcal{A})$ and $V(d,\mathcal{A})$, as well as their relationships to existing theories, were outlined in Section~\ref{subsec:lc_sample_comp_bound}.
Further details regarding $R(d,\mathcal{A})$ are provided in Appendices~\ref{apdx:subsec:detail_interpret_lc_rate_dist_func} and \ref{apdx:subsec_bridge_rate_dist_others}, while $V(d,\mathcal{A})$ is discussed in detail in Appendix~\ref{apdx:sec:bridge_ours_others}.

In this section, we elaborate on the term corresponding to $\Delta_k$, which was omitted in Section~\ref{subsec:lc_sample_comp_bound}.
In short, the term $\Delta_k$ evaluates the sample efficiency of the global scheme relative to the local scheme when employing the learning algorithm $\mathcal{A}$.
This interpretation primarily stems from the first term of $\Delta_k$, namely $\frac{\mu_{\Delta}}{k}$. 
Recall that $\mu_{\Delta}$ is defined as $\mu_{\Delta}:=\mathbb{E}_{P_{\mathrm{joint}}}[\jmath_{\mathrm{glob}}(w^k,t,h,d,\mathcal{A}) - \jmath_{\mathrm{local}}(w^k,t^k,h^k,d,\mathcal{A})]$.
The expectation of the first term corresponds to the minimum rate in the global scheme, whereas that of the second term corresponds to the minimum rate in the local scheme.
Thus, $\mu_{\Delta}$ represents, on average, the reduction in the minimum required number of training samples for the global scheme compared to the local scheme.
Specifically, $\mu_{\Delta}$ takes a larger negative value when training on multiple regions $W^k$ collectively requires fewer samples to achieve the target distortion $d$ than training on each partial region $W_i$ independently.
For instance, in the context of deep learning, it is expected that training on collective regions $W^k$ achieves lower generalization error with fewer samples than training on isolated regions.
In such cases, $\mu_{\Delta} < 0$ and its magnitude $|\mu_{\Delta}|$ is expected to be large.
Conversely, if a linear regression model is applied to a task where the true relationship $P^*_{Y|\bm{X}}$ is highly non-linear, training on locally linear regions independently may yield better generalization performance than attempting to learn across multiple non-linear regions simultaneously.
In such a scenario, $\mu_{\Delta}$ is expected to be positive. 
Therefore, $\mu_{\Delta}$ quantifies the capability of the learning algorithm $\mathcal{A}$ to capture relationships between multiple regions and appropriately learn the task as a whole.

On the other hand, the second term of $\Delta_k$, namely $- \frac{3}{2 k^{\frac{5}{6}}}\big(2 V_{\Delta} \frac{\sqrt{V(d,\mathcal{A})}}{\phi(Q^{-1}(\epsilon))} \big)^{\frac{1}{3}}$, represents the variability in the difference in sample efficiency between the global and local schemes, a difference that was characterized by $\mu_{\Delta}$ on average.
In particular, $V_{\Delta}$ increases if the difficulty of global scheme (compared to local scheme) fluctuates significantly depending on the specific realization of the combined regions $W^k \sim \prod^k_{i=1}P_W$, and vice versa.

\textbf{Order of} $\Delta_k(d,\epsilon)$: 

Finally, we examine the order of these two terms. 
In $\Delta_k$, the two components that depend on $k$ are $\mu_{\Delta}$ and 
Regarding the orders of $\mu_{\Delta}$ and $V_{\Delta}$ with respect to $k$, the following result holds. The proof is provided in Appendix~\ref{apdx:subsec:proof_thm:V_delta_order}:
\begin{theorem}
    The order of $\mu_{\Delta}$ is $O(k)$, and the order of $V_{\Delta}$ is $O(k^2)$.
\label{thm:V_delta_order}
\end{theorem}

From Theorem \ref{thm:V_delta_order}, the order of $\frac{\mu_{\Delta}}{k}$ is $O(1)$.
Furthermore, based on Theorem \ref{thm:V_delta_order}, the term $- \frac{3}{2 k^{\frac{5}{6}}}\left(2 V_{\Delta} \frac{\sqrt{V(d,\mathcal{A})}}{\phi(Q^{-1}(\epsilon))} \right)^{\frac{1}{3}}$ is of order $O(k^{-1/6})$, which implies that $\lim_{k\rightarrow \infty}- \frac{3}{2 k^{\frac{5}{6}}}\left(2 V_{\Delta} \frac{\sqrt{V(d,\mathcal{A})}}{\phi(Q^{-1}(\epsilon))} \right)^{\frac{1}{3}}=0$ holds.

On the other hand, since Theorem \ref{thm:rate_distortion_theorem} implies that the asymptotic minimum rate satisfies $\limsup_{k\rightarrow\infty}R(k,d,\mathcal{A})= R(d,\mathcal{A})=\mathbb{R}_W(d,\mathcal{A})$, meaning that the asymptotic minimum rate of the global scheme $R(d,\mathcal{A})$ coincides with that of the local scheme $\mathbb{R}_W(d)$, it is expected that there exist conditions under which $\lim_{k\rightarrow \infty}\frac{\mu_{\Delta}}{k}=0$.
Theorem~\ref{thm:eq:ML_lossy_sc_eps_conv_bound} also holds for the global scheme by treating $W^k$ as a single symbol, taking the form $\epsilon \ge \sup_{\gamma \ge 0} \{\mathbb{P}[\jmath_{\mathrm{glob}}(W^k,T,H,d,\mathcal{A}) \ge bn + \gamma] - e^{-\gamma} \}$.
From this, it follows that $\jmath_{\mathrm{glob}}(W^k,T,H,d,\mathcal{A})$ corresponds to the minimum number of training samples $bn^*(k,d,\epsilon,\mathcal{A})$ for the global scheme at finite $k$.
As discussed in the main body, each element $\jmath_{W}(w_i, t_i, h_i, d, \mathcal{A})$ in the summation of $\jmath_{\mathrm{local}}$ also corresponds to the minimum number of training samples for $k=1$. 
Therefore, it is highly plausible that their expectations coincide as $k\rightarrow \infty$; that is, the condition $\lim_{k\rightarrow\infty}\frac{\mu_{\Delta}}{k}=0$ is likely to hold.

\subsection{Rate Distortion Theorem for $R(k,d,\epsilon,\mathcal{A})$}
\label{apdx:subsec:rate_distortion_theorem_RkdeA}

In this section, we establish a rate-distortion theorem for $R(k,d,\epsilon,\mathcal{A})$ based on Theorem~\ref{thm:ML_lossy_sc_rate_conv_bound}.
As a preliminary step, and analogous to Assumption (\Two), we define $\widebar{R}(k,d,\epsilon,\mathcal{A})$ as the minimum rate achieved when the learning process in Definition~\ref{def:lc_k_n_d_e_A_sampling} is replaced by the local scheme $\{H_i = \mathcal{A}(\mathcal{S}(W_i))\}^k_{i=1}$. 
Under the assumption that the global scheme yields a smaller or equal asymptotic minimum rate compared to the local scheme, specifically, $\limsup_{k\rightarrow\infty}R(k,d,\epsilon,\mathcal{A}) \le \limsup_{k\rightarrow\infty}\widebar{R}(k,d,\epsilon,\mathcal{A})$, the following result holds.
The proof is provided in Appendix~\ref{apdx:subsec:proof_thm:rate_distortion_theorem_epsilon}.

\begin{theorem}
    Suppose that, in addition to Assumptions (I) through (\Six), $\limsup_{k\rightarrow\infty}R(k,d,\epsilon,\mathcal{A}) \le \limsup_{k\rightarrow\infty}\widebar{R}(k,d,\epsilon,\mathcal{A})$ and $\limsup_{k \rightarrow \infty}\Delta_k(d,\epsilon)=0$.
    Then, for any $\epsilon \in (0,1)$, $d \in (d_{\mathrm{min}}, d_{\mathrm{max}})$, and $\mathcal{A}$, the following holds:
    \begin{align}
        \limsup_{k\rightarrow\infty}R(k,d,\epsilon,\mathcal{A}) = \mathbb{R}_W(d,\mathcal{A}).
    \end{align}
\label{thm:rate_distortion_theorem_epsilon}
\end{theorem}

\section{Alternative Version of Theorem~\ref{thm:ML_lossy_sc_rate_conv_bound}}
\label{apdx:sec:alternative_ver_rate_conv_bound}

This section demonstrates that an alternative lower bound for the minimum rate can be derived, where the terms $R(d,\mathcal{A})$ and $V(d,\mathcal{A})$ remain identical to those in Theorem~\ref{thm:ML_lossy_sc_rate_conv_bound}, while the $\Delta_k$ term is different.
First, we define the minimum rate for the learning scheme $\{H_i = \mathcal{A}(T_i), T_i=\mathcal{S}(W_i)\}_{i=1}^k$ (called local scheme).
Recall that $R(k,d,\epsilon,\mathcal{A})$ in Section~\ref{subsec:lc_def_key_concepts} was defined as the minimum rate such that the excess distortion probability satisfies $\mathbb{P}[\mathsf{d}(W^k; H) > d] \le \epsilon$ under the standard scheme $H=\mathcal{A}(\mathcal{S}(W^k))$, called global scheme (Definition~\ref{def:lc_k_n_d_e_A_sampling}). 
In contrast, we introduce the definition for the scheme $\{H_i = \mathcal{A}(\mathcal{S}(W_i))\}_{i=1}^k$ corresponding to Definition~\ref{def:lc_k_n_d_e_A_sampling} as follows:

\begin{definition}
  For sampling opportunities $k$, a $(k, n, d, \epsilon, \mathcal{A})$ local sampling in $\{\mathcal{W}^k, \mathcal{H}^k, P_{W^k}, \mathsf{d}:\mathcal{W}^k\times\mathcal{H}^k \rightarrow [0,+\infty]\}$ is defined as a sampling strategy $\mathcal{S}$ that satisfies the training data size $\sum^k_{i=1}|\mathcal{S}(W_i)|=n$ and the condition $\mathbb{P}[\sum^k_{i=1}\mathsf{d}(W_i; \mathcal{A}(\mathcal{S}(W_i))) > d] \le \epsilon$, given a distortion level $d$, an excess distortion probability $\epsilon$, and a learning algorithm $\mathcal{A}$.
  In the case of $k=1$, it is denoted as $(n, d, \epsilon, \mathcal{A})$ local sampling.
  We define the minimum training data size as $\bar{n}^*(k, d, \epsilon, \mathcal{A}) := \min\{ n: \exists (k, n, d, \epsilon, \mathcal{A}) \text{ local sampling} \}$ and the minimum rate as $\widebar{R}(k,d,\epsilon, \mathcal{A}) := \frac{b}{k}\bar{n}^*(k,d,\epsilon, \mathcal{A})$.
\label{def:lc_k_n_d_e_A_sampling_local}
\end{definition}

Under this definition, the following theorem holds.
The proof is provided in Appendix~\ref{apdx:subsec:proof_thm:ML_lossy_sc_rate_conv_bound_alternative}:

\begin{theorem}
  Under Assumptions (I)--(\Six), for any $k$, $\mathcal{A}$, $d \in (d_{\mathrm{min}},d_{\mathrm{max}})$ and $\epsilon \in (0,1)$:
  \begin{equation}
    \textstyle R(k,d,\epsilon, \mathcal{A}) \ge R(d,\mathcal{A}) + \sqrt{\frac{V(d,\mathcal{A})}{k}}Q^{-1}(\epsilon) + R_{\Delta}(k,d,\epsilon) + O\big(\frac{\log k}{k}\big).
  \label{eq:ML_lossy_sc_rate_conv_bound_alternative}
  \end{equation}
  Here,  $V(d, \mathcal{A}) := \mathrm{Var}_{=P^{\mathcal{A}}_{H|T}P^{\mathcal{S}^*}_{T|W} P_W}(\jmath_W(W,T,H,d,\mathcal{A}))$ and $R_{\Delta}(k,d,\epsilon):= R(k,d,\epsilon, \mathcal{A}) - \widebar{R}(k,d,\epsilon, \mathcal{A})$. 
  
\label{thm:ML_lossy_sc_rate_conv_bound_alternative}
\end{theorem}

$R_{\Delta}(k,d,\epsilon)$ shares the same interpretation as $\Delta_k(d,\epsilon)$ in Theorem~\ref{thm:ML_lossy_sc_rate_conv_bound}.
It takes a negative value when the minimum rate of the global scheme is lower than that of the local scheme.
Specifically, $R_{\Delta}(k,d,\epsilon)$ is negative when jointly sampling and learning from the entire sequence $W^k$ is more sample-efficient than performing these tasks independently for each $W_i$.
The greater this efficiency, the more significant the negative value of $R_{\Delta}(k,d,\epsilon)$.
On the other hand, unlike $\Delta_k(d,\epsilon)$, $R_{\Delta}(k,d,\epsilon)$ does not allow for a direct analysis of its order with respect to $k$. 
Consequently, Theorem~\ref{thm:ML_lossy_sc_rate_conv_bound} provides a more detailed characterization of the lower bound on the minimum rate.

\section{Relationship between Derived Lower Bounds and Existing Theoretical Frameworks}
\label{apdx:sec:bridge_ours_others}

In Section \ref{subsec:lc_sample_comp_bound}, we derived a lower bound on sample complexity and outlined the key terms appearing in it: the rate-distortion function $R(d,\mathcal{A})$ and the rate-dispersion function $V(d,\mathcal{A})~ (=V_{\mathrm{in}} + V_{\mathrm{bet}})$.
This section provides a detailed interpretation of $R(d,\mathcal{A})$, $V_{\mathrm{in}}$, and $V_{\mathrm{bet}}$, which were only briefly introduced in Section \ref{subsec:lc_sample_comp_bound}, and elucidates their relationships with existing theoretical frameworks.

As a preliminary step, we define the i.i.d. sampling $\mathcal{S}^{\mathrm{iid}}$ as follows:
\begin{definition}
    We define the i.i.d. sampling $\mathcal{S}^{\mathrm{iid}}$ as a sampling strategy satisfying the following conditions:
    (i) For a given number of sampling opportunities $k$, the size $n$ of the generated training dataset $T$ is equal to $k$.
    (ii) When $k>1$, the dataset $T=\{(\bm{X}_i,Y_i)\}^k_{i=1}$ is sampled according to:
    \begin{align}
        (\bm{X}_i,Y_i) \sim P^*_{\bm{X}Y|W=w_i}, ~~~ \forall i \in \{1,\dots,k\}.
    \end{align}
    We assume the generation of $T$ by $\mathcal{S}^{\mathrm{iid}}$ is represented by the Markov kernel $P^{\mathcal{S}^{\mathrm{iid}}}_{T|W}$.
    We denote the training dataset as $T^{\mathcal{S}^{\mathrm{iid}}}$ when explicitly indicating that $T$ is generated from $P^{\mathcal{S}^{\mathrm{iid}}}_{T|W}$.
    Similarly, we use the notation $H^{\mathcal{S}^{\mathrm{iid}}}_{T|W}$ to specify that $H$ is generated according to the distribution $P^{\mathcal{A}}_{H|T}P^{\mathcal{S}^{\mathrm{iid}}}_{T|W}$.
\label{def:iid_sampling}
\end{definition}
Due to the i.i.d. property of $W$, the dataset $T=\{(\bm{X}_i,Y_i)\}^n_{i=1}$ sampled via $\mathcal{S}^{\mathrm{iid}}$ can be regarded as a collection of samples obtained independently and identically from $P^*_{\bm{X}Y}$.
Furthermore, we denote the i.i.d. sampling of $n$ samples without explicit reference to the sub-distribution $W$ as $T \sim P^*_{\bm{X}^nY^n}=P^*_{\bm{X}Y} \times\cdots\times P^*_{\bm{X}Y}$.

\subsection{Rate-Distortion Function $R(d,\mathcal{A})$}
\label{apdx:subsec_bridge_rate_dist_others}

The interpretation of $R(d,\mathcal{A})$ was detailed in Section \ref{apdx:subsec:detail_interpret_lc_rate_dist_func}.
Here, we focus on explaining the relationship between $R(d,\mathcal{A})$ and other theoretical frameworks.
Combining Eq. \eqref{eq:lc_R_W} in Theorem~\ref{thm:rate_distortion_theorem} with Eq. \eqref{eq:I-WT_expansion} in Theorem \ref{thm:rate_dist_MI_transform}, $R(d,\mathcal{A})$ can be expressed as:
\begin{align}
    R(d,\mathcal{A})=\mathbb{I}(W;T^{\mathcal{S}^*}) = \mathbb{I}(T^{\mathcal{S}^*};H) + \mathbb{I}(W;T^{\mathcal{S}^*}|H) - \mathbb{I}(T^{\mathcal{S}^*};H|W).
\label{eq:rep_R_MI}
\end{align}
where $T^{\mathcal{S}^*}$ denotes a training dataset distributed according to $P^{\mathcal{S}^*}_{T|W}$.

In summary, this expression links $R(d,\mathcal{A})$ to $\mathbb{I}(T^{\mathcal{S}^{\mathrm{iid}}};H)$, a quantity that governs the generalization error upper bound in the IT-bound established by \cite{xu2017information}:
\begin{theorem}[Theorem 1 in \cite{xu2017information}]
    Let the training dataset $T$ consist of $n$ samples drawn i.i.d. from $P^*_{\bm{X}Y}$.
    Assuming that $l(h(\bm{X}),Y)$ is $\sigma$-subgaussian for any $h \in \mathcal{H}$ under $P^*_{\bm{X}Y}$, the following holds:
    \begin{align}
        |\mathrm{average~ test~ error} - \mathrm{average~ train~ error}| \le \sqrt{\frac{2\sigma^2}{n}\mathbb{I}(T^{\mathcal{S}^{\mathrm{iid}}};H)},
    \end{align}
    where $\mathbb{I}(T^{\mathcal{S}^{\mathrm{iid}}};H)$ is computed with respect to the joint distribution $P^{\mathcal{A}}_{H|T}P^*_{\bm{X}^n Y^n}$.
\end{theorem}

Specifically, Theorem~\ref{thm:rate_distortion_theorem} implies that $R(d,\mathcal{A})$ is equal to $\mathbb{I}(W;T^{\mathcal{S}^*})$, which is the infimum value of $\mathbb{I}(W;T)$ among all sampling schemes that can achieve an average distortion of at most $d$ using $\mathcal{A}$.
Thus, it can be upper-bounded as follows:
\begin{align}
    R(d,\mathcal{A}) = \mathbb{I}(W;T^{\mathcal{S}^*}) \le \mathbb{I}(W; T^{\mathcal{S}^{\mathrm{iid}}}).
\end{align}
Here, the number of samples $n$ drawn by $\mathcal{S}^{\mathrm{iid}}$ is required to be sufficiently large such that the average distortion can be maintained below $d$ using $\mathcal{A}$.
Then, from Theorem~\ref{thm:rate_dist_MI_transform}, the following holds:
\begin{align}
    R(d,\mathcal{A}) \le \mathbb{I}(T^{\mathcal{S}^{\mathrm{iid}}};H) + \mathbb{I}(W;T^{\mathcal{S}^{\mathrm{iid}}}|H) - \mathbb{I}(T^{\mathcal{S}^{\mathrm{iid}}};H|W)
\end{align}
Consequently, the upper bound of $R(d,\mathcal{A})$ includes $\mathbb{I}(T^{\mathcal{S}^{\mathrm{iid}}};H)$, a term that appears in the IT-bound proposed by \cite{xu2017information}. 
This indicates that the lower bound in Theorem~\ref{thm:ML_lossy_sc_rate_conv_bound}, which involves $R(d,\mathcal{A})$, has a direct connection to the IT-bound of \cite{xu2017information}.

Based on the discussion in Appendix~\ref{apdx:subsec:detail_interpret_lc_rate_dist_func}, $\mathbb{I}(W;T^{\mathcal{S}^{\mathrm{iid}}}|H)$ represents the amount of information regarding $W$ that $H$ failed to capture from $T^{\mathcal{S}^{\mathrm{iid}}}$.
Meanwhile, $\mathbb{I}(T^{\mathcal{S}^{\mathrm{iid}}};H|W)$ quantifies the extent to which $H$ memorized noise within $T^{\mathcal{S}^{\mathrm{iid}}}$ that is irrelevant to $W$.
Therefore, $R(d, \mathcal{A})$ can be upper-bounded by the IT-bound term $\mathbb{I}(T^{\mathcal{S}^{\mathrm{iid}}};H)$ and the two aforementioned intuitive terms that characterize the properties of $\mathcal{A}$.

\subsection{The First Term of the Rate-Dispersion Function: $V_{\mathrm{in}}$}
\label{apdx:subsec_bridge_Vin_others}

In Section \ref{subsec:lc_sample_comp_bound}, we briefly outlined the constituents of $V_{\mathrm{in}}$ and interpreted it as a metric evaluating the degree of overfitting exhibited by $\mathcal{A}$.
This section elaborates on the relationship between $V_{\mathrm{in}}$ and existing theoretical frameworks.

We begin by restating the decomposition of $V_{\mathrm{in}}$:
\begin{align}
\begin{split}
  &\textstyle V_{\mathrm{in}}(d,\mathcal{A}) = \mathbb{E}_{P_W}\Big[  
    \underbracket[0.5pt]{\mathrm{Var}_{P^{\mathcal{S}^*}_{T|W}} \Big(\iota_{W;T^{\mathcal{S}^*}}(W;T) \Big)}_{=:V_{\mathrm{in}, \mathcal{S}}^{ \iota}} \\
  &\textstyle~~~~~~~~~~~~~~~~~~ 
    + (\lambda^{*}_{\mathcal{A}}(d))^2\underbracket[0.5pt]{\mathrm{Var}_{P^{\mathcal{S}^*}_{T|W}} \Big( \mathbb{E}_{P^{\mathcal{A}}_{H|T}} \Big[ \mathsf{d}(W;H) \Big]  \Big)  }_{=:V_{\mathrm{in}, \mathcal{S}}^{ \mathsf{d}}}
    + (\lambda^{*}_{\mathcal{A}}(d))^2 \underbracket[0.5pt]{    
      \mathbb{E}_{P^{\mathcal{S}^*}_{T|W}} \Big[ \mathrm{Var}_{P^{\mathcal{A}}_{H|T}} \Big( \mathsf{d}(W;H) \Big)  \Big]}_{=:V_{\mathrm{in}, \mathcal{A}}^{ \mathsf{d}}} \\
  &\textstyle~~~~~~~~~~~~~~~~~~
        +  2\lambda^{*}_{\mathcal{A}}(d) \underbracket[0.5pt]{  \mathrm{Cov}_{P^{\mathcal{A}}_{H|T}P^{\mathcal{S}^*}_{T|W}} \Big(\iota_{W;T^{\mathcal{S}^*}}(W;T), \mathsf{d}(W;H) \Big) }_{=:V_{\mathrm{in}}^{ \mathrm{cov}}}
  \Big].
\end{split}
\label{eq:ML_lossy_sc_V_in_5terms}
\end{align}

In the following, we examine each of the components: $V_{\mathrm{in}, \mathcal{S}}^{ \iota}$, $V_{\mathrm{in}, \mathcal{S}}^{ \mathsf{d}}$, and $V_{\mathrm{in}, \mathcal{A}}^{ \mathsf{d}}$.
Since our analysis focuses on the terms inside the expectation $\mathbb{E}_{P_W}[\cdot]$, we assume the subsequent arguments hold for an arbitrary fixed $w \in \mathcal{W}$.

\subsection{$V_{\mathrm{in}, \mathcal{S}}^{ \iota}$}
\label{apdx:subsec:V_in_S_A_density}

\textbf{Meaning of $V_{\mathrm{in}, \mathcal{S}}^{ \iota}$.}
$V_{\mathrm{in}, \mathcal{S}}^{ \iota}$ measures the variance of $\iota_{W;T^{\mathcal{S}^*}}(w;t)$ with respect to $T\sim P^{\mathcal{S}^*}_{T|=w}$.
In other words, it quantifies the degree of variability in the size of the training dataset constructed via the optimal sampling scheme $\mathcal{S}^*$, which aims to maintain an average distortion below $d$ using $\mathcal{A}$.

\textbf{Relationship between $V_{\mathrm{in}, \mathcal{S}}^{ \iota}$ and Existing Frameworks.}
The variance $V_{\mathrm{in}, \mathcal{S}}^{ \iota}$ is related to the term $\mathrm{Var}(T^{\mathcal{S}^{\mathrm{iid}}};H)$, which measures the degree of overfitting, appearing in the IT-bound derived by \cite{hellstrom2020generalization}
\footnote{
Note that $\mathrm{Var}(T^{\mathcal{S}^{\mathrm{iid}}};H)$ is a reformulation of the term $M_2(W;Z^n)$ in \cite{hellstrom2020generalization}.
}:
\begin{theorem}[A special case of Corollary 6 in \cite{hellstrom2020generalization}]
    Let the training dataset $T$ consist of $n$ i.i.d. samples drawn from $P^*_{\bm{X}Y}$.
    Assume that for any $h \in \mathcal{H}$, the loss $l(h(\bm{X}),Y)$ is $\sigma$-subgaussian under $P^*_{\bm{X}Y}$.
    Furthermore, assume $P^{\mathcal{A}}_{H|T}P^*_{\bm{X}^nY^n}$ is absolutely continuous with respect to $P^{\mathcal{A}}_{H}P^*_{\bm{X}^nY^n}$, where $P^{\mathcal{A}}_{H}$ is the marginal distribution of $P^{\mathcal{A}}_{H|T}$ over $T\sim P^*_{\bm{X}^nY^n}$.
    Then, for any $\delta \in (0,1)$, the following holds with probability at least $1-\delta$:
    \begin{align}
        |\mathrm{average~ test~ error} - \mathrm{average~ train~ error}| \le \sqrt{\frac{2\sigma^2}{n}\bigg( \mathbb{I}(T^{\mathcal{S}^{\mathrm{iid}}};H) + \frac{\mathrm{Var}(T^{\mathcal{S}^{\mathrm{iid}}};H)}{(\delta/2)^{1/2}}  + \log \frac{2}{\delta}\bigg)}
    \end{align}
    where $\mathrm{Var}(T^{\mathcal{S}^{\mathrm{iid}}};H):=\mathrm{Var}_{P^{\mathcal{A}}_{H|T}P^*_{\bm{X}^nY^n} }\Big(\iota_{T^{\mathcal{S}^{\mathrm{iid}}};H}(T;H) \Big)$.
\label{thm:hellstrom2020thm6}
\end{theorem}

We now explain this relationship.
As a preliminary step, we present the following theorem.
\begin{theorem}
    For any $P^{\mathcal{A}}_{H|T}$, $P^{\mathcal{S}}_{T|W}$, $w \in \mathcal{W}$, $t \in \mathcal{T}$, and $h \in \mathcal{H}$, the following holds:
    \begin{align}
        \iota_{W;T}(w;t) = \iota_{T;H}(t;h) + \iota_{W;T}(w;t|h) - \iota_{T;H}(t;h|w). \label{eq:iota_WT_expansion}
    \end{align}
\label{thm:iota_WT_expansion}
\end{theorem}

\begin{proof}[Proof of Theorem \ref{thm:iota_WT_expansion}]

First, regarding $\iota_{W;TH}(w;t,h)$, the following two equations hold:
\begin{align}
    \textstyle \iota_{W;TH}(w;t,h) 
    &\textstyle= \log \frac{\mathrm{d} P^{\mathcal{A},\mathcal{S}}_{HTW}}{\mathrm{d}P_W \mathrm{d} P^{\mathcal{A},\mathcal{S}}_{HT}}(w,t,h) 
    = \log \frac{\mathrm{d} P^{\mathcal{S}}_{T|W=w} P^{\mathcal{A}}_{H|T}}{\mathrm{d} P^{\mathcal{S}}_T  \mathrm{d} P^{\mathcal{A}}_{H|T}}(t,h) 
    =  \log \frac{\mathrm{d} P^{\mathcal{S}}_{T|W=w} }{\mathrm{d} P^{\mathcal{S}}_T}(t) \notag \\ 
    &= \iota_{W;T}(w;t),
\end{align}
\begin{align}
    \iota_{W;TH}(w;t,h) 
    &\textstyle= \log \frac{\mathrm{d} P^{\mathcal{A},\mathcal{S}}_{HW} \mathrm{d} P^{\mathcal{A},\mathcal{S}}_{T|W,H} }{\mathrm{d} P_W \mathrm{d} P^{\mathcal{A}}_H \mathrm{d} P^{\mathcal{A},\mathcal{S}}_{T|H}}(w,t,h) 
    = \log \frac{\mathrm{d} P^{\mathcal{A},\mathcal{S}}_{HW}}{\mathrm{d} P_W \mathrm{d} P^{\mathcal{A}}_H }(w,h) + \log \frac{\mathrm{d} P^{\mathcal{A},\mathcal{S}}_{T|W=w,H}}{\mathrm{d} P^{\mathcal{A},\mathcal{S}}_{T|H}}(t,h) \notag \\
    &=\iota_{W;H}(w;h) + \iota_{W;T}(w;t|h).
\end{align}
From these, it follows that:
\begin{align}
    \iota_{W;T}(w;t) = \iota_{W;H}(w;h) + \iota_{W;T}(w;t|h). \label{eq:thm:iota_WT_expansion_tmp1} 
\end{align}
Next, expanding $\iota_{W;H}(w;h)$ yields:
\begin{align}
    \iota_{W;H}(w;h) 
    &\textstyle= \log \frac{\mathrm{d} P^{\mathcal{A}, \mathcal{S}}_{H|W=w} }{\mathrm{d} P^{\mathcal{A},\mathcal{S}}_{H}}(h)
    = \log \frac{\mathrm{d} P^{\mathcal{A}}_{H|T=t} }{\mathrm{d} P^{\mathcal{A},\mathcal{S}}_{H}}(h) + \log \frac{\mathrm{d} P^{\mathcal{A}, \mathcal{S}}_{H|W=w}}{\mathrm{d} P^{\mathcal{A}}_{H|T=t, W=w}}(h) \notag \\
    &= \iota_{T;H}(t;h) - \iota_{T;H}(t;h|w).
\end{align}
Substituting this result into Eq.~\eqref{eq:thm:iota_WT_expansion_tmp1}, we obtain the following:
\begin{align}
    \iota_{W;T}(w;t) = \iota_{T;H}(t;h) + \iota_{W;T}(w;t|h) - \iota_{T;H}(t;h|w). \notag
\end{align}

\end{proof}

Using Theorem~\ref{thm:iota_WT_expansion}, The term $\iota_{W;T^{\mathcal{S}^*}}(w;t)$ appearing in $\mathrm{Var}_{P^{\mathcal{S}^*}_{T|W=w}} (\iota_{W;T^{\mathcal{S}^*}}(w;T) )$ can be decomposed as follows:
\begin{align}
    \iota_{W;T^{\mathcal{S}^*}}(w;t) = \iota_{T^{\mathcal{S}^*};H}(t;h) + \iota_{W;T^{\mathcal{S}^*}}(w;t|h) - \iota_{T^{\mathcal{S}^*};H}(t;h|w). \label{eq:iota_WT_expansion_Tstar}
\end{align}

Using Eq.~\eqref{eq:iota_WT_expansion_Tstar}, we expand $\mathrm{Var}_{P^{\mathcal{S}^*}_{T|W=w}} (\iota_{W;T^{\mathcal{S}^*}}(w;T) )$ as follows:
\begin{align}
    V_{\mathrm{in}, \mathcal{S}}^{ \iota}
    &=\mathrm{Var}_{P^{\mathcal{S}^*}_{T|W=w}} (\iota_{W;T^{\mathcal{S}^*}}(w;T) ) \notag \\
    &=\mathrm{Var}_{P^{\mathcal{A}}_{H|T} P^{\mathcal{S}^*}_{T|W=w}}(\iota_{W;T^{\mathcal{S}^*}}(w;T) ) \notag \\
    &=\mathrm{Var}_{P^{\mathcal{A}}_{H|T} P^{\mathcal{S}^*}_{T|W=w}}(\iota_{T^{\mathcal{S}^*};H}(T;H) + \iota_{W;T^{\mathcal{S}^*}}(w;T|H) - \iota_{T^{\mathcal{S}^*};H}(T;H|w) ) \notag \\
    &=\mathrm{Var}_{P^{\mathcal{A}}_{H|T} P^{\mathcal{S}^*}_{T|W=w}}(\iota_{T^{\mathcal{S}^*};H}(T;H))
        + \mathrm{Var}_{P^{\mathcal{A}}_{H|T} P^{\mathcal{S}^*}_{T|W=w}}( \iota_{W;T^{\mathcal{S}^*}}(w;T|H) - \iota_{T^{\mathcal{S}^*};H}(T;H|w)) \notag \\
    &~~~~~+ \mathrm{Cov}_{P^{\mathcal{A}}_{H|T} P^{\mathcal{S}^*}_{T|W=w}}(\iota_{T^{\mathcal{S}^*};H}(T;H), \iota_{W;T^{\mathcal{S}^*}}(w;T|H) - \iota_{T^{\mathcal{S}^*};H}(T;H|w)).
\label{eq:V_in_iota_expansion}
\end{align}

The first term on the right-hand side of Eq.~\eqref{eq:V_in_iota_expansion}, $\mathrm{Var}_{P^{\mathcal{A}}_{H|T} P^{\mathcal{S}^*}_{T|W=w}}(\iota_{T^{\mathcal{S}^*};H}(T;H))$, differs from $\mathrm{Var}(T^{\mathcal{S}^{\mathrm{iid}}};H)$ in Theorem \ref{thm:hellstrom2020thm6} only in the sampling scheme over which the expectation is taken (i.e., $\mathcal{S}^*$ versus $\mathcal{S}^{\mathrm{iid}}$). 
Therefore, $\mathrm{Var}_{P^{\mathcal{A}}_{H|T}P^{\mathcal{S}^*}_{T|W}=w} (\iota_{W;T^{\mathcal{S}^*}}(w;T) )$ can be regarded as the counterpart to $\mathrm{Var}(T^{\mathcal{S}^{\mathrm{iid}}};H)$ in \cite{hellstrom2020generalization}, evaluated under the sampling scheme $\mathcal{S}^*$.

Furthermore, following the discussion in Appendix~\ref{apdx:subsec:detail_interpret_lc_rate_dist_func}, $\iota_{W;T}(w;t|h)$ represents the amount of information regarding $w$ that $h$ failed to capture from $t$, while $\iota_{T;H}(t;h|w)$ quantifies the extent to which $h$ memorized noise within $t$ that is irrelevant to $w$.
Based on this interpretation, we discuss the second and third terms on the right-hand side of Eq.~\eqref{eq:V_in_iota_expansion}.
First, regarding the second term, $\mathrm{Var}_{P^{\mathcal{A}}_{H|T} P^{\mathcal{S}^*}_{T|W=w}}( \iota_{W;T^{\mathcal{S}^*}}(w;T|H) - \iota_{T^{\mathcal{S}^*};H}(T;H|w))$, we note that since $\mathcal{S}^*$ is an optimal sampling strategy capable of freely selecting samples from the support, the variance arising from $T \sim P^{\mathcal{S}^*}_{T|W=w}$ is expected to be small.
Thus, the variance represented by this term primarily stems from the variability of the learning algorithm $P^{\mathcal{A}}_{H|T}$.
Specifically, this variance signifies the degree to which the information omission $\iota_{W;T}(w;t|h)$ and the degree of fit to the inherent noise $\iota_{T;H}(t;h|w)$ fluctuate due to minor perturbations in $T$ or the inherent stochastic behavior of $\mathcal{A}$ (e.g., randomness in initial values).

Next, we discuss the third term on the right-hand side of Eq.~\eqref{eq:V_in_iota_expansion}, $\mathrm{Cov}_{P^{\mathcal{A}}_{H|T} P^{\mathcal{S}^*}_{T|W=w}}(\iota_{T^{\mathcal{S}^*};H}(T;H), \iota_{W;T^{\mathcal{S}^*}}(w;T|H) - \iota_{T^{\mathcal{S}^*};H}(T;H|w))$. 
The dependency of $h$ on $t$, denoted by $\iota_{T^{\mathcal{S}^*};H}(t;h)$, can be considered higher when the amount of omitted information $\iota_{W;T^{\mathcal{S}^*}}(w;t|h)$ is small and the degree of fit to the inherent noise $\iota_{T^{\mathcal{S}^*};H}(t;h|w)$ is large. 
Consequently, this covariance is generally expected to be negative.

In summary, $V_{\mathrm{in}, \mathcal{S}}^{ \iota}$ can be upper-bounded by the combination of two components: the term $\mathrm{Var}(T;H^{\mathcal{S}^{\mathrm{iid}}}_{\mathcal{A}})$ from the error bound in \cite{hellstrom2020generalization} with the expectation taken over $\mathcal{S}^*$ instead of $\mathcal{S}^{\mathrm{iid}}$, and terms representing the instability of $\mathcal{A}$. 

\subsection{$V_{\mathrm{in}, \mathcal{S}}^{ \mathsf{d}}$ and $V_{\mathrm{in}, \mathcal{A}}^{ \mathsf{d}}$}
\label{apdx:subsec:V_in_S_A_distortion}

\textbf{Meaning of $V_{\mathrm{in}, \mathcal{S}}^{ \mathsf{d}}$ and $V_{\mathrm{in}, \mathcal{A}}^{ \mathsf{d}}$.}
The term $V_{\mathrm{in}, \mathcal{S}}^{ \mathsf{d}}=\mathrm{Var}_{P^{\mathcal{S}^*}_{T|W}} ( \mathbb{E}_{P^{\mathcal{A}}_{H|T}} [ \mathsf{d}(w;H) ]  )$ is calculated as the variance, with respect to $T \sim P^{\mathcal{S}^*}_{T|W=w}$, of the expected distortion $\mathbb{E}_{P^{\mathcal{A}}_{H|T}} [ \mathsf{d}(w;H) ]$, which is averaged over the randomness of $\mathcal{A}$ for a fixed $T$.
This term can be interpreted as measuring the extent to which the distortion of the hypothesis $h$ produced by $\mathcal{A}$ fluctuates due to variations in the training set $T \sim P^{\mathcal{S}^*}_{T|W=w}$.
Conversely, $V_{\mathrm{in}, \mathcal{A}}^{ \mathsf{d}}= \mathbb{E}_{P^{\mathcal{S}^*}_{T|W=w}} [ \mathrm{Var}_{P^{\mathcal{A}}_{H|T}} ( \mathsf{d}(w;H) )]$ is computed as the expectation, over $T \sim P^{\mathcal{S}^*}_{T|W=w}$, of the variance $\mathrm{Var}_{P^{\mathcal{A}}_{H|T}} ( \mathsf{d}(w;H) )$, which captures the fluctuation of $\mathsf{d}(w;H)$ due to the randomness of $\mathcal{A}$ for a fixed $T$.
Accordingly, this measures the variability in distortion attributable solely to the inherent randomness of the algorithm $\mathcal{A}$.

\textbf{Relationship to Existing Theoretical Frameworks.}
The quantity evaluated by $V_{\mathrm{in}, \mathcal{S}}^{ \mathsf{d}}$ is closely related to {\it uniform stability} $\beta_n$ in stability theory \cite{bousquet2002stability,elisseeff2005stability}, which measures the sensitivity of the learned hypothesis's loss to variations in the training dataset $T$.
In this framework, $\beta_n$ provides an upper bound on the fluctuation of the error when a single sample is removed from the dataset $T=\{(\bm{x}_i,y_i)\}^n_{i=1}$.
\begin{definition}[Definition 13 in \cite{elisseeff2005stability}]
   A randomized learning algorithm $\mathcal{A}$ is said to have uniform stability $\beta_n$ with respect to the loss function $l:\mathcal{H}\times\mathcal{X}\times\mathcal{Y}\rightarrow \mathbb{R}_+$ if:
    \begin{align}
    \forall i \in \{1,\dots,n\}, \sup_{t,(\bm{x},y)}\Big|\mathbb{E}_{P^{\mathcal{A}}_{H|T=t}}[l(H,(\bm{x},y))] - \mathbb{E}_{P^{\mathcal{A}}_{H|T=t^{\setminus i}}}[l(H,(\bm{x},y))]  \Big| \le \beta_n,
    \end{align}
    where, $t^{\setminus i}=t \setminus \{(\bm{x}_i, y_i)\}$.
\label{def:uniform_stability_data}
\end{definition}
We note that the definition of stability varies across the literature; some studies define it based on the difference between $t$ and a dataset $t^{i}$ where the $i$-th sample $(\bm{x}_i,y_i)$ is {\it replaced} by a new independent sample $(\bm{x}',y')\sim P^*_{\bm{X}Y}$, rather than removed \cite{kuzbnorskij2018data}.
We will clarify later that this distinction does not fundamentally alter the relationship between $V_{\mathrm{in}, \mathcal{S}}^{ \mathsf{d}}$ and $\beta_n$.

Similarly, the quantity evaluated by $V_{\mathrm{in}, \mathcal{A}}^{ \mathsf{d}}$ relates to the parameter $\rho \in \mathbb{R}_+$ in the same theoretical framework \cite{elisseeff2005stability}, which measures the fluctuation in the loss of the learned hypothesis due to the randomness of $\mathcal{A}$.
In this theory, the random component of $\mathcal{A}$ is represented as a random variable $\bm{r} = (r_1,\dots,r_{B}) \in \mathbb{R}^{B}$ ($B \in \mathbb{N}_+$), where each element is independently and identically distributed according to a distribution $P_r$.
Let $h_{T,(r_1,\dots,r_B)}$ denote the hypothesis learned from the dataset $T$ and the random component $\bm{r}$.
The parameter $\rho$ bounds the maximum fluctuation in the loss value when one component of $\bm{r}$ is replaced by a new independent draw $r'_b$, as follows:
\begin{align}
    \sup_{r_1,\dots,r_B, r'_b} \sup_{(\bm{x},y)} \Big| l(h_{T, (r_1,\dots,r_B)}, (\bm{x},y)) - l(h_{T, (r_1,\dots,r_{b-1}, r'_b, r_{b+1},\dots,r_B)}, (\bm{x},y))  \Big| \le \rho.
\label{eq:uniform_stability_random}
\end{align}
Using $\beta_n$ and $\rho$, the generalization error is upper-bounded as shown below:
\begin{theorem}[Theorem 15 in \cite{elisseeff2005stability}]
    Assume that $\mathcal{A}$ has uniform stability $\beta_n$ with respect to a loss function $l$ satisfying $0\le l(h,(\bm{x},y))\le L$, and that there exists a $\rho$ satisfying Eq. \eqref{eq:uniform_stability_random} for any $b$.
    Here, $L \in \mathbb{R}_+$ represents the maximum value of $l$.
    Then, for any $n \ge 1$ and $\delta \in (0,1)$, the following holds with probability at least $1-\delta$:
    \begin{align}
        \mathrm{average~ test~ error} - \mathrm{average~ train~ error} \le 2\beta_n +\bigg(\frac{L+4n \beta_n}{\sqrt{2n}} + \sqrt{2B}\rho \bigg)\sqrt{\log 2 /\delta}.
    \end{align}
\end{theorem}

Next, we explain the relationship between $V_{\mathrm{in}, \mathcal{S}}^{ \mathsf{d}}$ and $\beta_n$.
Since stability theory assumes that each sample in the training dataset $T$ follows the same distribution $P^*_{\bm{X}Y}$ independently, in this context, we consider the variance under $\mathcal{S}^{\mathrm{iid}}$, denoted by $\mathrm{Var}_{P^{\mathcal{S}^{\mathrm{iid}}}_{T|W=w}} ( \mathbb{E}_{P^{\mathcal{A}}_{H|T}} [ \mathsf{d}(w;H) ])$, with $T \sim P^*_{\bm{X}^nY^n}$.
Based on the main result (Eq. (2.1)) of \cite{steele1986efron} (the Efron-Stein inequality), this variance is upper-bounded as follows:
\begin{align}
    \mathrm{Var}_{P^{\mathcal{S}^{\mathrm{iid}}}_{T|W=w}} ( \mathbb{E}_{P^{\mathcal{A}}_{H|T}} [ \mathsf{d}(w;H) ]) 
    \le \frac{1}{2}\sum^n_{i=1}\mathbb{E}_{P^{\mathcal{S}^{\mathrm{iid}}}_{T|W=w}P^*_{\bm{X}Y}}\Big[\Big(\mathbb{E}_{P^{\mathcal{A}}_{H|T}} [ \mathsf{d}(w;H) ] - \mathbb{E}_{P^{\mathcal{A}}_{H|T^{i}}} [ \mathsf{d}(w;H)] \Big)^2 \Big],
\label{eq:var_3_iid_dist_steele}
\end{align}
where $T^{i}$ is the dataset obtained by replacing the $i$-th sample $(\bm{X}_i,Y_i)$ in $T$ with a new sample $(\bm{X}',Y') \sim P^*_{\bm{X}Y}$.
We redefine uniform stability by replacing $l(H,(\bm{x},y))$ in Definition \ref{def:uniform_stability_data} with the distortion $\mathsf{d}(w;H)$:
\begin{align}
    \forall i \in \{1,\dots,n\}, \sup_{t,w}\Big|\mathbb{E}_{P^{\mathcal{A}}_{H|T=t}}[\mathsf{d}(w;H)] - \mathbb{E}_{P^{\mathcal{A}}_{H|T=t^{\setminus i}}}[\mathsf{d}(w;H)]  \Big| \le \beta_n.
\end{align}
By the triangle inequality, the following holds for any $t \in \mathcal{X}^n\times\mathcal{Y}^n$ and $(\bm{x}',y') \in \mathcal{X}\times\mathcal{Y}$:
\begin{align}
    &\Big|\mathbb{E}_{P^{\mathcal{A}}_{H|T}} [ \mathsf{d}(w;H) ] - \mathbb{E}_{P^{\mathcal{A}}_{H|T^{i}}} [ \mathsf{d}(w;H)]\Big| \notag \\
    &\le \Big|\mathbb{E}_{P^{\mathcal{A}}_{H|T}} [ \mathsf{d}(w;H) ] - \mathbb{E}_{P^{\mathcal{A}}_{H|T^{\setminus i}}} [ \mathsf{d}(w;H)]\Big| 
    + \Big|\mathbb{E}_{P^{\mathcal{A}}_{H|T^{i}}} [ \mathsf{d}(w;H) ] - \mathbb{E}_{P^{\mathcal{A}}_{H|T^{\setminus i}}} [ \mathsf{d}(w;H)] \Big| \notag \\
    &\le 2\beta_n.
\label{eq:stability_triangle_transformation}
\end{align}
Applying this inequality to Eq. \eqref{eq:var_3_iid_dist_steele}, we obtain an upper bound on $\mathrm{Var}_{P^{\mathcal{S}^{\mathrm{iid}}}_{T|W=w}} ( \mathbb{E}_{P^{\mathcal{A}}_{H|T}} [ \mathsf{d}(w;H) ])$, thereby demonstrating its relationship with $\beta_n$:
\begin{align}
    \mathrm{Var}_{P^{\mathcal{S}^{\mathrm{iid}}}_{T|W=w}} ( \mathbb{E}_{P^{\mathcal{A}}_{H|T}} [ \mathsf{d}(w;H) ])
    \le 2n(\beta_n)^2.
\end{align}
Note that if $\beta_n$ is defined based on the difference between $T$ and $T^{i}$ (replace-one stability) rather than $T$ and $T^{\setminus i}$ (remove-one stability), the left-hand side of \eqref{eq:stability_triangle_transformation} is directly bounded by $\beta_n$.
In that case, we have $\mathrm{Var}_{P^{\mathcal{S}^{\mathrm{iid}}}_{T|W=w}} ( \mathbb{E}_{P^{\mathcal{A}}_{H|T}} [ \mathsf{d}(w;H) ]) \le \frac{n}{2}\beta_n^2$.
The distinction between $V_{\mathrm{in}, \mathcal{S}}^{ \mathsf{d}}$ in our framework and $\mathrm{Var}_{P^{\mathcal{S}^{\mathrm{iid}}}_{T|W=w}} ( \mathbb{E}_{P^{\mathcal{A}}_{H|T}} [ \mathsf{d}(w;H) ])$, evaluated by $\beta_n$, lies solely in the process of obtaining $T$: our framework employs $\mathcal{S}^*$, whereas stability theory assumes $\mathcal{S}^{\mathrm{iid}}$.
Consequently, $V_{\mathrm{in}, \mathcal{S}}^{ \mathsf{d}}$ can be viewed as the counterpart under $\mathcal{S}^*$ to the quantity $\mathrm{Var}_{P^{\mathcal{S}^{\mathrm{iid}}}_{T|W=w}} ( \mathbb{E}_{P^{\mathcal{A}}_{H|T}} [ \mathsf{d}(w;H) ])$ essentially measured by $\beta_n$.

Finally, since the optimal sampling strategy $\mathcal{S}^*$ can freely select samples from the sub-distribution, it is expected to generate a training set $T$ with smaller variance compared to i.i.d. sampling $\mathcal{S}^{\mathrm{iid}}$.
Based on this intuition, assuming $\mathrm{Var}_{P^{\mathcal{S}^{*}}_{T|W=w}}( \mathbb{E}_{P^{\mathcal{A}}_{H|T}} [ \mathsf{d}(w;H) ]) \le \mathrm{Var}_{P^{\mathcal{S}^{\mathrm{iid}}}_{T|W=w}}( \mathbb{E}_{P^{\mathcal{A}}_{H|T}} [ \mathsf{d}(w;H) ])$, we obtain:
\begin{align}
    V_{\mathrm{in}, \mathcal{S}}^{ \mathsf{d}}=\mathrm{Var}_{P^{\mathcal{S}^{*}}_{T|W=w}}( \mathbb{E}_{P^{\mathcal{A}}_{H|T}} [ \mathsf{d}(w;H) ]) \le \mathrm{Var}_{P^{\mathcal{S}^{\mathrm{iid}}}_{T|W=w}}( \mathbb{E}_{P^{\mathcal{A}}_{H|T}} [ \mathsf{d}(w;H) ]) \le 2n(\beta_n)^2.
\end{align}
The assumption that $\mathcal{S}^*$ yields a lower variance for $T$ is intuitively sound.
Thus, under this natural assumption, we have demonstrated a direct relationship between $V_{\mathrm{in}, \mathcal{S}}^{ \mathsf{d}}$ and $\beta_n$.

Uniform stability $\beta_n$ is widely used in theoretical analyses of algorithms such as SGD \cite{hardt2016train,kuzbnorskij2018data,bassily2020stability,lei2021stability,lei2023stability} and SGLD \cite{mou2018generalization}.
These studies analyze the generalization performance of SGD by deriving upper bounds on $\beta_n$.
Consequently, the results from these analyses could potentially serve as upper bounds for $V_{\mathrm{in}, \mathcal{S}}^{ \mathsf{d}}$.

Next, we elucidate the relationship between $V_{\mathrm{in}, \mathcal{A}}^{ \mathsf{d}}$ and $\rho$.
Analogous to the case of $V_{\mathrm{in}, \mathcal{S}}^{ \mathsf{d}}$, we consider the variance under i.i.d. sampling, $\mathbb{E}_{P^{\mathcal{S}^{\mathrm{iid}}}_{T|W=w}} [ \mathrm{Var}_{P^{\mathcal{A}}_{H|T}} ( \mathsf{d}(w;H) )]$, with $T \sim P^*_{\bm{X}^nY^n}$.
According to the main result (Eq. (2.1)) of \cite{steele1986efron}, this variance is upper-bounded as:
\begin{align}
\begin{split}
    &\mathbb{E}_{P^{\mathcal{S}^{\mathrm{iid}}}_{T|W=w}} [ \mathrm{Var}_{P^{\mathcal{A}}_{H|T}} ( \mathsf{d}(w;H) )] \\
    &\le \frac{1}{2}\sum^n_{i=1}\mathbb{E}_{P^{\mathcal{S}^{\mathrm{iid}}}_{T|W=w}}\Big[
    \mathbb{E}_{(P_r)^{B+1}}
    \Big[\Big( l(h_{T, (r_1,\dots,r_B)}, (\bm{x},y)) - l(h_{T, (r_1,\dots,r_{b-1}, r'_b, r_{b+1},\dots,r_B)}, (\bm{x},y)) \Big)^2 \Big] \Big].
\end{split}
\label{eq:var_4_iid_dist_steele}
\end{align}
Applying \eqref{eq:uniform_stability_random} to this inequality yields:
\begin{align}
    \mathbb{E}_{P^{\mathcal{S}^{\mathrm{iid}}}_{T|W=w}} [ \mathrm{Var}_{P^{\mathcal{A}}_{H|T}} ( \mathsf{d}(w;H) )]
    \le \frac{n}{2}\rho^2.
\end{align}

The distinction between $\mathbb{E}_{P^{\mathcal{S}^{\mathrm{iid}}}_{T|W=w}} [ \mathrm{Var}_{P^{\mathcal{A}}_{H|T}} ( \mathsf{d}(w;H) )]$, measured by $\rho$ and $V_{\mathrm{in}, \mathcal{A}}^{ \mathsf{d}}$ lies solely in the process of obtaining $T$: stability theory assumes $\mathcal{S}^{\mathrm{iid}}$, whereas our framework employs $\mathcal{S}^*$.
Therefore, $V_{\mathrm{in}, \mathcal{A}}^{ \mathsf{d}}$ can be viewed as the counterpart under $\mathcal{S}^*$ to the quantity essentially measured by $\rho$.

Furthermore, an attempt to relate these quantities via an inequality yields the following.
Similar to the reasoning for $V_{\mathrm{in}, \mathcal{S}}^{ \mathsf{d}}$, based on the intuition that $\mathcal{S}^*$ produces a lower variance in $T$ than $\mathcal{S}^{\mathrm{iid}}$, we assume $ \mathbb{E}_{P^{\mathcal{S}^{*}}_{T|W=w}} [ \mathrm{Var}_{P^{\mathcal{A}}_{H|T}} ( \mathsf{d}(w;H) )] \le  \mathbb{E}_{P^{\mathcal{S}^{\mathrm{iid}}}_{T|W=w}} [ \mathrm{Var}_{P^{\mathcal{A}}_{H|T}} ( \mathsf{d}(w;H) )]$. This leads to:
\begin{align}
    V_{\mathrm{in}, \mathcal{A}}^{ \mathsf{d}} = \mathbb{E}_{P^{\mathcal{S}^{*}}_{T|W=w}} [ \mathrm{Var}_{P^{\mathcal{A}}_{H|T}} ( \mathsf{d}(w;H) )] \le \frac{n}{2}\rho^2.
\end{align}
Therefore, under the assumption, we have established a connection between $V_{\mathrm{in}, \mathcal{A}}^{ \mathsf{d}}$ and $\rho$ via an inequality.

\section{Proof of Section~\ref{sec:lc_analysis}}
\label{apdx:sec:lc_proof}

\subsection{Proof of Theorem~\ref{thm:rate_distortion_theorem}}
\label{apdx:proof:thm:rate_distortion_theorem}

First, from Lemma 10.4.1 in \cite{cover1999elements}, $\mathbb{R}_W(d,\mathcal{A})$ is non-increasing and convex in $d$.

To establish \eqref{eq:lc_R_W}, we prove the following two points:

(i) For any finite $d>0$, $R(d,\mathcal{A}) \ge \mathbb{R}_W(d, \mathcal{A})$, 

(ii) $\limsup_{k \rightarrow \infty}\widebar{R}(k, d,\mathcal{A}) \le \mathbb{R}_W(d, \mathcal{A})$.

Once (ii) is established, Assumption (\Two) implies that $R(d, \mathcal{A}) \le \mathbb{R}_W(d, \mathcal{A})$. 
Combining this with the result of (i) yields \eqref{eq:lc_R_W}.

\textbf{Proof of (i):}

This proof follows the converse of Theorem 10.2.1 in \cite{cover1999elements}.
Consider an arbitrary pair of $\langle k,n,d,\mathcal{A} \rangle$-sampling $(\mathcal{S}, \mathcal{A})$.
In this case, the total number of possible training datasets output by $\mathcal{S}$ is $2^{bn}$, and the rate $R$ is given by $R = \frac{bn}{k}$.
For this rate, the following holds:
\begin{align}
    kR =bn
    &\ge \mathbb{H}(T) \notag \\
    &\ge \mathbb{H}(T) - \mathbb{H}(T|W^k) \notag \\
    &= \mathbb{I}(W^k; T) \notag \\
    &= \mathbb{H}(W^k) - \mathbb{H}(W^k | T) \notag \\
    &= \sum^k_{i=1}\mathbb{H}(W_i) - \sum^k_{i=1} \mathbb{H}(W_i | T, W_{i-1},\dots, W_1) ~~ \because W^k ~\text{is i.i.d.} \notag \\
    &\ge \sum^k_{i=1}\mathbb{H}(W_i) - \sum^k_{i=1} \mathbb{H}(W_i | T) \notag \\
    &= \sum^k_{i=1}\mathbb{I}(W_i; T) \notag \\
    &\ge \sum^k_{i=1} \mathbb{R}_W(d_i, \mathcal{A}) ~~~ \text{where~} d_i=\mathbb{E}_{P^{\mathcal{A}}_{H|T}P^{\mathcal{S}}_{T|W}P_W}[d(W_i;H)] \notag \\
    &= k \bigg(\frac{1}{k} \sum^k_{i=1} \mathbb{R}_W(d_i, \mathcal{A}) \bigg) \notag \\
    &\ge k \mathbb{R}_W \bigg( \frac{1}{k}\sum^k_{i=1} d_i, \mathcal{A} \bigg) ~~~ \because~ \text{Jensen's inequality} \notag \\
    &\ge k \mathbb{R}_W(d, \mathcal{A}).
\end{align}
In the first inequality, we utilized the fact that the maximum value of the entropy $\mathbb{H}(T)$ is $bn$, as the total number of possible sequences $T$ is $bn$.
In addition, the fourth inequality is obtained by applying the definition of $\mathbb{R}_W$, which characterizes $\mathbb{R}_W$ as the infimum of $\mathbb{I}(W;T)$ for any sampling strategy.
Therefore, the rate $R$ for any $\langle k,n,d, \mathcal{A} \rangle$-sampling is lower-bounded by $\mathbb{R}_W(d, \mathcal{A})$. 
Consequently, by taking the limit as $k \rightarrow \infty$, we establish (i).

We provide an intuitive explanation for the validity of the fourth inequality in the above derivation.
First, $\mathbb{R}_W(d_i, \mathcal{A})$ is defined as the minimum value of $\mathbb{I}(W; T)$ achievable across all possible sampling strategies.
Given a sampling strategy $\mathcal{S}$ that constructs $T$ from $W^k$, one can always derive a localized sampling strategy $\mathcal{S}_i$ by isolating the process of acquiring data specifically from $W_i$.
Since $\mathcal{S}_i$ is necessarily contained within the set of all possible sampling strategies, $\mathbb{R}_W(d_i, \mathcal{A})$ is inherently less than or equal to the mutual information $\mathbb{I}(W; T)$ achieved by $\mathcal{S}_i$.
It should be noted that even if $T$ does not directly contain samples from $W_i$, the information $\mathbb{I}(W_i; T)$ does not vanish; this is because samples obtained from other variables $W_j$ can serve to compensate for the prediction of $W_i$.
An alternative perspective is as follows.
The target distortion $d_i$ is defined as the average distortion realized when the sampling strategy $\mathcal{S}$ (which assumes $T=\mathcal{S}(W^k)$) is employed within the setting $\{H_i = \mathcal{A}(\mathcal{S}(W_i))\}$.
Under such a configuration, $\mathcal{S}$ is not necessarily optimal. Consequently, $d_i$ can be viewed as being larger than the distortion $d$ that would satisfy $\mathbb{I}(W_i; T) = \mathbb{R}_W(d, \mathcal{A})$.

We present two lemmas from \cite{cover1999elements} used in the proof of (ii).

\begin{lemma}[Lemma 10.5.2 in \cite{cover1999elements}]

Let $p(x,\hat{x})$ be a joint probability distribution on the space $\mathcal{X} \times \widehat{\mathcal{X}}$.
For any $\epsilon >0$, suppose there exist pairs of sequences $(x^k, \hat{x}^k)$ satisfying the following conditions:
\begin{align}
    \textstyle | -\frac{1}{k}\log p(x^k) - \mathbb{H}(X) | &< \epsilon, \notag \\
    \textstyle | -\frac{1}{k}\log p(\hat{x}^k) - \mathbb{H}(\widehat{X}) | &< \epsilon, \notag \\
    \textstyle | -\frac{1}{k}\log p(x^k, \hat{x}^k) - \mathbb{H}(X,\widehat{X}) | &< \epsilon. \notag \\
\end{align}
For any such pair $(x^k, \hat{x}^k)$, the following holds:
\begin{align}
    p(\hat{x}^k) \ge p(\hat{x}^k | x^k) 2^{-k (\mathbb{I}(X;\widehat{X}) + 3\epsilon)}.
\end{align}
    
\end{lemma}

\begin{lemma}[Lemma 10.5.3 in \cite{cover1999elements}]
    For $0 \le x$, $y \le 1$, $k >0$,
    \begin{align}
        (1 -xy)^k \le 1 - x + e^{-yk}.
    \end{align}
\end{lemma}

\textbf{Proof of (ii):}

This proof is based on the achievability of Theorem 10.2.1 in \cite{cover1999elements}.
Recall that $\widebar{R}(d,\mathcal{A})$ is the minimum rate as $k \rightarrow \infty$ when hypotheses are constructed via the scheme $\{H_i = \mathcal{A}(\mathcal{S}(W_i))\}_{i=1}^k$.
We now derive an upper bound for this rate.
First, let $\widebar{\mathcal{S}}^*$ be a sampling strategy that achieves the infimum in the definition of $\mathbb{R}_W(d, \mathcal{A})$. 
We define $P^{\widebar{\mathcal{S}}^*,\mathcal{A}}_H(h)$ as the marginal distribution of the hypothesis when employing $\widebar{\mathcal{S}}^*$ and $\mathcal{A}$.

For any $\delta > 0$, we set the target rate to $\widebar{R} = \mathbb{R}_W(d,\mathcal{A}) + \delta$ and construct $M = 2^{k\widebar{R}}$ training datasets using $\widebar{\mathcal{S}}^*$.
Specifically, we independently generate $M$ sequences of datasets $T(m) = (T_1(m), \dots, T_k(m))$ for $m = 1, \dots, M$, according to the marginal distribution $P^{\widebar{\mathcal{S}}^*}_T$ of a training dataset under $\widebar{\mathcal{S}}^*$.
We denote the resulting collection as $\mathcal{C}_T = \{T(1), \dots, T(M)\}$.

Following the approach in \cite{cover1999elements}, we define the distortion typical set $A^{(k)}_{\mathsf{d},\epsilon}$ for any $\epsilon > 0$. 
The set $A^{(k)}_{\mathsf{d},\epsilon}$ is the collection of sequence pairs $(w^k, t^k)$ satisfying the following:
\begin{align}
    \bigg|-\frac{1}{k}\log \prod^k_{i=1} P_W(w_i) - \mathbb{H}(W) \bigg| &< \epsilon, \notag \\
    \bigg|-\frac{1}{k}\log \prod^k_{i=1} P^{\widebar{\mathcal{S}}^*}_T(t_i) - \mathbb{H}(T) \bigg| &< \epsilon, \notag \\
    \bigg|-\frac{1}{k}\log \prod^k_{i=1} P^{\widebar{\mathcal{S}}^*}_{T|W}P_W(w_i,t_i) - \mathbb{H}(W, T) \bigg| &< \epsilon, \notag \\
    \bigg|\frac{1}{k}\sum^k_{i=1}\mathbb{E}_{P^{\mathrm{A}}_{H|T=t_i}}[\mathsf{d}(w_i, H)] - \mathbb{E}_{P^{\mathrm{A}}_{H|T} P^{\widebar{\mathcal{S}}^*}_{T|W}}[\mathsf{d}(W, H)]  \bigg| &\le \epsilon.
\end{align}
Next, for a given $W^k$, we consider a sampling strategy $\mathcal{S}$ that searches for an index $m^*$ such that $(W^k, T^k(m^*)) \in A^{(k)}_{\mathsf{d},\epsilon}$ among $\mathcal{C}_T$.
If such an index is found, the strategy outputs $T^k(m^*)$; otherwise, it outputs $T^k(1)$. 
The probability of error $P_e$, defined as the probability that no such $m^*$ exists (i.e., none of the $M$ datasets fall within $A^{(k)}_{\mathsf{d},\epsilon}$), is expressed as follows:
\begin{align}
    P_e = \sum_{w^k \in \mathcal{W}^k} P(w^k) \bigg\{ 1 - \sum_{t^k}P^{\widebar{\mathcal{S}}^*}_{T}(t^k) \mathbbm{1}_{[(w^k, t^k) \in A^{(k)}_{\mathsf{d},\epsilon}]}  \bigg\}^M. 
\end{align}
By Lemma 10.5.2 of \cite{cover1999elements}, the inequality $P^{\widebar{\mathcal{S}}^*}_{T}(t^k) \ge P^{\widebar{\mathcal{S}}^*}_{T|W}(t^k|w^k)2^{-k (\mathbb{I}(W;T^{\widebar{\mathcal{S}}^*}) + 3\epsilon)}$ holds.
Here, the notation $T^{\widebar{\mathcal{S}}^*}$ is used to explicitly indicate that $T$ was obtained via $\widebar{\mathcal{S}}^*$.
Furthermore, from Lemma 10.5.3 of \cite{cover1999elements}, for any $0 \le a_1, a_2 \le 1$ and $a_3 > 0$, we have $(1-a_1a_2)^{a_3} \le 1 - a_1 + e^{-a_2 a_3}$. 
Applying these results, we obtain:
\begin{align}
    P_e 
    &\le \sum_{w^k \in \mathcal{W}^k} P(w^k) \bigg\{ 1 - \sum_{t^k}P^{\widebar{\mathcal{S}}^*}_{T|W}(t^k|w^k) 2^{-k \mathbb{I}(W;T^{\widebar{\mathcal{S}}^*}) + 3\epsilon} \mathbbm{1}_{[(w^k, t^k) \in A^{(k)}_{\mathsf{d},\epsilon}]}  \bigg\}^M \notag \\
    &\le \sum_{w^k \in \mathcal{W}^k} P(w^k) \bigg\{ 1 - \sum_{t^k}P^{\widebar{\mathcal{S}}^*}_{T|W}(t^k|w^k)\mathbbm{1}_{[(w^k, t^k) \in A^{(k)}_{\mathsf{d},\epsilon}]} + e^{-2^{-k \mathbb{I}(W;T^{\widebar{\mathcal{S}}^*}) + 3\epsilon}M} \bigg\} \notag \\
    &= \bigg\{ 1 - \sum_{w^k \in \mathcal{W}^k}\sum_{t^k}P^{\widebar{\mathcal{S}}^*}_{T|W}(t^k|w^k)P(w^k)\mathbbm{1}_{[(w^k, t^k) \in A^{(k)}_{\mathsf{d},\epsilon}]} \bigg\} + e^{-2^{-k \mathbb{I}(W;T^{\widebar{\mathcal{S}}^*}) + 3\epsilon}2^{k\widebar{R}}} \notag \\
    &= \bigg\{ 1 - \sum_{w^k \in \mathcal{W}^k}\sum_{t^k}P^{\widebar{\mathcal{S}}^*}_{T|W}(t^k|w^k)P(w^k)\mathbbm{1}_{[(w^k, t^k) \in A^{(k)}_{\mathsf{d},\epsilon}]} \bigg\} + e^{-2^{k( \mathbb{R}_W(d,\mathcal{A}) - \mathbb{I}(W;T^{\widebar{\mathcal{S}}^*})) +k (\delta - 3\epsilon)}} \notag \\
    &= \bigg\{ 1 - \sum_{w^k \in \mathcal{W}^k}\sum_{t^k}P^{\widebar{\mathcal{S}}^*}_{T|W}(t^k|w^k)P(w^k)\mathbbm{1}_{[(w^k, t^k) \in A^{(k)}_{\mathsf{d},\epsilon}]} \bigg\} + e^{-2^{k (\delta - 3\epsilon)}}. 
\end{align}
According to Lemma 10.5.1 of \cite{cover1999elements}, the first term on the right-hand side, $1 - \sum_{w^k \in \mathcal{W}^k}\sum_{t^k}P^{\widebar{\mathcal{S}}^*}_{T|W}(t^k|w^k)P(w^k)\mathbbm{1}_{[(w^k, t^k) \in A^{(k)}_{\mathsf{d},\epsilon}]}$, converges to $0$ as $k \rightarrow \infty$.
Moreover, by choosing $\epsilon < \delta/3$, the second term $e^{-2^{k (\delta - 3\epsilon)}}$ also vanishes as $k \rightarrow \infty$. 
Consequently, $\lim_{k \rightarrow \infty} P_e = 0$, implying that as $k \rightarrow \infty$, the strategy $\mathcal{S}$ can output $T^k(m^*)$ with probability $1$.

Now, consider the case where $\mathcal{S}$ outputs $T^k(m^*)$ for a finite $k$.
Under this condition, the definition of $A^{(k)}_{\mathsf{d},\epsilon}$ implies the following:
\begin{align}
    \frac{1}{k}\sum^k_{i=1}\mathbb{E}_{P^{\mathcal{A}}_{H|T=T_i(m^*)}}[\mathsf{d}(W_i, H)] \le \mathbb{E}_{P^{\mathcal{A}}_{H|T}P^{\widebar{\mathcal{S}}^*}_{T|W}P_W}[\mathsf{d}(W,H)] + \epsilon  \le d + \epsilon. \notag
\end{align}
Given the scheme $\{H_i = \mathcal{A}(\mathcal{S}(W_i))\}_{i=1}^k$, the algorithm $\mathcal{A}$ is executed independently for each $T_i(m^*)$, producing hypotheses $H_i \sim P^{\mathcal{A}}_{H|T=T_i(m^*)}$.
By the law of large numbers, the average distortion $\frac{1}{k}\sum_{i=1}^k \mathsf{d}(W_i; H_i)$ converges to $\mathbb{E}_{P^{\mathcal{A}}_{H|T} P^{\widebar{\mathcal{S}}^*}_{T|W}}[\mathsf{d}(W;H)]$ as $k \rightarrow \infty$. 
Consequently, the inequality $\mathbb{E}_{P^{\mathcal{A}}_{H|T} P^{\widebar{\mathcal{S}}^*}_{T|W}}[\mathsf{d}(W;H)] \le d + \epsilon$ holds in the limit $k \rightarrow \infty$.
Since $\delta$ and $\epsilon$ can be made arbitrarily small, taking the limits $\delta \rightarrow 0$ and $\epsilon \rightarrow 0$ yields $\widebar{R} \le \mathbb{R}_{W}(d,\mathcal{A})$.
This implies the existence of a sampling strategy that achieves an asymptotic rate no greater than $\mathbb{R}_{W}(d,\mathcal{A})$ while maintaining an average distortion within $d$; thus, we obtain $\widebar{R}(d,\mathcal{A}) \le \mathbb{R}_{W}(d, \mathcal{A})$.

\subsection{Proof of Theorem~\ref{thm:eq:ML_lossy_sc_eps_conv_bound}}
\label{apdx:proof:thm:eq:ML_lossy_sc_eps_conv_bound}

\begin{proof}[Proof of Theorem~\ref{thm:eq:ML_lossy_sc_eps_conv_bound}]

This proof follows the methodology used in the proof of Theorem 1 in \cite{kostina2016nonasymptotic}.
Let $P^{\mathcal{A}}_{H|T}$ be a learning algorithm, and let $P^{\mathcal{S}}_{T|W}$ denote a sampling strategy that constitutes an $(n,d,\epsilon,\mathcal{A})$ sampling.
Here, we define the random variable $T$ to take values in $\mathcal{T}_n$, the set of all possible training datasets of size $n$.
We define the set $B_d(w)$ as follows:
\begin{align}
  B_d(w) := \{ h \in \mathcal{H} : \mathsf{d}(w,h) \le d \}.
\end{align}

For any $\gamma \ge 0$, the following holds:

\begin{align}
  &\mathbb{P}[\jmath_W(W,T,H,d,\mathcal{A}) \ge bn + \gamma] \notag \\
  &= \mathbb{P}[\jmath_W(W,T,H,d,\mathcal{A}) \ge bn + \gamma, \mathsf{d}(W;H) > d] \notag \\
  &~~~~~~~~~~+ \mathbb{P}[\jmath_W(W,T,H,d,\mathcal{A}) \ge bn + \gamma, \mathsf{d}(W;H) \le d] \notag \\
  &\le \epsilon + \sum_{w\in\mathcal{W}} P_W(w) \sum_{t \in \mathcal{T}_n} P^{\mathcal{S}}_{T|W}(t|w) \int_{B_d(w)}  \mathbbm{1}_{[e^{bn} \le \exp(\jmath_W(w,t,h,d,\mathcal{A}) - \gamma)]}\mathrm{d}P^{\mathcal{A}}_{H|T}(h|t) \\
  &\le \epsilon + e^{-\gamma} \sum_{w\in\mathcal{W}} P_W(w) \sum_{t \in \mathcal{T}_n} \frac{1}{2^{bn}} \int_{B_d(w)}  e^{\jmath_W(w,t,h,d,\mathcal{A})} \mathrm{d}P^{\mathcal{A}}_{H|T}(h|t)  \label{eq:Mls_ecb_tmp2} \\
  &\le \epsilon + e^{-\gamma} \sum_{w\in\mathcal{W}} P_W(w) \sum_{t \in \mathcal{T}_n} \frac{1}{2^{bn}} \int_{\mathcal{H}} e^{\lambda^*_{\mathcal{A}}(d)d - \lambda^*_{\mathcal{A}}(d)\mathsf{d}(w;h)}  e^{\jmath_W(w,t,h,d,\mathcal{A})} \mathrm{d}P^{\mathcal{A}}_{H|T}(h|t) \label{eq:Mls_ecb_tmp3} \\
  &= \epsilon + e^{-\gamma} \sum_{w\in\mathcal{W}} P_W(w) \sum_{t \in \mathcal{T}_n} \frac{1}{2^{bn}} e^{\iota_{W;T^{\mathcal{S}^*}}(w;t) } \int_{\mathcal{H}}   \mathrm{d}P^{\mathcal{A}}_{H|T}(h|t) \label{eq:Mls_ecb_tmp4} \\
  &= \epsilon + e^{-\gamma}   \sum_{t \in \mathcal{T}_n} \frac{1}{2^{bn}} \int_{\mathcal{H}} \mathrm{d} P^{\mathcal{A}}_{H|T}(h|t) \sum_{w\in\mathcal{W}} P_{W|T^{\mathcal{S}^*}}(w|t) \label{eq:Mls_ecb_tmp5} \\
  &= \epsilon + e^{-\gamma}. \label{eq:Mls_ecb_tmp6}
\end{align}

Here, the first inequality follows from the definition of an $(n,d,\epsilon, \mathcal{A})$-sampling, which satisfies $\mathbb{P}[\mathsf{d}(W;\mathcal{A}(\mathcal{S}(W))) > d] \le \epsilon$.
Equation \eqref{eq:Mls_ecb_tmp2} utilizes the fact that:
\begin{align}
  P^{\mathcal{S}}_{T|W}(t|w) \mathbbm{1}_{[e^{bn} \le \exp(\jmath_W(w,t,h,d,\mathcal{A}) - \gamma)]} 
  \le \frac{e^{-\gamma}}{e^{bn}}e^{\jmath_W(w,t,h,d,\mathcal{A})} 
  \le \frac{e^{-\gamma}}{2^{bn}}e^{\jmath_W(w,t,h,d,\mathcal{A})}.
\end{align}
In Eq. \eqref{eq:Mls_ecb_tmp3}, we applied:
\begin{align}
  \int_{B_d(w)} \mathrm{d}P^{\mathcal{A}}_{H|T}(h|t) 
  = \int_{\mathcal{H}} \mathbbm{1}_{[\mathsf{d}(w;h) \le d]}  \mathrm{d}P^{\mathcal{A}}_{H|T}(h|t)
  \le \int_{\mathcal{H}}  e^{\lambda^*_{\mathcal{A}}(d)d - \lambda^*_{\mathcal{A}}(d)\mathsf{d}(w;h)} \mathrm{d}P^{\mathcal{A}}_{H|T}(h|t).
\end{align}
Equation \eqref{eq:Mls_ecb_tmp4} follows from the definition of $\jmath_W$, while Eq. \eqref{eq:Mls_ecb_tmp5} uses the definition of information density $\iota_{W;T^{\mathcal{S}^*}}(w;t) = \log \frac{\mathrm{d} P_{W|T^{\mathcal{S}^*}}(w|t) }{\mathrm{d} P_W(w)}$.
Finally, Eq. \eqref{eq:Mls_ecb_tmp6} holds because each sample is represented by $b$ bits, implying $|\mathcal{T}_n| = 2^{bn}$.

\end{proof}

\subsection{Proof of Theorem~\ref{thm:ML_lossy_sc_rate_conv_bound}}
\label{apdx:proof:thm:ML_lossy_sc_rate_conv_bound}

The proof of Theorem~\ref{thm:ML_lossy_sc_rate_conv_bound} relies on the following theorem.
\begin{theorem}[Berry-Esseen CLT \cite{erokhin1958CLT,kostina2012fixed}]

Let $k$ be a positive integer, and let $Z_1, \dots, Z_k$ be a sequence of independent random variables.
Then, for any real value $\alpha$, the following holds:
\begin{align}
  \bigg| \mathbb{P}\bigg[ \sum^k_{i=1} Z_i >  k \bigg( \mu_k + \alpha \sqrt{\frac{V_k}{k}} \bigg)  \bigg] - Q(\alpha)  \bigg| \le \frac{B_k}{\sqrt{k}},
\label{eq:berry_essen_clt}
\end{align}
where,
\begin{align}
& \mu_k = \frac{1}{k}\sum^k_{i=1}\mathbb{E}[Z_i], \\
& V_k = \frac{1}{k}\sum^k_{i=1}\mathrm{Var}(Z_i), \\
& A_k = \frac{1}{k}\sum^k_{i=1}\mathbb{E}[|Z_i - \mu_i |^3], \\
& B_k = 6 \frac{A_k}{V^{3/2}_k}.
\end{align}
\label{thm:berry_essen_clt}
\end{theorem}

As a preliminary step, we establish the following lemma.

\begin{lemma}
    $Q^{-1}(x)$ is a strictly convex function on the interval $x \in (0, 0.5)$.
\label{lem:Q_inv_convex}
\end{lemma}

\begin{proof}[Proof of Lemma~\ref{lem:Q_inv_convex}]

Let $\phi(y) = \frac{1}{\sqrt{2\pi}} e^{-y^2/2}$ denote the probability density function (PDF) of the standard normal distribution. 
Since the Q-function is defined as $Q(y) = \int^{\infty}_{y}\phi(y') \mathrm{d}y'$, its derivative is given by $Q'(y)=-\phi(y)$.
Now, let $y=Q^{-1}(x)$.
Given that $Q(0)=0.5$ and $\lim_{y\rightarrow \infty}Q(y)=0$, the range of $y$ corresponding to $x \in (0, 0.5)$ is $(0, \infty)$.
We now derive the first and second derivatives of $Q^{-1}(x)$. The first derivative is obtained as:
\begin{align}
    \frac{\mathrm{d}}{\mathrm{d}x}Q^{-1}(x) = \frac{\mathrm{d}y}{\mathrm{d}x} = \frac{1}{Q'(y)}=-\frac{1}{\phi(y)}.
\end{align}
Differentiating the above with respect to $x$ yields the second derivative:
\begin{align}
    \frac{\mathrm{d}^2}{\mathrm{d}x^2}Q^{-1}(x) = -\frac{\mathrm{d}}{\mathrm{d}x}\frac{1}{\phi(y)}
    = -\frac{\mathrm{d}}{\mathrm{d}y}\frac{1}{\phi(y)}\cdot \frac{\mathrm{d}y}{\mathrm{d}x}.
\end{align}
Note that $-\frac{\mathrm{d}}{\mathrm{d}y}\frac{1}{\phi(y)}=\phi(y)^{-2}\phi'(y)$.
Furthermore, since $\phi'(y)=-y\phi(y)$, we have:
\begin{align}
    -\frac{\mathrm{d}}{\mathrm{d}y}\frac{1}{\phi(y)}=\phi(y)^{-2}(-y\phi(y))=-\frac{y}{\phi(y)}
\end{align}
Consequently, the second derivative is expressed as:
\begin{align}
    \frac{\mathrm{d}^2}{\mathrm{d}x^2}Q^{-1}(x) = \bigg(-\frac{y}{\phi(y)} \bigg) \cdot \bigg( - \frac{1}{\phi(y)} \bigg)
    = \frac{y}{\phi(y)^2}.
\end{align}
Evaluating the sign of the second derivative, we observe that $\frac{\mathrm{d}^2}{\mathrm{d}x^2}Q^{-1}(x) > 0$ for all $x \in (0, 0.5)$, which corresponds to $y \in (0,\infty)$. 
Therefore, $Q^{-1}(x)$ is strictly convex on the interval $x \in (0, 0.5)$.
\end{proof}

\begin{proof}[Proof of Theorem~\ref{thm:ML_lossy_sc_rate_conv_bound}]

In the global scheme, by treating $w^k$ as a single symbol and applying Theorem~\ref{thm:eq:ML_lossy_sc_eps_conv_bound}, the following holds for any $\gamma \ge 0$ regarding the sampling strategy $\mathcal{S}_{\mathrm{glob}}^{**}(k,\epsilon)$ that achieves $R(k,d,\epsilon,\mathcal{A})$:
\begin{align}
    \epsilon \ge \mathbb{P}_{\mathrm{joint}}[\jmath_{\mathrm{glob}}(W^k,T_{\mathrm{glob}}, H_{\mathrm{glob}}, d,\mathcal{A}) \ge k R(k,d,\epsilon,\mathcal{A}) + \gamma] - e^{-\gamma}.
\label{eq:pr_thm:ML_lossy_sc_rate_conv_bound_tmp_eps_bound}
\end{align}
where $\mathbb{P}_{\mathrm{joint}}$ denotes the probability measure under $P_{\mathrm{joint}}$.

Next, we decompose the tilted information as $\jmath_{\mathrm{glob}}(W^k,T_{\mathrm{glob}}, H_{\mathrm{glob}}, d,\mathcal{A}) = \jmath_{\mathrm{local}}(W^k,T^k,H^k,d,\mathcal{A}) + \Delta \jmath$. 
For any $\nu > 0$, we define the following two events:
\begin{align}
    E_1 &:= \{\jmath_{\mathrm{local}}(W^k,T^k,H^k,d,\mathcal{A}) \ge k R(k,d,\epsilon,\mathcal{A}) + \gamma - \mu_{\Delta} + \nu \}, \\
    E_2 &:= \{ \Delta \jmath \ge \mu_{\Delta} - \nu \}.
\end{align}
By definition, if $E_1 \cap E_2$ occurs, then the event $\{\jmath_{\mathrm{local}}(W^k,T^k,H^k,d,\mathcal{A}) + \Delta \jmath \ge k R(k,d,\epsilon,\mathcal{A}) + \gamma \}$ holds. Consequently, we obtain:
\begin{align}
     \mathbb{P}_{\mathrm{joint}}[\jmath_{\mathrm{glob}}(W^k,T_{\mathrm{glob}}, H_{\mathrm{glob}}, d,\mathcal{A}) 
     \ge k R(k,d,\epsilon,\mathcal{A}) + \gamma]
     &\ge \mathbb{P}_{\mathrm{joint}}[E_1 \cap E_2] \notag \\
     &\ge \mathbb{P}_{\mathrm{joint}}[E_1] + \mathbb{P}_{\mathrm{joint}}[E_2] -1.
\label{eq:pr_thm:ML_lossy_sc_rate_conv_bound_tmp_PboundE}
\end{align}
From Cantelli’s inequality, the following relationship is established:
\begin{align}
    \mathbb{P}_{\mathrm{joint}}[\Delta \jmath \le \mu_{\Delta} - \nu] \le \frac{V_{\Delta}}{V_{\Delta} + \nu^2}.
\end{align}
Therefore, regarding $\mathbb{P}_{\mathrm{joint}}[E_2] = \mathbb{P}_{\mathrm{joint}}[\Delta \jmath \ge \mu_{\Delta} - \nu]$, it follows that:
\begin{align}
    \mathbb{P}_{\mathrm{joint}}[E_2] \ge 1 - \frac{V_{\Delta}}{V_{\Delta} + \nu^2}.
\end{align}
Substituting this into Eq.~\eqref{eq:pr_thm:ML_lossy_sc_rate_conv_bound_tmp_PboundE} yields:
\begin{align}
    \mathbb{P}_{\mathrm{joint}}[\jmath_{\mathrm{glob}}(W^k,T_{\mathrm{glob}}, H_{\mathrm{glob}}, d,\mathcal{A}) 
     \ge k R(k,d,\epsilon,\mathcal{A}) + \gamma] \ge \mathbb{P}_{\mathrm{joint}}[E_1] - \frac{V_{\Delta}}{V_{\Delta} + \nu^2}.
\end{align}
Furthermore, applying Eq.~\eqref{eq:pr_thm:ML_lossy_sc_rate_conv_bound_tmp_eps_bound}, we have:
\begin{align}
    \epsilon \ge \mathbb{P}_{\mathrm{joint}}[E_1] - \frac{V_{\Delta}}{V_{\Delta} + \nu^2} -e^{-\gamma}. \notag
\end{align}
Rearranging this leads to:
\begin{align}
    \mathbb{P}_{\mathrm{joint}}[E_1] \le \epsilon + e^{-\gamma} + \frac{V_{\Delta}}{V_{\Delta} + \nu^2}.
\end{align}
We define the right-hand side as $\epsilon'(\nu) := \epsilon + e^{-\gamma} + \frac{V_{\Delta}}{V_{\Delta} + \nu^2}$.
Thus, the inequality can be expressed as:
\begin{align}
    \mathbb{P}_{\mathrm{local}}[\jmath_{\mathrm{local}}(W^k,T^k,H^k,d,\mathcal{A}) \ge k R(k,d,\epsilon,\mathcal{A}) + \gamma - \mu_{\Delta} + \nu]
    \le \epsilon'(\nu),
\label{eq:pr_thm:ML_lossy_sc_rate_conv_bound_tmp_epilon_prime_bound}
\end{align}
where $\mathbb{P}_{\mathrm{local}}$ is the probability measure under $P^{\mathcal{A}}_{H^k|T^k}P^{\mathcal{S}^*}_{T^k|W^k}P_{W^k}$.
Note that since $E_1$ depends solely on $P^{\mathcal{A}}_{H^k|T^k}P^{\mathcal{S}^*}_{T^k|W^k}P_{W^k}$, we have replaced $\mathbb{P}_{\mathrm{joint}}$ with $\mathbb{P}_{\mathrm{local}}$ in Eq.~\eqref{eq:pr_thm:ML_lossy_sc_rate_conv_bound_tmp_epilon_prime_bound}. 
Here, we let $\theta := k R(k,d,\epsilon,\mathcal{A}) + \gamma - \mu_{\Delta} + \nu$.

Since $\jmath_{\mathrm{local}}(w^k,t^k,h^k,d,\mathcal{A})= \sum^k_{i=1} \jmath_{W}(w_i, t_i, h_i, d, \mathcal{A})$ with $\mathbb{E}_{P^{\mathcal{A}}_{H|T} P^{\mathcal{S}^*}_{T|W} P_W}[\jmath_W(W,T,H,d,\mathcal{A})] = R(d,\mathcal{A})$ and $\mathrm{Var}_{P^{\mathcal{A}}_{H|T} P^{\mathcal{S}^*}_{T|W} P_W}(\jmath_W(W,T,H,d,\mathcal{A})) = V(d,\mathcal{A})$, the Berry-Esseen Theorem (Theorem~\ref{thm:berry_essen_clt}) ensures that for any $\alpha \in \mathbb{R}$:
\begin{align}
    \mathbb{P}_{\mathrm{local}}\bigg[\sum^k_{i=1} \jmath_{W}(W_i, T_i, H_i, d, \mathcal{A}) 
        > k \bigg( R(d,\mathcal{A}) + \alpha \sqrt{\frac{V(d,\mathcal{A})}{k}} \bigg) \bigg]
    \ge Q(\alpha) - \frac{B_k}{\sqrt{k}}.
\label{eq:pr_thm:ML_lossy_sc_rate_conv_bound_tmp_clt}
\end{align}
Here, we have defined $A_k = \frac{1}{k}\sum^k_{i=1}\mathbb{E}_{P^{\mathcal{A}}_{H|T} P^{\mathcal{S}^*}_{T|W} P_W}[|\jmath_W(W_i,T_i, H_i,d,\mathcal{A}) - \mathbb{R}_W(d,\mathcal{A}) |^3]$ and $B_k = 6\frac{A_k}{(V(d,\mathcal{A}))^{3/2}}$.

We choose $\alpha$ such that $k ( R(d,\mathcal{A}) + \alpha \sqrt{\frac{V(d,\mathcal{A})}{k}} ) = \theta$.
 Solving this equation for $\alpha$ yields:
\begin{align}
    \alpha = \frac{\theta - kR(d,\mathcal{A})}{\sqrt{k V(d,\mathcal{A})}}.
\label{eq:pr_thm:ML_lossy_sc_rate_conv_bound_tmp_alpha}
\end{align}
Substituting this into Eq.~\eqref{eq:pr_thm:ML_lossy_sc_rate_conv_bound_tmp_clt} gives:
\begin{align}
    \mathbb{P}_{\mathrm{local}}\bigg[\sum^k_{i=1} \jmath_{W}(W_i, T_i, H_i, d, \mathcal{A}) > \theta \bigg]
    \ge Q(\alpha) - \frac{B_k}{\sqrt{k}}.
\end{align}
Combining the above with Eq.~\eqref{eq:pr_thm:ML_lossy_sc_rate_conv_bound_tmp_epilon_prime_bound}, we obtain:
\begin{align}
    \epsilon'(\nu) \ge \mathbb{P}_{\mathrm{local}}\bigg[ \sum^k_{i=1} \jmath_{W}(W_i, T_i, H_i, d, \mathcal{A}) > \theta \bigg] \ge Q(\alpha) - \frac{B_k}{\sqrt{k}}.
\end{align}
Rearranging this inequality with respect to $\alpha$ leads to:
\begin{align}
    \alpha \ge Q^{-1}\bigg(\epsilon'(\nu) + \frac{B_k}{\sqrt{k}} \bigg).
\end{align}
where we have used the fact that $Q(\cdot)$ is strictly monotonically decreasing.
Let $\phi(x) = \frac{1}{\sqrt{2\pi}}e^{-x^2/2}$ be the probability density function of the standard normal distribution.
Using the equation $\frac{\mathrm{d}}{\mathrm{d}x}Q^{-1}(x) = -\frac{1}{\phi(Q^{-1}(x))}$, the Taylor expansion of $Q^{-1}\big(\epsilon'(\nu) + \frac{B_k}{\sqrt{k}} \big)$ around $\epsilon'(\nu)$ is given by:
\begin{align}
    Q^{-1}\bigg(\epsilon'(\nu) + \frac{B_k}{\sqrt{k}} \bigg)
    = Q^{-1}(\epsilon'(\nu)) - \frac{1}{\phi(Q^{-1}(\epsilon'(\nu)))} \frac{B_k}{\sqrt{k}} + O\bigg(\frac{1}{k}\bigg).
\end{align}
Thus, we have $\alpha \ge Q^{-1}(\epsilon'(\nu)) + O(\frac{1}{\sqrt{k}})$. 
Substituting Eq.~\eqref{eq:pr_thm:ML_lossy_sc_rate_conv_bound_tmp_alpha} into this inequality, we get:
\begin{align}
    \frac{\theta - kR(d,\mathcal{A})}{\sqrt{k V(d,\mathcal{A})}}
    &\ge Q^{-1}(\epsilon'(\nu)) + O\bigg(\frac{1}{\sqrt{k}}\bigg) \notag \\
    \theta  &\ge kR(d,\mathcal{A}) + \sqrt{k V(d,\mathcal{A})}Q^{-1}(\epsilon'(\nu)) + O(1).
\end{align}
From the definition of $\theta$, it follows that:
\begin{align}
    k R(k,d,\epsilon, \mathcal{A}) +\gamma - \mu_{\Delta} + \nu 
    \ge kR(d,\mathcal{A}) + \sqrt{k V(d,\mathcal{A})}Q^{-1}(\epsilon'(\nu)) + O(1).
\end{align}
Finally, solving for $R(k,d,\epsilon, \mathcal{A})$ yields:
\begin{align}
    R(k,d,\epsilon, \mathcal{A})
    \ge R(d,\mathcal{A}) + \sqrt{\frac{V(d,\mathcal{A})}{k}}Q^{-1}(\epsilon'(\nu)) 
        + \frac{\mu_{\Delta}}{k} - \frac{\nu}{k} - \frac{\gamma}{k} + O\bigg(\frac{1}{k} \bigg).
\label{eq:pr_thm:ML_lossy_sc_rate_conv_bound_tmp_Rglobal_lower1}
\end{align}

From Lemma~\ref{lem:Q_inv_convex}, $Q^{-1}(x)$ is a strictly convex function for $x \in (0, 0.5)$, which implies the following:
\begin{align}
    Q^{-1}(\epsilon'(\nu)) 
    \ge Q^{-1}(\epsilon) - \frac{1}{\phi(Q^{-1}(\epsilon))}(\epsilon'(\nu) - \epsilon).
\end{align} 
By the definition of $\epsilon'(\nu)$, we have $\epsilon'(\nu) - \epsilon = e^{-\gamma} + \frac{V_{\Delta}}{V_{\Delta}+\nu^2}$, and thus:
\begin{align}
    Q^{-1}(\epsilon'(\nu)) \ge Q^{-1}(\epsilon) - \frac{1}{\phi(Q^{-1}(\epsilon))} \bigg(e^{-\gamma} + \frac{V_{\Delta}}{V_{\Delta}+\nu^2} \bigg).
\end{align}
Substituting this result into Eq.~\eqref{eq:pr_thm:ML_lossy_sc_rate_conv_bound_tmp_Rglobal_lower1} yields:
\begin{align}
\begin{split}
    R(k,d,\epsilon, \mathcal{A})
    &\ge R(d,\mathcal{A}) + \sqrt{\frac{V(d,\mathcal{A})}{k}} \bigg\{ Q^{-1}(\epsilon) - \frac{1}{\phi(Q^{-1}(\epsilon))} \bigg(e^{-\gamma} + \frac{V_{\Delta}}{V_{\Delta}+\nu^2} \bigg) \bigg\}  \\
    &+ \frac{\mu_{\Delta}}{k} - \frac{\nu}{k} - \frac{\gamma}{k} + O\bigg(\frac{1}{k} \bigg).
\end{split}
\end{align}
Let $C = \frac{\sqrt{V(d,\mathcal{A})}}{\phi(Q^{-1}(\epsilon))}$. We can then rearrange the above as follows:
\begin{align}
\begin{split}
    R(k,d,\epsilon, \mathcal{A})
    &\ge R(d,\mathcal{A}) + \sqrt{\frac{V(d,\mathcal{A})}{k}} Q^{-1}(\epsilon)
     - \frac{C}{\sqrt{k}}e^{-\gamma} - \frac{C}{\sqrt{k}}\frac{V_{\Delta}}{V_{\Delta}+\nu^2}  \\
    &+ \frac{\mu_{\Delta}}{k} - \frac{\nu}{k} - \frac{\gamma}{k} + O\bigg(\frac{1}{k} \bigg).
\end{split}
\end{align}
By setting $\gamma = \frac{1}{2}\log k$, this inequality can be simplified to:
\begin{align}
\begin{split}
    R(k,d,\epsilon, \mathcal{A})
    &\ge R(d,\mathcal{A}) + \sqrt{\frac{V(d,\mathcal{A})}{k}} Q^{-1}(\epsilon)
        \\
    &- \frac{C}{\sqrt{k}}\frac{V_{\Delta}}{V_{\Delta}+\nu^2}- \frac{\nu}{k} + \frac{\mu_{\Delta}}{k} + O\bigg(\frac{\log k}{k} \bigg).
\end{split}
\label{eq:pr_thm:ML_lossy_sc_rate_conv_bound_tmp_Rglobal_lower2}
\end{align}

Here, let $L(\nu) := \frac{C}{\sqrt{k}}\frac{V_{\Delta}}{V_{\Delta}+\nu^2} + \frac{\nu}{k}$.
Since $V_{\Delta}$ is non-negative, the relation $V_{\Delta} + \nu^2 \ge \nu^2$ holds, yielding the following upper bound:
\begin{align}
    L(\nu) \le \frac{\nu}{k} + \frac{V_{\Delta}C}{\nu^2 \sqrt{k}} =: \widebar{L}(\nu).
\end{align}
The lower bound in Eq.~\eqref{eq:pr_thm:ML_lossy_sc_rate_conv_bound_tmp_Rglobal_lower2} remains valid even if $L(\nu)$ is replaced by $\widebar{L}(\nu)$.
To set a specific value for $\nu$, we determine the optimal $\nu^*$ that minimizes $\widebar{L}(\nu)$.
By setting the derivative $\widebar{L}'(\nu) = \frac{1}{k} - \frac{2V_{\Delta}C}{\nu^3 \sqrt{k}} = 0$ and solving for $\nu$, we obtain $\nu^* = (2V_{\Delta}C)^{\frac{1}{3}}k^{\frac{1}{6}}$. 
Substituting this back into $\widebar{L}(\nu)$ gives:
\begin{align}
    \widebar{L}(\nu^*) 
    &= \frac{(2V_{\Delta}C)^{\frac{1}{3}}k^{\frac{1}{6}}}{k} + \frac{V_{\Delta} C}{\sqrt{k} ((2V_{\Delta}C)^{\frac{1}{3}}k^{\frac{1}{6}})^2} \notag\\
    &=\frac{(2V_{\Delta}C)^{\frac{1}{3}}}{k^{\frac{5}{6}}} + \frac{V_{\Delta} C}{(2 V_{\Delta} C)^{\frac{2}{3}} k^{\frac{5}{6}} } \notag \\
    &=\bigg(1 + \frac{1}{2} \bigg) \frac{(2V_{\Delta}C)^{\frac{1}{3}}}{k^{\frac{5}{6}}} \notag \\
    &= \frac{3}{2}\frac{(2V_{\Delta}C)^{\frac{1}{3}}}{k^{\frac{5}{6}}}.
\end{align}
Consequently, applying the relation $L(\nu^*) \le \widebar{L}(\nu^*) = \frac{3}{2}\frac{(2V_{\Delta}C)^{\frac{1}{3}}}{k^{\frac{5}{6}}}$ to Eq.~\eqref{eq:pr_thm:ML_lossy_sc_rate_conv_bound_tmp_Rglobal_lower2} yields:
\begin{align}
    R(k,d,\epsilon, \mathcal{A})
    &\ge R(d,\mathcal{A}) + \sqrt{\frac{V(d,\mathcal{A})}{k}} Q^{-1}(\epsilon)
        + \frac{\mu_{\Delta}}{k} - \frac{3}{2 k^{\frac{5}{6}}}\bigg(2 V_{\Delta} \frac{\sqrt{V(d,\mathcal{A})}}{\phi(Q^{-1}(\epsilon))} \bigg)^{\frac{1}{3}} + O\bigg(\frac{\log k}{k} \bigg).
\end{align}
\end{proof}

\subsection{Proof of Theorem~\ref{thm:ML_lossy_sc_distortion_conv_bound}}
\label{apdx:proof:thm:ML_lossy_sc_distortion_conv_bound}

\begin{proof}[Proof of Theorem~\ref{thm:ML_lossy_sc_distortion_conv_bound}]

This proof follows the methodology used in the proof of Theorem 14 in \cite{kostina2012fixed}.
In this section, we derive Eq. \eqref{eq:ML_lossy_sc_distortion_conv_bound} from Eq. \eqref{eq:ML_lossy_sc_rate_conv_bound}.
First, let $(d_{\infty}, R_{\infty})$ be a fixed point on the rate-distortion curve $R(d,\mathcal{A})$, where $d_{\infty} \in (\underline{d}, \bar{d})$ and $R_{\infty}=R(d_{\infty},\mathcal{A})$.
We define $d_k = D(k,R_{\infty}, \epsilon, \mathcal{A})$.
According to the results in \cite{kostina2012fixed} (Appendix E), it follows that $|d_k - d_{\infty}| \le O(\frac{1}{\sqrt{k}})$.

Since $R(k, \cdot, \epsilon, \mathcal{A})$ and $D(k, \cdot, \epsilon, \mathcal{A})$ are inverse functions, the definition of $d_k$ implies $R_{\infty} = R(k, d_k, \epsilon, \mathcal{A})$.
Consequently, from Eq. \eqref{eq:ML_lossy_sc_rate_conv_bound} in Theorem \ref{thm:ML_lossy_sc_rate_conv_bound}, we have:
\begin{align}
  R_{\infty} \ge R(d_k, \mathcal{A}) + \sqrt{\frac{V(d_k,\mathcal{A})}{k}}Q^{-1}(\epsilon) +\Delta_k(d_k,\epsilon)+ O\bigg(\frac{\log k}{k}\bigg).
\label{eq:Mls_dcb_tmp1}
\end{align}
Taylor expanding $R(d_k, \mathcal{A})$, $V(d_k,\mathcal{A})$ and $\Delta_k(d_k,\epsilon))$ around $d_{\infty}$ yields:
\begin{align}
  & R(d_k, \mathcal{A}) = R(d_{\infty},\mathcal{A}) + R'(d_{\infty},\mathcal{A})(d_k - d_{\infty}) + O \bigg( \frac{1}{k} \bigg), \\
  & V(d_k, \mathcal{A}) = V(d_{\infty},\mathcal{A}) + O\bigg( \frac{1}{\sqrt{k}} \bigg), \\
  & \Delta_k(d_k,\epsilon) = \Delta_k(d_{\infty}, \epsilon) + O\bigg(\frac{1}{\sqrt{k}}\bigg).
\end{align}
In the above derivation, we have utilized the fact that $|d_k - d_{\infty}| \le O(1/\sqrt{k})$ and that the order of $\Delta_k(d,\epsilon)$ is at most $O(1)$ (Detailed in Appendix~\ref{apdx:subsec:detail_thm:ML_lossy_sc_rate_conv_bound_statement_interpretation}).
Substituting these expansions into Eq. \eqref{eq:Mls_dcb_tmp1} gives:
\begin{align}
  R_{\infty} \ge R(d_{\infty},\mathcal{A}) + R'(d_{\infty},\mathcal{A})(d_k - d_{\infty}) + \sqrt{\frac{V(d_{\infty},\mathcal{A})}{k}}Q^{-1}(\epsilon) + \Delta_k(d_{\infty},\epsilon) + O\bigg(\frac{\log k}{k}\bigg).
\label{eq:Mls_dcb_tmp2}
\end{align}
Using the definition $R_{\infty}=R(d_{\infty},\mathcal{A})$ and the property $R'(d_{\infty},\mathcal{A}) < 0$, we rearrange Eq. \eqref{eq:Mls_dcb_tmp2} to obtain:
\begin{align}
  d_k - d_{\infty} \ge - \frac{1}{R'(d_{\infty},\mathcal{A})}\sqrt{\frac{V(d_{\infty},\mathcal{A})}{k}}Q^{-1}(\epsilon)
    - \frac{1}{R'(d_{\infty},\mathcal{A})}\Delta_k(d_{\infty},\epsilon)+ O\bigg(\frac{\log k}{k}\bigg).
  \label{eq:Mls_dcb_tmp3}
\end{align}
By the derivative of the inverse function, $D'(R_{\infty}, \mathcal{A}) = \frac{1}{R'(d_{\infty}, \mathcal{A})}$ holds.
Furthermore, since $R(d,\mathcal{A}) =\mathbb{R}_W(d,\mathcal{A})$, we have $R'(d,\mathcal{A}) <0$ for any $d$.
Combining this with the definition of $\mathcal{V}(R,\mathcal{A})$, we obtain:
\begin{align}
  \sqrt{\mathcal{V}(R, \mathcal{A})} = |D'(R,\mathcal{A})|\sqrt{V(D(R, \mathcal{A}), \mathcal{A})} = -D'(R,\mathcal{A}) \sqrt{V(D(R,\mathcal{A}), \mathcal{A})}.
\end{align}
Applying these relations to Eq. \eqref{eq:Mls_dcb_tmp3} results in:
\begin{align}
  &d_k - d_{\infty} \ge  \sqrt{\frac{\mathcal{V}(R_{\infty}, \mathcal{A})}{k}}Q^{-1}(\epsilon) - D'(R_{\infty},\mathcal{A})\Delta_k(d_{\infty},\epsilon)  + O\bigg(\frac{\log k}{k}\bigg) \notag \\
  &D(k, R_{\infty}, \epsilon, \mathcal{A}) \notag \\
  &~~~~~~\ge D(R_{\infty}, \mathcal{A}) + \sqrt{\frac{\mathcal{V}(d_{\infty}, \mathcal{A})}{k}}Q^{-1}(\epsilon) - D'(R_{\infty},\mathcal{A})\Delta_k(D(R_{\infty}, \mathcal{A}),\epsilon)  + O\bigg(\frac{\log k}{k}\bigg). \notag
\end{align}
Here, we used the identities $d_k = D(k, R_{\infty}, \epsilon, \mathcal{A})$ and $d_{\infty}=D(R_{\infty}, \mathcal{A})$.
Since $(d_{\infty}, R_{\infty})$ is an arbitrary point on the rate-distortion curve $R(d,\mathcal{A})$ satisfying $d_{\infty} \in (\underline{d}, \bar{d})$, this completes the proof of \eqref{eq:ML_lossy_sc_distortion_conv_bound}.

\end{proof}

\subsection{Decomposition of the Rate-Dispersion Function $V(d,\mathcal{A})$}
\label{apdx:subsec:derive_V_decomposition}

In this section, we detail the decomposition $V(d,\mathcal{A})=V_{\mathrm{in}}(d,\mathcal{A})+V_{\mathrm{bet}}(d,\mathcal{A})$, along with the further decompositions of $V_{\mathrm{in}}$ and $V_{\mathrm{bet}}$.
First, the following identity is derived by applying the variance decomposition with respect to $P_W$:
\begin{align}
\begin{split}
  &V(d,\mathcal{A})= \underbracket[0.5pt]{\mathbb{E}_{P_W}\big[ \mathrm{Var}_{P^{\mathcal{A}}_{H|T}P^{\mathcal{S}^*}_{T|W}} (\jmath_{W}) \big]  }_{=: V_{\mathrm{in}}(d,\mathcal{A})}
    + \underbracket[0.5pt]{\mathrm{Var}_{P_W}\big( \mathbb{E}_{P^{\mathcal{A}}_{H|T}P^{\mathcal{S}^*}_{T|W}}[\jmath_{W}] \big)   }_{=: V_{\mathrm{bet}}(d,\mathcal{A})}. 
\end{split}
\end{align}

Similarly, the expansion of $V_{\mathrm{in}}$ follows from the definition of $\jmath_{W}$ and the variance decomposition:
\begin{align}
&\mathbb{E}_{P_W}\Big[ \mathrm{Var}_{P^{\mathcal{A}}_{H|T}P^{\mathcal{S}^*}_{T|W}} (\jmath_{W}(W,T, H,d,\mathcal{A})|W) \Big] \notag \\
&= \mathbb{E}_{P_W}\Big[ 
     \mathrm{Var}_{P^{\mathcal{A}}_{H|T}P^{\mathcal{S}^*}_{T|W}} \Big(\iota_{W;T^{\mathcal{S}^*}}(W;T) + \lambda^*_{\mathcal{A}}(d)(\mathsf{d}(W;H) -d) |W \Big) 
     \Big]  \notag \\
&= \mathbb{E}_{P_W}\Big[ 
  \mathrm{Var}_{P^{\mathcal{A}}_{H|T}P^{\mathcal{S}^*}_{T|W}} \Big(\iota_{W;T^{\mathcal{S}^*}}(W;T) + \lambda^*_{\mathcal{A}}(d)(\mathsf{d}(W;H)) |W \Big) 
      \Big] \notag \\
\begin{split}
  &= \mathbb{E}_{P_W}\Big[ 
    \mathrm{Var}_{P^{\mathcal{S}^*}_{T|W}} \Big(\iota_{W;T^{\mathcal{S}^*}}(W;T) \Big) 
    + (\lambda^*_{\mathcal{A}}(d))^2 \mathrm{Var}_{P^{\mathcal{A},\mathcal{S}^*}_{H|W}} \Big(\mathsf{d}(W;H) \Big) \\
  &~~~~~~+ 2\lambda^*_{\mathcal{A}}(d) \mathrm{Cov}_{P^{\mathcal{A},\mathcal{S}^*}_{H|W}} \Big(\iota_{W;H^{\mathcal{S}^*}_{\mathcal{A}}}(W;H), \mathsf{d}(W;H) \Big)
      \Big]
\end{split}
 \notag \\
\begin{split}
  &= \mathbb{E}_{P_W}\Big[ 
      \mathrm{Var}_{P^{\mathcal{S}^*}_{T|W}} \Big(\iota_{W;T^{\mathcal{S}^*}}(W;T) \Big)  \\
  &~~~~~ + (\lambda^*_{\mathcal{A}}(d))^2 \Big\{ 
       \mathrm{Var}_{P^{\mathcal{S}^*}_{T|W}} \Big( \mathbb{E}_{P^{\mathcal{A}}_{H|T}} \Big[ \mathsf{d}(W;H) \Big]  \Big) 
      +  \mathbb{E}_{P^{\mathcal{S}^*}_{T|W}} \Big[ \mathrm{Var}_{P^{\mathcal{A}}_{H|T}} \Big( \mathsf{d}(W;H) \Big)  \Big]\Big\} \\
    &~~~~~ +  2\lambda^*_{\mathcal{A}}(d) \mathrm{Cov}_{P^{\mathcal{A}}_{H|T}P^{\mathcal{S}^*}_{T|W}} \Big( \iota_{W;H^{\mathcal{S}^*}_{\mathcal{A}}}(W;H), \mathsf{d}(W;H) \Big)
  \Big].
\end{split}
\end{align}

Likewise, the expansion of $V_{\mathrm{bet}}$ is obtained using the definition of $\jmath_{W}$:
\begin{align}
&\mathrm{Var}_{P_W}\Big( \mathbb{E}_{P^{\mathcal{A}}_{H|T}P^{\mathcal{S}^*}_{T|W}}\Big[\jmath_{W}(W,T,H,d,\mathcal{A})|W \Big] \Big)  \notag \\
&=\mathrm{Var}_{P_W}\Big( \mathbb{E}_{P^{\mathcal{A}}_{H|T}P^{\mathcal{S}^*}_{T|W}}\Big[ \iota_{W;T^{\mathcal{S}^*}}(W;T) + \lambda^*_{\mathcal{A}}(d)(\mathsf{d}(W;H) -d) |W \Big] \Big) \notag \\
&=\mathrm{Var}_{P_W}\Big( 
    \mathbb{E}_{P^{\mathcal{A}}_{H|T}P^{\mathcal{S}^*}_{T|W}}\Big[\iota_{W;T^{\mathcal{S}^*}}(W;T) |W \Big]  
    + \mathbb{E}_{P^{\mathcal{A}}_{H|T}P^{\mathcal{S}^*}_{T|W}}\Big[\lambda^*_{\mathcal{A}}(d)\mathsf{d}(W;H) |W \Big]  
    \Big) \notag \\
\begin{split}
&= \mathrm{Var}_{P_W}\Big(  \mathbb{E}_{P^{\mathcal{S}^*}_{T|W}}\Big[\iota_{W;T^{\mathcal{S}^*}}(W;T) |W \Big] \Big) 
  + (\lambda^*_{\mathcal{A}}(d))^2\mathrm{Var}_{P_W}\Big(  \mathbb{E}_{P^{\mathcal{A}}_{H|T}P^{\mathcal{S}^*}_{T|W}}\Big[\mathsf{d}(W;H) |W \Big] \Big)  \\
&~~~~~ + 2\lambda^*_{\mathcal{A}}(d) \mathrm{Cov}_{P_W}\Big( 
            \mathbb{E}_{P^{\mathcal{S}^*}_{T|W}}\Big[\iota_{W;T^{\mathcal{S}^*}}(W;T) |W \Big], 
            \mathbb{E}_{P^{\mathcal{A}}_{H|T}P^{\mathcal{S}^*}_{T|W}}\Big[\mathsf{d}(W;H) |W \Big] 
         \Big).
\end{split}
\end{align}

\section{Proof of Appendix}
\label{apdx:sec:proof_for_appendix}

\subsection{Proof of Theorem~\ref{thm:V_delta_order} in Appendix~\ref{apdx:subsec:rate_distortion_theorem_RkdeA}}
\label{apdx:subsec:proof_thm:V_delta_order}

\begin{proof}[Proof of Theorem~\ref{thm:V_delta_order}]

First, we evaluate the order of $V_{\Delta}$.
$V_{\Delta}$ can be upper-bounded as follows by applying the variance inequality $\mathrm{Var}(X-Y) \le 2\mathrm{Var}(X) + 2\mathrm{Var}(Y)$:
\begin{align}
    V_{\Delta} \le 2\mathrm{Var}(\jmath_{\mathrm{glob}}(W^k,T,H,d,\mathcal{A})) + 2\sum^k_{i=1}\mathrm{Var}(\jmath_{\mathrm{local}}(W,T,H,d,\mathcal{A})).
\label{eq:V_Delta_bound_tmp1}
\end{align}
Since $\mathrm{Var}(\jmath_{\mathrm{local}}(W,T,H,d,\mathcal{A}))=V(d,\mathcal{A})$, the second term on the right-hand side of Eq. \eqref{eq:V_Delta_bound_tmp1} becomes $2\sum^k_{i=1}\mathrm{Var}(\jmath_{\mathrm{local}}(W,T,H,d,\mathcal{A}))=2kV(d,\mathcal{A})$, which is of order $O(k)$.

Next, we examine the order of the first term on the right-hand side of Eq. \eqref{eq:V_Delta_bound_tmp1} with respect to $k$.
Since $P^{\mathcal{S}^*_{\mathrm{glob}}}_{T|W}$ is the minimizer of $\mathbb{I}(W^k;T)$ under the constraint $\mathbb{E}[\mathsf{d}(W^k;\mathcal{A}(\mathcal{S}(W^k))] \le d$, it is given by the following Gibbs distribution form according to \cite{cover1999elements} (Chapter 10.7):
\begin{align}
    P^{\mathcal{S}^*_{\mathrm{glob}}}_{T|W}(t|w^k) = \frac{P^{\mathcal{S}^*_{\mathrm{glob}}}_{T}(t) e^{-k\lambda^{*\mathrm{glob}}_{\mathcal{A}}(d)\bar{\mathsf{d}}(W^k;t)} }{ Z(W^k)},
\end{align}
where $P^{\mathcal{S}^*_{\mathrm{glob}}}_{T}(t)=\mathbb{E}_{P_W}[P^{\mathcal{S}^*_{\mathrm{glob}}}_{T|W}(t|W^k)]$, $\bar{\mathsf{d}}(W^k;T):=\mathbb{E}_{P^{\mathcal{A}}_{H|T=T}}[\mathsf{d}(W^k;H)]$, $Z(W^k):=\sum_{t \in\mathcal{T}}P^{\mathcal{S}^*_{\mathrm{glob}}}_{T}(t) e^{-k\lambda^{*\mathrm{glob}}_{\mathcal{A}}(d)\bar{\mathsf{d}}(W^k;t)}$.
Using this relationship, $\mathrm{Var}(\jmath_{\mathrm{glob}}(W^k,T,H,d,\mathcal{A}))$ can be expressed as:
\begin{align}
    &\mathrm{Var}(\jmath_{\mathrm{glob}}(W^k,T,H,d,\mathcal{A})) \notag \\
    &=\mathrm{Var}\bigg(\log\frac{P^{\mathcal{S}^*_{\mathrm{glob}}}_{T|W}(T|w^k)}{P^{\mathcal{S}^*_{\mathrm{glob}}}_{T}(t)} + k\lambda^{*\mathrm{glob}}_{\mathcal{A}}(d)\bigg(\mathsf{d}(W^k;H) -d  \bigg) \bigg) \notag \\
    &=\mathrm{Var}\bigg(
        \underbrace{-\log Z(W^k)}_{=:A(W^k)} + 
        \underbrace{k\lambda^{*\mathrm{glob}}_{\mathcal{A}}(d) (\mathsf{d}(W^k;H) - \bar{\mathsf{d}}(W^k;T) )}_{=:B(W^k, T, H)}
        \bigg). \notag \\
\end{align}

Regarding $B(W^k, T, H)$, we first calculate the expectation with respect to $H$ conditioned on $W^k$ and $T$:
\begin{align}
    \mathbb{E}_{P^\mathcal{A}_{H|T}}[\lambda^{*\mathrm{glob}}_{\mathcal{A}}(d) (\mathsf{d}(W^k;H) - \bar{\mathsf{d}}(W^k;T) ) |W^k, T] =  (\bar{\mathsf{d}}(W^k;T)-\bar{\mathsf{d}}(W^k;T)) = 0.
\end{align}
Consequently, the covariance between $A(W^k)$ and $B(W^k, T, H)$ is zero:
\begin{align}
    \mathrm{Cov}(A(W^k), B(W^k, T, H)) 
    = \mathbb{E}_{P^{\mathcal{S}^{**}_{\mathrm{glob}}(k,\epsilon)}_{T|W}P_W}[A(W^k) \cdot  \mathbb{E}_{P^\mathcal{A}_{H|T}}[B(W^k, T, H) | W^k,T]] = 0.
\end{align}

We then evaluate the order of $\lambda^{*\mathrm{glob}}_{\mathcal{A}}(d)$. 
The term $\lambda^{*\mathrm{glob}}_{\mathcal{A}}(d)$ is defined as a negative gradient of $\frac{1}{k} \mathbb{I}(W^k;T^{\mathcal{S}^*_{\mathrm{glob}}})$.
From the properties of mutual information, we have $\mathbb{I}(W^k;T^{\mathcal{S}^*_{\mathrm{glob}}}) \le \mathbb{H}(W^k)=k\mathbb{H}(W)$, which implies that the order of $\mathbb{I}(W^k;T^{\mathcal{S}^*_{\mathrm{glob}}})$ is $O(k)$.
Furthermore, $\mathbb{I}(W^k;T^{\mathcal{S}^*_{\mathrm{glob}}})=\inf_{\mathcal{S}: \mathbb{E}[\mathsf{d}(W^k;\mathcal{A}(\mathcal{S}(W^k))] \le d}\mathbb{I}(W^k;T)$ holds, which corresponds to the definition of $\mathbb{R}_W(d,\mathcal{A})$ with $W$ replaced by $W^k$. 
Thus, as in the case of Theorem~\ref{thm:rate_distortion_theorem}, $\mathbb{I}(W^k;T^{\mathcal{S}^*_{\mathrm{glob}}})$ is also convex with respect to $d$.
To explicitly denote the dependence on $d$, let $R_k(d):=\mathbb{I}(W^k;T^{\mathcal{S}^*_{\mathrm{glob}}})$.
From the convexity of $R_k(d)$, the following inequality holds for $\lambda^{*\mathrm{glob}}_{\mathcal{A}}(d)$ using a small constant $\delta>0$:
\begin{align}
    \lambda^{*\mathrm{glob}}_{\mathcal{A}}(d) 
    = -\frac{1}{k}\frac{\mathrm{d} R_k(d)}{\mathrm{d}d}
    \le \frac{1}{k}\frac{R_k(d-\delta) - R_k(d)}{\delta}
    \le \frac{1}{k}\frac{R_k(d-\delta)}{\delta},
\end{align}
where the final inequality follows from the non-negativity of mutual information.
Since $R_k(d-\delta)$ represents $\mathbb{I}(W^k;T^{\mathcal{S}^*_{\mathrm{glob}}})$ with a target distortion of $d-\delta$, its order with respect to $k$ is $O(k)$.
Consequently, the order of $\lambda^{*\mathrm{glob}}_{\mathcal{A}}(d)$ is also $O(1)$.

Next, we evaluate $\mathrm{Var}(A(W^k))$. 
Since $A(W^k)$ is a function of the independent variables $W_1, \dots, W_k$, we analyze the variation of $Z(W^k) = \sum_{t \in \mathcal{T}} P^{\mathcal{S}^*_{\mathrm{glob}}}_{T}(t) e^{-\lambda^{*\mathrm{glob}}_{\mathcal{A}}(d)\bar{\mathsf{d}}(W^k;t)}$ when $W_i$ is replaced by an independent copy $W'_i \sim P_W$. 
Let $W^{k,i}$ denote the sequence where the $i$-th element is replaced.
For any $t$, the difference in the exponent of $Z$ is given by:
\begin{align}
    |\lambda^{*\mathrm{glob}}_{\mathcal{A}}(d)\bar{\mathsf{d}}(W^k;t) - \lambda^{*\mathrm{glob}}_{\mathcal{A}}(d)\bar{\mathsf{d}}(W^{k,i};t)| 
    = \lambda^{*\mathrm{glob}}_{\mathcal{A}}(d) \frac{1}{k} | \mathbb{E}_{P^\mathcal{A}_{H|T}}[\mathsf{d}(W_i;H) - \mathsf{d}(W'_i;H)]|.
\end{align}
Given that the maximum value of $\mathsf{d}(\cdot;\cdot)$ is $d_{\mathrm{max}}$, the absolute difference is upper-bounded by $\frac{\lambda^{*\mathrm{glob}}_{\mathcal{A}}(d) d_{\mathrm{max}}}{k}$. 
Due to the monotonicity of the exponential function, the following inequality holds:
\begin{align}
    e^{-\frac{\lambda^{*\mathrm{glob}}_{\mathcal{A}}(d) d_{\mathrm{max}}}{k}} Z(W^k) \le Z(W^{k,i}) \le e^{\frac{\lambda^{*\mathrm{glob}}_{\mathcal{A}}(d) d_{\mathrm{max}}}{k}} Z(W^k).
\end{align}
By taking the logarithm of each side, we obtain the upper bound for the difference between $-\log(Z(W^k))$ and $-\log(Z(W^{k,i}))$: 
\begin{align}
    |-\log(Z(W^k)) - (-\log(Z(W^{k,i})))| \le \frac{\lambda^{*\mathrm{glob}}_{\mathcal{A}}(d) d_{\mathrm{max}}}{k}.
\end{align}
Applying the Efron-Stein inequality \cite{boucheron2013concentration}, we have:
\begin{align}
    \mathrm{Var}(A(W^k)) 
    &\le \frac{1}{2}\sum^k_{i=1}\mathbb{E}[|-\log(Z(W^k)) - (-\log(Z(W^{k,i})))|]  \notag \\
    &\le  \frac{1}{2}\sum^k_{i=1} \bigg(\frac{\lambda^{*\mathrm{glob}}_{\mathcal{A}}(d) d_{\mathrm{max}}}{k} \bigg)^2 \notag \\
    &\le \frac{\lambda^{*\mathrm{glob}}_{\mathcal{A}}(d)^2 d_{\mathrm{max}}^2}{2k}.
\end{align}
Considering that $\lambda^{*\mathrm{glob}}_{\mathcal{A}}(d)$ is $O(1)$, the order of $\mathrm{Var}(A(W^k))$ with respect to $k$ is $O(k^{-1})$.

Finally, we evaluate $\mathrm{Var}(B(W^k, T, H))$.
Since $\mathbb{E}_{P^\mathcal{A}_{H|T}}[\lambda^{*\mathrm{glob}}_{\mathcal{A}}(d) (\mathsf{d}(W^k;H) - \bar{\mathsf{d}}(W^k;T) ) |W^k, T] = 0$, the variance can be expanded as follows:
\begin{align}
    \mathrm{Var}(B(W^k, T, H))
    &= \mathbb{E}_{P^{\mathcal{S}^{**}_{\mathrm{glob}}(k,\epsilon)}_{T|W}P_W}[\mathrm{Var}_{P^{\mathcal{A}}_{H|T}}(
        k\lambda^{*\mathrm{glob}}_{\mathcal{A}}(d) (\mathsf{d}(W^k;H) - \bar{\mathsf{d}}(W^k;T)
        )] \notag \\
    &=k^2\lambda^{*\mathrm{glob}}_{\mathcal{A}}(d)^2\mathbb{E}_{P^{\mathcal{S}^{**}_{\mathrm{glob}}(k,\epsilon)}P_W}[\mathrm{Var}_{P^{\mathcal{A}}_{H|T}}(
         (\mathsf{d}(W^k;H))].
\end{align}
Because the maximum value of $\mathsf{d}(W^k;H)$ is $d_{\mathrm{max}}$, the variance $\mathrm{Var}_{P^{\mathcal{A}}_{H|T}}(\mathsf{d}(W^k;H))$ is $O(1)$. 
Combined with the fact that $\lambda^{*\mathrm{glob}}_{\mathcal{A}}(d)$ is $O(1)$, we conclude that the order of $\mathrm{Var}(B(W^k, T, H))$ is $O(k^2)$.

Next, we evaluate the order of $\mu_{\Delta}$.
The term $\mu_{\Delta}$ can be expressed as $\mu_{\Delta}=\mathbb{E}_{P_{\mathrm{joint}}}[\jmath_{\mathrm{glob}}(W^k,T,H,d,\mathcal{A})]-\sum^k_{i=1}\mathbb{E}_{P_{\mathrm{joint}}}[\jmath_{\mathrm{local}}(W,T,H,d,\mathcal{A})]$.
The second term $\sum^k_{i=1}\mathbb{E}_{P_{\mathrm{joint}}}[\jmath_{\mathrm{local}}(W,T,H,d,\mathcal{A})]=\sum^k_{i=1}\mathbb{E}_{P_{\mathrm{local}}}[\jmath_{\mathrm{local}}(W,T,H,d,\mathcal{A})]$ is equal to $\sum^k_{i=1} \mathbb{R}_W(d,\mathcal{A})$ by the definition of $\jmath_W$.
Thus, the second term is of order $O(k)$.

The first term is given by $\mathbb{E}_{P_{\mathrm{joint}}}[\jmath_{\mathrm{glob}}(W^k,T,H,d,\mathcal{A})]=\mathbb{E}_{P_{\mathrm{joint}}}[\iota_{W;T^{\mathrm{S}^*_{\mathrm{glob}}}}(W^k;T) + k\lambda^{*\mathrm{glob}}_{\mathcal{A}}(d)(\mathsf{d}(W^k;H) -d)]$.
Here, we evaluate the order of the terms within the expectation.
First, since the maximum value of $\mathsf{d}(\cdot;\cdot)$ is $d_{\mathrm{max}}$ and $\lambda^{*\mathrm{glob}}_{\mathcal{A}}(d)$ is the negative gradient of $\frac{1}{k}\mathbb{I}(W^k;T^{\mathcal{S}^*_{\mathrm{glob}}})$ with respect to $d$, the order of $k\lambda^{*\mathrm{glob}}_{\mathcal{A}}(d)(\mathsf{d}(W^k;H) -d)$ is $O(1)$.
Next, we evaluate the order of $\iota_{W;T^{\mathrm{S}^*_{\mathrm{glob}}}}(W^k;T)$.
Using the fact that $P^{\mathcal{S}^*_{\mathrm{glob}}}_{T|W}(t|w^k) = \frac{P^{\mathcal{S}^*_{\mathrm{glob}}}_{T}(t) e^{-k\lambda^{*\mathrm{glob}}_{\mathcal{A}}(d)\bar{\mathsf{d}}(W^k;t)} }{ Z(W^k)}$, the information density $\iota_{W;T^{\mathrm{S}^*_{\mathrm{glob}}}}(W^k;T)$ can be expressed as:
\begin{align}
    \iota_{W;T^{\mathrm{S}^*_{\mathrm{glob}}}}(W^k;T) 
    = \log \frac{P^{\mathcal{S}^*_{\mathrm{glob}}}_{T|W}(T|W^k)}{P^{\mathcal{S}^*_{\mathrm{glob}}}_{T}(T)}
    = -k\lambda^{*\mathrm{glob}}_{\mathcal{A}}(d)\bar{\mathsf{d}}(W^k;T) - \log Z(W^k).
\end{align}
Because the maximum value of $\mathsf{d}(\cdot;\cdot)$ is $d{\mathrm{max}}$, the order of $-k\lambda^{*\mathrm{glob}}_{\mathcal{A}}(d)\bar{\mathsf{d}}(W^k;T)$ is $O(k)$.
Furthermore, given $Z(W^k)=\sum_{t\in\mathcal{T}}P^{\mathcal{S}^*_{\mathrm{glob}}}_{T}(t)\exp(-k\lambda^{*\mathrm{glob}}_{\mathcal{A}}(d)\bar{\mathsf{d}}(W^k;t))$, the following inequality holds for the exponential term:
\begin{align}
    \exp(-k\lambda^{*\mathrm{glob}}_{\mathcal{A}}(d) d_{\mathrm{max}}) \le \exp(-k\lambda^{*\mathrm{glob}}_{\mathcal{A}}(d)\bar{\mathsf{d}}(W^k;t)) \le \exp(0)=1.
\end{align}
Taking the expectation of each side with respect to $P^{\mathcal{S}^*_{\mathrm{glob}}}_{T}(t)$ yields:
\begin{align}
    \exp(-k\lambda^{*\mathrm{glob}}_{\mathcal{A}}(d) d_{\mathrm{max}}) \le Z(W^k) \le 1.
\end{align}
Consequently, we have $0 \le -\log Z(W^k) \le k\lambda^{\mathrm{glob}}{\mathcal{A}}(d) d{\mathrm{max}}$.
This implies that $-\log Z(W^k)$ is also of order $O(k)$.
In summary, since the integrand of $\mu_{\Delta}$ is of order $O(k)$, the order of $\mu_{\Delta}$ is likewise $O(k)$.

Consequently, the order of $\mu_{\Delta}$ is $O(k)$ in the worst case.

\end{proof}

\subsection{Proof of Theorem~\ref{thm:rate_distortion_theorem_epsilon}}
\label{apdx:subsec:proof_thm:rate_distortion_theorem_epsilon}

\begin{proof}[Proof of Theorem~\ref{thm:rate_distortion_theorem_epsilon}]

The proof is completed by establishing the following two points:

(i) For any $\mathcal{A}$, $d \in (d_{\mathrm{min}},d_{\mathrm{max}})$, and $\epsilon \in (0,1)$, $\limsup_{k \rightarrow \infty} R(k,d,\epsilon,\mathcal{A}) \ge \mathbb{R}_{W}$ holds.

(ii) $\limsup_{k \rightarrow \infty} \widebar{R}(k,d,\epsilon,\mathcal{A}) \le \mathbb{R}_{W}$ holds.

If (ii) is established, the assumption implies that $\limsup_{k \rightarrow \infty} R(k,d,\epsilon,\mathcal{A}) \le \mathbb{R}_{W}$.
Combining this with (i) yields the desired result: $\limsup_{k \rightarrow \infty} R(k,d,\epsilon,\mathcal{A}) = \mathbb{R}_{W}$.

\textbf{Proof of (i):}

From Theorem~\ref{thm:ML_lossy_sc_rate_conv_bound}, for any $k$, $\mathcal{A}$, $d \in (d_{\mathrm{min}},d_{\mathrm{max}})$, and $\epsilon \in (0,1)$, the following inequality holds:
\begin{align}
    R(k,d,\epsilon, \mathcal{A}) \ge R(d,\mathcal{A}) + \sqrt{\frac{V(d,\mathcal{A})}{k}}Q^{-1}(\epsilon) + \Delta_k(d,\epsilon)+ O\bigg(\frac{\log k}{k}\bigg).
\end{align}
Since $\lim_{k\rightarrow \infty}\Delta_k(d,\epsilon)=0$ by assumption, the lower bound converges to $R(d,\mathcal{A})$ as $k \rightarrow \infty$. 
Furthermore, Theorem~\ref{thm:rate_distortion_theorem} implies that $R(d,\mathcal{A})=\mathbb{R}_W(d,\mathcal{A})$.
Consequently, we obtain $\limsup_{k \rightarrow \infty} R(k,d,\epsilon,\mathcal{A}) \ge \mathbb{R}_{W}(d,\mathcal{A})$.

\textbf{Proof of (ii):}

This proof follows the format of Theorem 1 in \cite{kostina2012fixed}.
Since $\mathbb{R}_W(d,\mathcal{A})$ is convex and non-increasing with respect to $d$, it is continuous.
Therefore, for any small constant $\delta > 0$, one can choose sufficiently small $\gamma > 0$ and $\nu > 0$ such that the following holds:
\begin{align}
    \mathbb{R}_W(d - \gamma - \nu, \mathcal{A}) \le \mathbb{R}_W(d, \mathcal{A}) + \frac{\delta}{2}.
\end{align}
Let $\widebar{\mathcal{S}}^*$ be an optimal sampling strategy that achieves $\mathbb{R}_W(d - \gamma - \nu, \mathcal{A})$. 
Under this strategy, the expected distortion satisfies $\mathbb{E}_{P^{\mathcal{A}}_{H|T}P^{\widebar{\mathcal{S}}^*}_{T|W}P_W}[\mathsf{d}(W;H)] \le d - \gamma - \nu$, and the mutual information under $P^{\mathcal{A}}_{H|T}P^{\widebar{\mathcal{S}}^*}_{T|W}P_W$ satisfies $\mathbb{I}(W;T) \le \mathbb{R}_W(d,\mathcal{A}) + \frac{\delta}{2}$.
We define the target rate as $\widebar{R} = \mathbb{R}_W(d,\mathcal{A}) + \delta$, with $\eta = \frac{\delta}{4}$ and $M=\lceil e^{k\widebar{R}} \rceil$.

We construct a set $\mathcal{C}= \{T^k(1), \dots, T^k(M)\}$ by independently generating $M$ sequences $t^k$ according to the marginal distribution $P^{\widebar{\mathcal{S}}^*}_{T^k}(t^k) = \prod^{k}_{i=1} P^{\widebar{\mathcal{S}}^*}_{T}(t_i)$ induced by $\widebar{\mathcal{S}}^*$. 
We then define the set $\mathcal{E}$ on $\mathcal{W}^k \times \mathcal{T}^k$ as follows:
\begin{align}
    \mathcal{E} := \bigg\{ (w^k, t^k) \bigg| \frac{1}{k}\sum^k_{i=1}\mathbb{E}_{P^{\mathcal{A}}_{H|T=t_i}} [\mathsf{d}(w_i,H)] \le d - \gamma ~~\land~~ \iota_{W;T}(w^k; t^k) \le k (\widebar{R} - \eta)    \bigg\}.
\end{align}

For a fixed $w^k$, we evaluate the probability $\mathbb{P}_{T^k}^{\widebar{\mathcal{S}}^*}(\mathcal{E}_{w^k})$ that a single sequence $t'^k$ randomly generated from $P^{\widebar{\mathcal{S}}^*}_{T^k}(t^k)$ belongs to the set $\mathcal{E}_{w^k} := \{t^k \mid (w^k, t^k) \in \mathcal{E}\}$ as:
\begin{align}
    \mathbb{P}_{T^k}^{\widebar{\mathcal{S}}^*}(\mathcal{E}_{w^k}) 
    = \sum_{t^k \in \mathcal{E}_{w^k}}  P_{T^k}^{\widebar{\mathcal{S}}^*}(t^k)
    = \sum_{t^k \in \mathcal{E}_{w^k}} P_{T^k|W^k}^{\widebar{\mathcal{S}}^*}(t^k|w^k) \exp(-\iota_{W;T}(w^k; t^k)).
\end{align}
From the definition of $\mathcal{E}$, it holds that $\exp(-\iota_{W;T}(w^k;t^k)) \ge \exp(-k\widebar{R} + k\eta) \ge \frac{e^{k\eta}}{M}$. 
Substituting this into the above equation yields:
\begin{align}
    \mathbb{P}_{T^k}^{\widebar{\mathcal{S}}^*}(\mathcal{E}_{w^k}) \ge \frac{e^{k\eta}}{M} \sum_{t^k \in \mathcal{E}_{w^k}} P_{T^k|W^k}^{\widebar{\mathcal{S}}^*}(t^k|w^k) = \frac{e^{k\eta}}{M} \mathbb{P}_{T^k|W^k=w^k}^{\widebar{\mathcal{S}}^*}(\mathcal{E}_{w^k}).
\end{align}

Now, consider a sampling strategy $\mathcal{S}$ that, given $w^k$, outputs a sequence from $\mathcal{C}$ if it is contained in $\mathcal{E}_{w^k}$; otherwise, it outputs $T^k(1)$.
The number of possible training datasets that this sampling strategy can output is at most $M$.

The probability that no sequence in $\mathcal{C}$ is contained in $\mathcal{E}_{w^k}$, denoted by $P_e(\mathcal{C}|w^k)$, can be upper-bounded as follows.
This derivation utilizes the i.i.d. nature of $T^k$ and the inequality $(1 - a_1 a_2)^M \le e^{- a_1 M a_2}$, which holds for $a_1 a_2 \le 1$:
\begin{align}
    P_e(\mathcal{C}|w^k) = (1 - \mathbb{P}_{T^k}^{\widebar{\mathcal{S}}^*}(\mathcal{E}_{w^k}))^M \le \exp(-e^{k\eta} \mathbb{P}_{T^k|W^k=w^k}^{\widebar{\mathcal{S}}^*}(\mathcal{E}_{w^k})).
\end{align}
Furthermore, applying the inequality $e^{-a_1 a_2} \le 1 - a_2 + e^{-a_1}$ (valid for $a_1 > 0$ and $a_2 \in [0,1]$) yields:
\begin{align}
    P_e(\mathcal{C}|w^k) \le 1 - \mathbb{P}_{T^k|W^k=w^k}^{\widebar{\mathcal{S}}^*}(\mathcal{E}_{w^k}) + \exp(-e^{k \eta}).
\end{align}
Taking the expectation of both sides with respect to $w^k$, we obtain:
\begin{align}
    \mathbb{E}_{P_W}[P_e(\mathcal{C}|W^k)] 
    &\le \mathbb{E}_{P_W}[1 - \mathbb{P}_{T^k|W^k}^{\widebar{\mathcal{S}}^*}(\mathcal{E}_{W^k}) + \exp(-e^{k \eta})]  \notag \\
    &= \mathbb{E}_{P_W}[1 - \mathbb{P}_{T^k|W^k}^{\widebar{\mathcal{S}}^*}(\mathcal{E}_{W^k})]  + \exp(-e^{k \eta}).
\end{align}
Based on this result, the upper bound for $\mathbb{E}_{P_W}[P_e(\mathcal{C}|W^k)]$ can be expressed as:
\begin{align}
    &\mathbb{E}_{P_W}[P_e(\mathcal{C}|W^k)] \notag \\
    &\le \mathbb{P}^{\widebar{\mathcal{S}}^*}_{W^k, T^k}\bigg(
        \frac{1}{k}\sum^k_{i=1}\mathbb{E}_{P^{\mathcal{A}}_{H|T=T_i}} [\mathsf{d}(W_i,H)] > d - \gamma ~~\lor~~ \frac{1}{k}\iota_{W;T}(W^k; T^k) > \widebar{R} - \eta  
    \bigg) + \exp(-e^{k\eta}). 
\end{align}

First, we evaluate the term $\frac{1}{k}\iota_{W;T}(W^k; T^k) > \widebar{R} - \eta$:
\begin{align}
    \frac{1}{k}\iota_{W;T}(W^k; T^k) > \widebar{R} - \eta = \mathbb{R}_W(d,\mathcal{A}) + \frac{3}{4}\delta = \bigg(\mathbb{R}_W(d,\mathcal{A}) + \frac{\delta}{2} \bigg) + \frac{\delta}{4} \ge \mathbb{I}(W;T) + \frac{\delta}{4}.
\end{align}
Under $P^{\widebar{\mathcal{S}}^*}_{T^k|W^k}P_{W^k}$, the pairs $\{(W_i, T_i)\}^k_{i=1}$ are i.i.d.; thus, by the weak law of large numbers (WLLN), $\frac{1}{k}\iota_{W;T}(W^k; T^k)$ converges in probability to the expectation $\mathbb{I}(W;T)$.
Consequently, $\mathbb{P}^{\widebar{\mathcal{S}}^*}_{W^k, T^k}(\frac{1}{k}\iota_{W;T}(W^k; T^k) \ge \mathbb{I}(W;T) + \frac{\delta}{4} )$ vanishes as $k \rightarrow \infty$.

Similarly, for the term $\frac{1}{k}\sum^k_{i=1}\mathbb{E}_{P^{\mathcal{A}}_{H|T=T_i}} [\mathsf{d}(W_i,H)] > d - \gamma$, the condition $\mathbb{E}_{P^{\mathcal{A}}_{H|T}P^{\widebar{\mathcal{S}}^*}_{T|W}P_{W}} [\mathsf{d}(W,H)] \le d - \gamma - \nu$ implies:
\begin{align}
    &\mathbb{P}^{\widebar{\mathcal{S}}^*}_{W^k, T^k}\bigg( \frac{1}{k}\sum^k_{i=1}\mathbb{E}_{P^{\mathcal{A}}_{H|T=T_i}} [\mathsf{d}(W_i,H)] > d - \gamma \bigg) \notag \\
    &\le \mathbb{P}^{\widebar{\mathcal{S}}^*}_{W^k, T^k}\bigg( \frac{1}{k}\sum^k_{i=1}\mathbb{E}_{P^{\mathcal{A}}_{H|T=T_i}} [\mathsf{d}(W_i,H)] > \mathbb{E}_{P^{\mathcal{A}}_{H|T}P^{\widebar{\mathcal{S}}^*}_{T|W}P_{W}} [\mathsf{d}(W,H)] + \nu \bigg).    
\end{align}
By the WLLN, the right-hand side also converges to $0$ as $k \rightarrow \infty$.
Furthermore, $\exp(-e^{k\eta})$ vanishes as $k \rightarrow \infty$. 
Therefore, for this sampling strategy $\mathcal{S}$, we conclude that $\lim_{k \rightarrow \infty} \mathbb{E}_{P_W}[P_e(\mathcal{C}|W^k)] = 0$.

Finally, we evaluate the excess distortion probability achieved by this sampling strategy $\mathcal{S}$.
The total distortion $\frac{1}{k}\sum^k_{i=1} \mathsf{d}(W_i; H_i)$ resulting from the hypothesis sequence $H^k$ (generated independently by the learning algorithm $\mathcal{A}$) consists of independent random variables with expectation $\frac{1}{k}\sum^k_{i=1}\mathbb{E}_{P^{\mathcal{A}}_{H|T=t_i}} [\mathsf{d}(w_i,H)]$.
From Hoeffding's inequality \cite{mohri2018foundations}, letting $d_{\mathrm{max}}$ denote the maximum value of $\mathsf{d}(\cdot;\cdot)$, the following holds:
\begin{align}
    \mathbb{P}\bigg(\frac{1}{k}\sum^k_{i=1} \mathsf{d}(W_i; H_i) - \frac{1}{k}\sum^k_{i=1}\mathbb{E}_{P^{\mathcal{A}}_{H|T=t_i}} [\mathsf{d}(w_i,H)]  > \gamma \bigg| W^k, T^k \bigg) 
    \le \exp\bigg( - \frac{2k\gamma^2}{d^2_{\mathrm{max}}} \bigg).
\end{align}
Using this result, we evaluate the excess distortion probability as follows:
\begin{align}
    &\mathbb{P}\bigg(\frac{1}{k}\sum^k_{i=1} \mathsf{d}(W_i; H_i) > d \bigg) \notag \\
    &= \mathbb{P}\bigg(\frac{1}{k}\sum^k_{i=1} \mathsf{d}(W_i; H_i) > d \cap (W^k, T^k) \notin \mathcal{E}  \bigg)
        + \mathbb{P}\bigg(\frac{1}{k}\sum^k_{i=1} \mathsf{d}(W_i; H_i) > d \cap (W^k, T^k) \in \mathcal{E}  \bigg) \notag \\
    &\le \mathbb{E}_{P_W}[P_e(\mathcal{C}|W^k)] 
    + \mathbb{E}_{P^{\widebar{\mathcal{S}}^*}_{T^k|W^k}P_{W^k}}\bigg[\mathbbm{1}_{[(W^k, T^k) \in \mathcal{E}]} 
                        \mathbb{P}\bigg( \frac{1}{k}\sum^k_{i=1} \mathsf{d}(W_i; H_i) > d   \bigg| W^k, T^k \bigg) \bigg]\notag \\
    &\le \mathbb{E}_{P_W}[P_e(\mathcal{C}|W^k)] + \mathbb{E}_{P^{\widebar{\mathcal{S}}^*}_{T^k|W^k}P_{W^k}}\bigg[ \notag \\
    &~~~~~~~~~~~~ \mathbbm{1}_{[(W^k, T^k) \in \mathcal{E}]} 
                        \mathbb{P}\bigg( \frac{1}{k}\sum^k_{i=1} \mathsf{d}(W_i; H_i) -\frac{1}{k}\sum^k_{i=1}\mathbb{E}_{P^{\mathcal{A}}_{H|T=T_i}} [\mathsf{d}(W_i,H)] > \gamma  \bigg| W^k, T^k \bigg) \bigg]\notag \\
    &\le \mathbb{E}_{P_W}[P_e(\mathcal{C}|W^k)] 
            + \mathbb{E}_{P^{\widebar{\mathcal{S}}^*}_{T^k|W^k}P_{W^k}}\bigg[\mathbbm{1}_{[(W^k, T^k) \in \mathcal{E}]} 
                        \exp \bigg( - \frac{2 k \gamma^2}{d^2_{\mathrm{max}}}  \bigg) \bigg]\notag \\
    &\le \mathbb{E}_{P_W}[P_e(\mathcal{C}|W^k)]  + \exp\bigg( - \frac{2k\gamma^2}{d^2_{\mathrm{max}}} \bigg).
\end{align}
Since both terms on the right-hand side of the above vanish as $k \rightarrow \infty$, we have $\lim_{k \rightarrow \infty} \mathbb{P}(\frac{1}{k}\sum^k_{i=1} \mathsf{d}(W_i; H_i) > d) = 0$.
This implies that for any $\epsilon \in (0,1)$, there exists a sufficiently large $K$ such that for all $k > K$, there exists a sampling strategy $\mathcal{S}$ capable of generating at most $M$ training datasets while satisfying $\mathbb{P}(\frac{1}{k}\sum^k_{i=1} \mathsf{d}(W_i; H_i) > d) \le \epsilon$. 
Thus, for all $k > K$, we obtain:
\begin{align}
    \widebar{R}(k, d, \epsilon, \mathcal{A}) \le \widebar{R} = \mathbb{R}_W(d,\mathcal{A}) + \delta.
\end{align}
Taking the limit superior of both sides as $k \rightarrow \infty$ yields:
\begin{align}
    \limsup_{k \rightarrow \infty} \widebar{R}(k, d, \epsilon, \mathcal{A}) \le \mathbb{R}_W(d,\mathcal{A}) + \delta.
\end{align}
Finally, since $\delta > 0$ can be chosen arbitrarily small, taking the limit as $\delta \rightarrow 0$ establishes that $\limsup_{k \rightarrow \infty} \widebar{R}(k, d, \epsilon, \mathcal{A}) \le \mathbb{R}_W(d,\mathcal{A})$.

\end{proof}

\subsection{Proof of Theorem~\ref{thm:ML_lossy_sc_rate_conv_bound_alternative} in Appendix~\ref{apdx:sec:alternative_ver_rate_conv_bound}}
\label{apdx:subsec:proof_thm:ML_lossy_sc_rate_conv_bound_alternative}

First, we define the $\mathcal{A}$-specified $d$-tilted information for the scheme $\{H_i = \mathcal{A}(T_i), T_i=\mathcal{S}(W_i)\}_{i=1}^k$ with $k > 1$.
In Appendix~\ref{apdx:subsec:detail_thm:ML_lossy_sc_rate_conv_bound_preparation}, we refer to $\{H_i = \mathcal{A}(\mathcal{S}(W_i))\}_{i=1}^k$ as the local scheme and define its corresponding tilted information as $\jmath_{\mathrm{local}}(w^k,t^k,h^k,d,\mathcal{A})$. 
We restate the definition below:
\begin{align}
    \jmath_{\mathrm{local}}(w^k,t^k,h^k,d,\mathcal{A}) 
    := \jmath_{W}(w^k,t^k,h^k,d,\mathcal{A}) 
    := \iota_{W;T^{\mathrm{S}^*}}(w^k;t^k) + k\lambda^{*}_{\mathcal{A}}(d)(\mathsf{d}(w^k;h^k) -d). \notag
\end{align}
In the local scheme, the relations $P^{\mathcal{S}^*}_{T^k|W^k} = \prod_{i=1}^k P^{\mathcal{S}^*}_{T_i|W_i}$ and $P^{\mathcal{A}}_{H^k|T^k} = \prod_{i=1}^k P^{\mathcal{A}}_{H_i|T_i}$ hold.
Consequently, we have $\iota_{W;T^{\mathcal{S}^*}}(w^k;t^k) = \sum_{i=1}^k \iota_{W;T^{\mathcal{S}^*}}(w_i;t_i)$.
Furthermore, by definition, $\mathsf{d}(w^k;h^k) = \frac{1}{k}\sum_{i=1}^k \mathsf{d}(w_i;h_i)$.
Therefore, the following holds:
\begin{align}
    \jmath_{\mathrm{local}}(w^k,t^k,h^k,d,\mathcal{A})
    &= \sum^k_{i=1}\iota_{W;T^{\mathcal{S}^*}}(w_i;t_i) + \sum^k_{i=1}\lambda^{*}_{\mathcal{A}}(d)(\mathsf{d}(w_i;h_i) -d) \notag \\
    &= \sum^k_{i=1} \jmath_{W}(w_i, t_i, h_i, d, \mathcal{A}).
\end{align}

\begin{proof}[Proof of Theorem~\ref{thm:ML_lossy_sc_rate_conv_bound_alternative}]
  
This proof follows the methodology used in the proof of Theorem 12 in \cite{kostina2012fixed}.

To apply Theorem \ref{thm:berry_essen_clt} with $Z_i=\jmath_W(W_i,T^{^{\mathcal{S}^*}}_i,H_{i},d,\mathcal{A})$, we define $A_k = \frac{1}{k}\sum^k_{i=1}\mathbb{E}[|\jmath_W(W_i,T^{\mathcal{S}^*}_i, H_{i},d,\mathcal{A}) - \mathbb{R}_W(d,\mathcal{A}) |^3]$ and $B_k = 6\frac{A_k}{(V(d,\mathcal{A}))^{3/2}}$.
Here, $T^{\mathcal{S}^*}_{i}$ is the training dataset derived from $W_i$ via $\mathcal{S}^*$.
Additionally, in Theorem \ref{thm:eq:ML_lossy_sc_eps_conv_bound}, we set $\gamma = \frac{1}{2}\log k$ and specify:
\begin{align}
  bn &:= kR(d,\mathcal{A}) + \sqrt{k V(d,\mathcal{A})}Q^{-1}(\epsilon_k) - \gamma, \label{eq:Mls_rcb_Km}\\
  \epsilon_k &:=  \epsilon + e^{-\gamma} + \frac{B_k}{\sqrt{k}}. \label{eq:Mls_rcb_eps_n}
\end{align}
Applying Theorem \ref{thm:eq:ML_lossy_sc_eps_conv_bound} yields the following for any $(k, n, d, \epsilon', \mathcal{A})$ local sampling:
\begin{align}
  \epsilon' 
  \ge \mathbb{P}\bigg[ \sum^k_{i=1}\jmath_W(W_i,T^{\mathcal{S}^*}_i, H_{i}, d,\mathcal{A}) \ge kR(d,\mathcal{A}) + \sqrt{kV(d,\mathcal{A})}Q^{-1}(\epsilon_k)  \bigg] - e^{-\gamma}.
\label{eq:Mls_rcb_tmp1}
\end{align}

When $Z_i=\jmath_W(W_i,T^{\mathcal{S}^*}_{i}, H_i,d,\mathcal{A})$, the parameters $\mu_k$ and $V_k$ in Theorem \ref{thm:berry_essen_clt} correspond to $\mathbb{R}_W(d,\mathcal{A})$ and $V(d,\mathcal{A})$, respectively.
Furthermore, Theorem \ref{thm:rate_distortion_theorem} implies that $\mathbb{R}_W(d,\mathcal{A}) = R(d,\mathcal{A})$.
Consequently, setting $\alpha=Q^{-1}(\epsilon_k)$ makes the first term on the right-hand side of Eq. \eqref{eq:Mls_rcb_tmp1} identical to the left-hand side of Eq. \eqref{eq:berry_essen_clt} in Theorem \ref{thm:berry_essen_clt}.
Thus, by Theorem \ref{thm:berry_essen_clt}, we obtain:
\begin{align}
  \mathbb{P}\bigg[ \sum^k_{i=1}\jmath_W(W_i,T^{\mathcal{S}^*}_{i}, H_i, d,\mathcal{A}) \ge kR(d,\mathcal{A}) + \sqrt{kV(d,\mathcal{A})}Q^{-1}(\epsilon_k)  \bigg] - Q(Q^{-1}(\epsilon_k))
  &\ge - \frac{B_k}{\sqrt{k}} \notag \\
  \mathbb{P}\bigg[ \sum^k_{i=1}\jmath_W(W_i,T^{\mathcal{S}^*}_{i}, H_i, d,\mathcal{A}) \ge kR(d,\mathcal{A}) + \sqrt{kV(d,\mathcal{A})}Q^{-1}(\epsilon_k)  \bigg]
  &\ge \epsilon_k + (\epsilon + e^{-\gamma} - \epsilon_k) \notag \\
  \mathbb{P}\bigg[ \sum^k_{i=1}\jmath_W(W_i,T^{\mathcal{S}^*}_{i}, H_i,d,\mathcal{A}) \ge kR(d,\mathcal{A}) + \sqrt{kV(d,\mathcal{A})}Q^{-1}(\epsilon_k)  \bigg] 
  &\ge \epsilon + e^{-\gamma}.
\label{eq:Mls_rcb_tmp2}
\end{align}

Substituting Eq. \eqref{eq:Mls_rcb_tmp2} into Eq. \eqref{eq:Mls_rcb_tmp1} yields $\epsilon' \ge \epsilon$.
Consequently, specifying $bn$ as in Eq. \eqref{eq:Mls_rcb_Km} precludes achieving an excess distortion probability strictly less than $\epsilon$.
That is, any $(k, n, d, \epsilon, \mathcal{A})$ local sampling satisfies:
\begin{align}
  bn \ge kR(d,\mathcal{A}) + \sqrt{k V(d,\mathcal{A})}Q^{-1}(\epsilon_k) - \gamma.
\label{eq:Mls_rcb_tmp3}
\end{align}
Dividing both sides by $k$ gives:
\begin{align}
  \frac{bn}{k} \ge R(d,\mathcal{A}) + \sqrt{\frac{V(d,\mathcal{A})}{k}}Q^{-1}(\epsilon_k) - \frac{\gamma}{k}.
\label{eq:Mls_rcb_tmp4}
\end{align}
Here, letting $\epsilon_k = \epsilon + \Delta$ with $\Delta=e^{-\gamma} + \frac{B_k}{\sqrt{k}}$, we can Taylor expand $Q^{-1}(\epsilon_k)$ as follows:
\begin{align}
  Q^{-1}(\epsilon_k) = +  Q^{-1}(\epsilon) + (Q^{-1})'(\epsilon) \times \Delta + \cdots.
\end{align}
Since $\gamma = \frac{1}{2}\log k$ implies that $\Delta$ is of order $O(1/\sqrt{k})$, the second term on the right-hand side of Eq. \eqref{eq:Mls_rcb_tmp4} can be expanded as follows:
\begin{align}
  \sqrt{\frac{V(d,\mathcal{A})}{k}}Q^{-1}(\epsilon_n) 
  & = \sqrt{\frac{V(d,\mathcal{A})}{k}}Q^{-1}(\epsilon) + (Q^{-1})'(\epsilon) \sqrt{\frac{V(d,\mathcal{A})}{k}}\Delta + \cdots \notag \\
  &= \sqrt{\frac{V(d,\mathcal{A})}{k}}Q^{-1}(\epsilon)  + O\bigg(\frac{1}{k}\bigg). \notag
\end{align}
Applying this result to Eq. \eqref{eq:Mls_rcb_tmp4}, the following holds for any $(k, n, d, \epsilon, \mathcal{A})$ local sampling:
\begin{align}
  \frac{bn}{k} \ge R(d,\mathcal{A}) + \sqrt{\frac{V(d,\mathcal{A})}{k}}Q^{-1}(\epsilon) - \frac{1}{2}\frac{\log k}{k} + O\bigg(\frac{1}{k}\bigg).
  \label{eq:Mls_rcb_tmp5}
\end{align}

Since $\widebar{R}(k,d,\epsilon, \mathcal{A})$ is the minimum value of $\frac{bn}{k}$ within the class of $(k, n, d, \epsilon, \mathcal{A})$ local sampling strategies, we have the following:
\begin{align}
    \widebar{R}(k,d,\epsilon, \mathcal{A}) \ge R(d,\mathcal{A}) + \sqrt{\frac{V(d,\mathcal{A})}{k}}Q^{-1}(\epsilon) - \frac{1}{2}\frac{\log k}{k} + O\bigg(\frac{1}{k}\bigg).
\end{align}

By definition, $R(k,d,\epsilon, \mathcal{A})= \widebar{R}(k,d,\epsilon, \mathcal{A}) + R_{\Delta}(k,d,\epsilon)$, which implies:
\begin{align}
    R(k,d,\epsilon, \mathcal{A}) \ge R(d,\mathcal{A}) + \sqrt{\frac{V(d,\mathcal{A})}{k}}Q^{-1}(\epsilon) + R_{\Delta}(k,d,\epsilon) + O\bigg(\frac{\log k}{k}\bigg). \notag
\end{align}
    
\end{proof}

\section{Restricting the Target Sampling Strategies for Analysis}
\label{apdx:sec_limiting_sampling_strategy}

In the main body of this paper, for a fixed learning algorithm $\mathcal{A}$, we derived the lower bounds for the minimum rate $R(k,d,\epsilon,\mathcal{A})$ and distortion $D(k,R,\epsilon,\mathcal{A})$ as theoretical limits when the optimal sampling strategy is employed among all possible sampling strategies.
This framework can be further extended to derive theoretical limits when the optimal sampling strategy is selected from within a specific class of sampling strategies (e.g., pool-based active learning).
In this section, we describe the methodology for this derivation.

First, we define the class of sampling strategies to be analyzed as $\mathfrak{S}$.
The following are examples of what this class $\mathfrak{S}$ may represent:

\textbf{Example 1: } 
A class in which $m \in \mathbb{N}_+$ samples are drawn from each $w_i$ according to the distribution $P^*_{\bm{X}Y|W=w_i}$.
In this class, $m$ is a degree of freedom, and the resulting total size of the training dataset is $n = m \times k$.

\textbf{Example 2 (pool-based active learning (AL) \cite{settles2009active} or core set learning \cite{agarwal2005geometric}):} 
A class in which $m \in \mathbb{N}_+$ samples are drawn from each $w_i$ according to the distribution $P^*_{\bm{X}Y|W=w_i}$ to obtain a pool of $m \times k$ samples, from which a total of $n$ samples are subsequently selected to form the training dataset.
If we consider the $m \times k$ samples as unlabeled instances and assume that labeling occurs only when the $n$ samples are selected, this class corresponds to pool-based AL.
Conversely, if we treat the $m \times k$ samples as already labeled, this class corresponds to data selection in core-set learning.

\textbf{Example 3 (including i.i.d. sampling):}
A class in which $m \in \mathbb{N}_+$ samples are drawn from each $w_i$ according to $P^*_{\bm{X}Y|W=w_i}$, and a total of $n$ samples are then selected via random sampling without replacement to form the training dataset.

We define $n^*(k,d,\epsilon, \mathcal{A}, \mathfrak{S})$ as the minimum $n$ among sampling strategies that satisfy the $(k,n,d,\epsilon,\mathcal{A})$-sampling condition and belong to the class $\mathfrak{S}$. 
Accordingly, the minimum rate is defined as $R(k,d,\epsilon, \mathcal{A}, \mathfrak{S}) := \frac{bn^*(k,d,\epsilon, \mathcal{A}, \mathfrak{S})}{k}$.
Furthermore, we define the case where the infimum of $\mathbb{R}_W(d,\mathcal{A})$ is restricted to the class $\mathfrak{S}$ as follows:
\begin{align}
    \mathbb{R}_W(d,\mathcal{A}, \mathfrak{S}) := \inf_{\mathcal{S}: \mathbb{E}[\mathsf{d}(W;\mathcal{A}(\mathcal{S}(W)))] \le d ~\land~ \mathcal{S} \in \mathfrak{S} } \mathbb{I}(W;T).
\label{eq:lc_R_W_limited_S}
\end{align}
To avoid extreme values of $d$, we assume the existence of $d'_{\mathrm{min}} := \inf \{d : \mathbb{R}_{W}(d, \mathcal{A}, \mathfrak{S}) < \infty \}$ (Assumption (\Three')).
Additionally, we assume the existence of a sampling strategy $\mathcal{S}^*_{\mathfrak{S}}$ that achieves the infimum in Eq.~\eqref{eq:lc_R_W_limited_S} with equality (Assumption (\Four')). 
When a training dataset $T$ is obtained through this sampling strategy, it is explicitly denoted as $T^{\mathcal{S}^*_{\mathfrak{S}}}$.
Subsequently, we define the $\mathcal{A}$-specified $d$-tilted information, restricted to the class $\mathfrak{S}$, as:
\begin{align}
    \jmath_W(w,t,h,d,\mathcal{A},\mathfrak{S}) 
    := \iota_{W;T^{\mathcal{S}^*_{\mathfrak{S}}}}(w;t) + \lambda^*_{\mathcal{A},\mathfrak{S}}(d)(\mathsf{d}(w;h) - d),
\end{align}
where $\lambda^*_{\mathcal{A},\mathfrak{S}}(d) := -\frac{\mathrm{d}}{\mathrm{d}d}\mathbb{R}_W(d,\mathcal{A}, \mathfrak{S})$.
By definition, $\mathbb{E}_{P^{\mathcal{A}}{H|T}P^{\mathcal{S}^*_{\mathfrak{S}}}_{T|W}P_W}[\jmath_W(W,T,H,d,\mathcal{A},\mathfrak{S})] = \mathbb{R}_W(d,\mathcal{A}, \mathfrak{S})$ holds.

Under these definitions and assumptions, an inequality holds where $\jmath_W(W,T,H,d,\mathcal{A})$ in Theorem~\ref{thm:eq:ML_lossy_sc_eps_conv_bound} is replaced by $\jmath_W(W,T,H,d,\mathcal{A},\mathfrak{S})$.
The proof is analogous to that of Theorem~\ref{thm:eq:ML_lossy_sc_eps_conv_bound} (see Appendix~\ref{apdx:proof:thm:eq:ML_lossy_sc_eps_conv_bound}).

\begin{theorem}[The version of Theorem~\ref{thm:eq:ML_lossy_sc_eps_conv_bound} for the class $\mathfrak{S}$]
For any $d > d'_{\mathrm{min}}$ and any sampling strategy $\mathcal{S} \in \mathfrak{S}$ that satisfies the $(n, d, \epsilon, \mathcal{A})$-sampling condition, the following holds:
\begin{equation}
    \textstyle \epsilon \ge \sup_{\gamma \ge 0} \{ \mathbb{P}[\jmath_W(W,T,H,d,\mathcal{A},\mathfrak{S}) \ge bn + \gamma] - e^{-\gamma} \}.
  \label{eq:ML_lossy_sc_eps_conv_bound_limited_S}
\end{equation}
\label{thm:eq:ML_lossy_sc_eps_conv_bound_limited_S}
\end{theorem}

(V') The required distortion level $d$ satisfies  $d \in (d'_{\mathrm{min}},d'_{\mathrm{max}})$, where $d'_{\mathrm{max}} := \sup \{d : \mathbb{R}_W(d,\mathcal{A},\mathfrak{S})>0 \}$. 
Additionally, the excess distortion probability $\epsilon$ satisfies $0 < \epsilon < 1$.
(\Six') $\jmath_W(w,t, h,d,\mathcal{A},\mathfrak{S})$ possesses a finite absolute third moment.

Under these assumptions and results, Theorem \ref{thm:ML_lossy_sc_rate_conv_bound} also holds when the analysis is restricted to the class $\mathfrak{S}$, as stated below.
The proof follows the same procedure as that of Theorem \ref{thm:ML_lossy_sc_rate_conv_bound} (see Appendix~\ref{apdx:sec:detail_thm:ML_lossy_sc_rate_conv_bound}).
\begin{theorem}
Under Assumptions (I) and (\Three')--(\Six'), for any $k$, $\mathcal{A}$, $d \in (d_{\mathrm{min}},d_{\mathrm{max}})$ and $\epsilon \in (0,1)$:
\begin{align}
\begin{split}
    &R(k,d,\epsilon, \mathcal{A},\mathfrak{S}) \\
    &\ge \mathbb{R}_W(d,\mathcal{A},\mathfrak{S}) + \sqrt{\frac{V(d,\mathcal{A},\mathfrak{S})}{k}} Q^{-1}(\epsilon)
    + \underbrace{\frac{\mu_{\Delta}}{k} - \frac{3}{2 k^{\frac{5}{6}}}\bigg(2 V_{\Delta} \frac{\sqrt{V(d,\mathcal{A},\mathfrak{S})}}{\phi(Q^{-1}(\epsilon))} \bigg)^{\frac{1}{3}} }_{=:\Delta_k(d,\epsilon,\mathfrak{S})}
    + O\bigg(\frac{\log k}{k} \bigg),
\end{split}
\label{eq:ML_lossy_sc_rate_conv_bound_limited_S}
\end{align}
where $V(d,\mathcal{A},\mathfrak{S})=\mathrm{Var}_{P^{\mathcal{A}}_{H|T}P^{\mathcal{S}^*_{\mathfrak{S}}}P_W}(\jmath_W(W,T,H,d,\mathcal{A},\mathfrak{S}))$.
The variables $\mu_{\Delta}$ and $V_{\Delta}$ are defined in the same manner as in Appendix~\ref{apdx:subsec:detail_thm:ML_lossy_sc_rate_conv_bound_preparation}.
\label{thm:ML_lossy_sc_rate_conv_bound_limited_S}
\end{theorem}

The quantities $\mathbb{R}_W(d,\mathcal{A},\mathfrak{S})$ and $V(d,\mathcal{A},\mathfrak{S})$ in Theorem~\ref{thm:ML_lossy_sc_rate_conv_bound_limited_S} are obtained by modifying the sampling strategies used for the expectations of $R(d,\mathcal{A})$ and $V(d,\mathcal{A})$ in Theorem~\ref{thm:ML_lossy_sc_rate_conv_bound} from $\mathcal{S}^*$ to $\mathcal{S}^*_{\mathfrak{S}}$.
Consequently, the interpretations of each term discussed in Section~\ref{subsec:lc_sample_comp_bound} and Appendix~\ref{apdx:subsec:detail_interpret_lc_rate_dist_func}, as well as the relationships with IT-bounds and stability theory discussed in Appendix~\ref{apdx:sec:bridge_ours_others}, remain valid for these terms.
Furthermore, by restricting the scope of the infimum in the definition of $\mathbb{R}_W(d,\mathcal{A},\mathfrak{S})$ to $\mathfrak{S}$, both $\mathbb{R}_W(d,\mathcal{A},\mathfrak{S})$ and $V(d,\mathcal{A},\mathfrak{S})$ become quantities more specialized to $\mathfrak{S}$.
Regarding the lower bound of the minimum distortion, based on Theorem~\ref{thm:ML_lossy_sc_rate_conv_bound_limited_S}, the corresponding results for the case restricted to $\mathfrak{S}$ can be obtained using a method similar to the proof of Theorem~\ref{thm:ML_lossy_sc_distortion_conv_bound} (Appendix~\ref{apdx:proof:thm:ML_lossy_sc_distortion_conv_bound}).

Similarly, Theorem \ref{thm:ML_lossy_sc_distortion_conv_bound} holds even when the sampling strategies under analysis are restricted to $\mathfrak{S}$.
First, let $D(k,R,\epsilon,\mathcal{A},\mathfrak{S})$ be the inverse function of $R(k,d,\epsilon, \mathcal{A},\mathfrak{S})$. 
Furthermore, let $R(k,d,\mathcal{A},\mathfrak{S})$ denote the minimum rate among all $\langle k,n,d,\mathcal{A} \rangle$ sampling strategies contained in $\mathfrak{S}$.
Then, following the same procedure as the proof of Theorem~\ref{thm:ML_lossy_sc_distortion_conv_bound} (Appendix~\ref{apdx:proof:thm:ML_lossy_sc_distortion_conv_bound}), we can show the following:

\begin{theorem}
In addition to Assumptions (I)--(\Six), assume that $R(d,\mathcal{A},\mathfrak{S})$ is twice differentiable with $\frac{\mathrm{d} \mathbb{R}_W(d,\mathcal{A},\mathfrak{S})}{\mathrm{d}d}\neq 0$, and that $V(d,\mathcal{A},\mathfrak{S})$ is differentiable on the interval $(\underline{d},\bar{d}] \subseteq (d_{\mathrm{min}}, d_{\mathrm{max}}]$. 
Then, for any rate $R$ satisfying $R=R(d,\mathcal{A},\mathfrak{S})$ for some $d \in (\underline{d},\bar{d})$, the following holds:
\begin{align}
\begin{split}
    \textstyle D(k,R,\epsilon, \mathcal{A},\mathfrak{S}) 
    &\textstyle \ge D(R,\mathcal{A},\mathfrak{S}) + \sqrt{\frac{\mathcal{V}(R,\mathcal{A},\mathfrak{S})}{k}}Q^{-1}(\epsilon)  \\
    &\textstyle~~~ -D'(R,\mathcal{A},\mathfrak{S})\Delta_k(D(R,\mathcal{A},\mathfrak{S}),\epsilon,\mathfrak{S}) +  O \big( \frac{\log k}{k} \big).
\end{split}
  \label{eq:ML_lossy_sc_distortion_conv_bound_limited_S}
\end{align}
Here, $D(R,\mathcal{A},\mathfrak{S}) := \lim_{k \rightarrow \infty} D(k,R,\mathcal{A},\mathfrak{S})$ and $\mathcal{V}(R,\mathcal{A},\mathfrak{S}) := (D'(R,\mathcal{A},\mathfrak{S}))^2 V(D(R,\mathcal{A},\mathfrak{S}), \mathcal{A},\mathfrak{S})$, where $D(k,R,\mathcal{A},\mathfrak{S})$ denotes the inverse function of $R(k,d,\mathcal{A},\mathfrak{S})$.
\label{thm:ML_lossy_sc_distortion_conv_bound_limited_S}
\end{theorem}

\section{Limitation}
\label{apdx_sec:limitation}

Our framework has two limitations motivating future research.
First, the lower bounds' dependence on the unknown $\mathcal{S}^*$ complicates numerical evaluation; necessitating a evaluation methodology.
Second, verifying $V_{\mathrm{bet}}$ as a valid measure of inductive bias mismatch is challenging, and requires a robust verification protocol.

\section{About Code of Ethics}
\label{apdx_sec:code_of_ethics_broader_impacts}

\textbf{About Potential Harms.}
This study does not involve human subjects, and all datasets used are publicly available.

\end{document}